\documentclass{article}

\usepackage{arxiv}

\usepackage[utf8]{inputenc} 
\usepackage[T1]{fontenc}    
\usepackage{hyperref}       
\usepackage{url}            
\usepackage{booktabs}       
\usepackage{amsfonts}       
\usepackage{nicefrac}       
\usepackage{microtype}      
\usepackage{lipsum}		
\usepackage{graphicx}
\usepackage{doi}
\usepackage{microtype}
\usepackage{subfigure}
\usepackage{booktabs} 
\usepackage{breqn}

\usepackage[ruled,vlined]{algorithm2e}
\usepackage[round]{natbib}

\DeclareMathOperator{\Tr}{Tr}
\usepackage{float}

\usepackage{amsmath}
\usepackage{amsthm}
\usepackage{amssymb}
\usepackage{multirow}
\usepackage{makecell}
\newtheorem{theorem}{Theorem}[section]

\newtheorem{proposition}{Proposition}[section]
\newtheorem{remark}{Remark}[section]
\newtheorem{definition}{Definition}[section]
\RequirePackage{algorithmic}

\newcommand{\RNum}[1]{\uppercase\expandafter{\romannumeral #1\relax}}

\newcommand\inv[1]{#1\raisebox{1.15ex}{$\scriptscriptstyle-\!1$}}

\usepackage[LGR,T1]{fontenc}
\usepackage[utf8]{inputenc}   

\usepackage{xr}
\makeatletter
\newcommand*{\addFileDependency}[1]{
  \typeout{(#1)}
  \@addtofilelist{#1}
  \IfFileExists{#1}{}{\typeout{No file #1.}}
}
\makeatother

\newcommand*{\myexternaldocument}[1]{%
    \externaldocument{#1}%
    \addFileDependency{#1.tex}%
    \addFileDependency{#1.aux}%
}
\myexternaldocument{supplement}

\title{Hierarchical Gaussian Processes with Wasserstein-2 Kernels}

\author{
 Sebastian G. Popescu \\
  Imperial College London\\
  \texttt{s.popescu16@imperial.ac.uk} \\
   \And
  David J. Sharp \\
  Imperial College London \\
  \texttt{david.sharp@imperial.ac.uk} \\
  \And
  James H. Cole \\
  University College London \\
  \texttt{james.cole@ucl.ac.uk} \\
  \AND
  Ben Glocker \\
  Imperial College London \\
  \texttt{b.glocker@imperial.ac.uk} \\
}

\date{}

\renewcommand{\shorttitle}{Distributional Deep Gaussian Processes}

\hypersetup{
pdftitle={A template for the arxiv style},
pdfsubject={q-bio.NC, q-bio.QM},
pdfauthor={David S.~Hippocampus, Elias D.~Striatum},
pdfkeywords={First keyword, Second keyword, More},
}

\DeclareUnicodeCharacter{2212}{-}

\begin{document}
\maketitle

\begin{abstract}
    Stacking Gaussian Processes severely diminishes the model's ability to detect outliers, which when combined with non-zero mean functions, further extrapolates low non-parametric variance to low training data density regions. We propose a hybrid kernel inspired from Varifold theory, operating in both Euclidean and Wasserstein space. We posit that directly taking into account the variance in the computation of Wasserstein-2 distances is of key importance towards maintaining outlier status throughout the hierarchy. We show improved performance on medium and large scale datasets and enhanced out-of-distribution detection on both toy and real data.
\end{abstract}


\section{Introduction}

Gaussian Processes (GP) are data-efficient Bayesian non-parametric models that offer calibrated uncertainty quantification and are robust to overfitting, recently finding applicability in data-scarce domains \citep{timonen2019lgpr, wang2020statistical} or where uncertainty is of utmost importance \citep{chen2014gaussian}. Deep Gaussian Processes (DGP) \citep{damianou2013deep} are multi-layered generalizations of GPs that share data-efficiency learning capabilities and show enhanced representational capacity. However, the reliability of uncertainties propagated throughout the hierarchy is up for debate. Recent work \citep{ustyuzhaninov2019compositional} has questioned the uncertainties present in the hidden layers of DGP, showing that approximate inference schemes using factorized variational Gaussian distributions result in all but the last GP collapsing to deterministic transformations in the case of noise-free data. Such pathological behaviour should be avoided as it undermines the utility of multi-layered GP. We propose desideratas for useful properties that a kernel must have in a DGP and based on this we introduce an alternative kernel operating in both Euclidean and Wasserstein-2 space, showcasing that the explicit distance between multivariate Gaussians as present in hidden layers is of vital importance towards correctly propagating non-parametric uncertainty forward. Prior work \citep{bachoc2017gaussian, bachoc2018gaussian, thi2019distribution} has only focused on shallow GPs operating on inputs which do not have an explicit distribution, thus using Wasserstein-2 based kernels to take advantage of the distributional nature of the input. We did not find any other work that expanded the concept of Wasserstein-2 based kernels to DGPs. In the remainder of this section we introduce two main motivations behind this paper.

\paragraph{Distributional Uncertainty.}
\cite{kendall2017uncertainties} propose a dichotomy of classifying sources of uncertainty into aleatoric or epistemic, pending on their reducibility status. More specifically, aleatoric uncertainty is caused by inherent randomness in the data, hence it cannot be reduced. The classical example of coin flipping is the most obvious one to give, even though one can also argue that coin flipping is deterministic given a complex system taking into consideration strength of coin tosser, initial side, wind conditions and physics of coins tossing. Conversely, epistemic uncertainty refers to the epistemic state of the agent/model, or more colloquially the ignorance with respect to the current state. This ignorance can be reduced by feeding the agent/model with more information/data. Besides the dichotomy consisting of aleatoric and epistemic uncertainty, reliably highlighting certain inputs which have undergone a domain shift \citep{lakshminarayanan2017simple} or out-of-distribution (OOD) samples \citep{hendrycks2016baseline} has garnered a lot of interest in the past years. Succinctly, the aim is to measure the degree to which a model knows when it does not know, or more precisely if a network trained on a specific dataset is evaluated at testing time on a completely different dataset (potentially from a different modality or another application domain), then the expectation is that the network should output high predictive uncertainty on this set of data points that are very far from the training data manifold. We remind ourselves that epistemic uncertainty can be reduced by adding more data. By this logic, epistemic uncertainty cannot be reduced outside the data manifold of our dataset since we don't add data points which do not stem from the same data generative pipeline (this is not true in the case of OOD detection models which explicitly use OOD samples during training/testing \citep{liang2017enhancing, hafner2020noise}). Hence, epistemic uncertainty can only be reduced inside the data manifold and should be zero outside the data manifold (under the assumption that the model is \emph{distance-aware}, a concept subsequently introduced). Conversely, our choosen measure for OOD detection should grow outside the data manifold and be close to or 0 inside the data manifold. \cite{malinin2018predictive} introduced for the  first time the separation of total uncertainty into three components: epistemic, aleatoric and distributional uncertainty. To make the distinction clearer, the authors argue that aleatoric uncertainty is a "known-unknown", whereby the model confidently states that an input data point is hard to classify (class overlap). Contrary, distributional uncertainty is an "unknown-unknown" due to the fact that the model is unfamiliar with the input space region that the test data comes from, thereby not being able to make confident predictions.

\paragraph{Do Hierarchical GPs reliably propagate distributional uncertainty?} Since their inception, research on deep GPs has predominantly focused on enhancing their approximate inference framework. Little attention has been paid so far in how deep GPs fare on OOD detection. \cite{domingues2018deep} have proposed an autoencoder architecture with both encoder and decoder governed by deep GPs to detect outliers. However, this line of inquiry stems from reconstruction-based methods for outlier detection \citep{baur2021autoencoders}, whereby at training time an autoenconder is trained on normative data, subsequently at testing time the model is presented with OOD samples. The differences between reconstruction and OOD sample represent the regions in input space which are anomalous. Conversely, in our line of research we are interested in accurate predictive models which also have OOD detection capabilities. Deep GPs have thus so far not been tested if they are reliable OOD detection models. This will represent a major focus of this paper.

\paragraph{Contributions}
To summarize, this paper makes the following advancements:
\begin{itemize}
    \item We provide empirical evidence that as the number of inducing point increases, the non-parametric variance in hidden layers collapses in the case of zero-mean function DGP.
    \item We show that using Principal Component Analysis (PCA) mean functions solves this problem, but this procedure induces low variance perpendicular to the eigenvectors in the case of encoding down high-dimensional data.
    \item We derive theoretical requirements for DGPs to propagate outliers with high total uncertainty, respectively high non-parametric variance throughout the hierarchy.
    \item We propose a variant of DGPs that operates on both Euclidean and Wasserstein-2 space. We demonstrate that this formulation guarantees higher non-parametric variance for OOD points compared to in-distribution points, alongside increased chances of not suffering from \emph{feature-collapse}. 
    \item We empirically show enhanced out-of-distribution detection capabilities, alongside improved test time results on a wide range of datasets for our proposed model.
\end{itemize}

\section{Background}

In this section we introduce the necessary theoretical
background needed for this paper.

\subsection{Primer on Gaussian Processes} \label{sec:intro_gp}

A stochastic process assigns a random variable $f(x)$ to every element of an index set $\mathbb{X}$, more precisely $f = \{ f(x) \}_{x \in \mathbb{X}}$. Stochastic processes can be defined via their finite dimensional marginal distributions, more concretely the stochastic process can be defined as a random function $f:\mathbb{X} \to \mathbb{Y}$, with any finite sequence $x_{1:n} = \{x_{1}, \cdots, x_{n} \}$ having a joint marginal distribution over their values $y_{1:n} = \{y_{1}, \cdots, y_{n} \} = \{f(x_{1}), \cdots, f(x_{n}) \}$. Stochastic processes must respect two key conditions from the Kolmogorov Extension Theorem, which we now outline. Consider finite sequence $x_{1:n} \in \mathbb{X}$ with corresponding probability measure $p_{x_{1:n}}$ on $\left(\mathbb{R}^{D}\right)^{n}$.
\paragraph{Exchangeability.} For all permutations $\pi$ of $x_{1:n} = \{x_{1}, \cdots, x_{n} \}$ we impose $p_{x_{1:n}}(y_{1:n}) = p_{\pi\left(x_{1:n}\right)}(\pi\left(y_{1:n}\right))$ which translates into invariance of the joint distribution with respect to permutations of elements in any given finite sequence.

\paragraph{Consistency.} $p_{x_{1:m}}(y_{1:m}) = \int p_{x_{1:n}}(y_{1:n}) dy_{m+1:n}$ which ensures that if we marginalize a subset of elements, the remainder will remain unchanged.

Random functions are intrinsically complex mathematical objects with a Gaussian process representing the simplest random function due to its complete characterization by its mean and covariance functions \citep{pugachev2013theory}. A Gaussian Process can be seen as a generalization of multivariate Gaussian random variables to infinite sets. We define this statement in more detail now. We consider $f(x)$ to be a stochastic field, with $x \in \mathbb{R}^{D}$ and we define $m(x) = \mathbb{E}\left[ f(x)\right]$ and $C(x_{i}, x_{j}) = Cov\left[f(x_{i}), f(x_{j}) \right]$. We denote a Gaussian Process (GP) $f(x)$ as:
\begin{equation}
    f(x) \sim \mathcal{GP}\left(m(x), C\left(x_{i}, x_{j}\right) \right)
\end{equation}
Throughout this paper, we will generally consider zero-mean processes, hence our GPs are strictly characterized by their covariance functions. The latter have the condition to generate non-negative-definite covariance matrices, more specifically they have to satisfy:
\begin{equation}
    \sum_{i,j}a_{i}a_{j}C\left( x_{i}, x_{j} \right) \geq 0
\end{equation}
for any finite set $\{x_{1}, \cdots, x_{n} \}$ and any real valued coefficients $\{a_{1}, \cdots, a_{n} \}$. Throughout this paper we will only consider second-order stationary processes which have constant means and $Cov\left[f(x_{i}), f(x_{j}) \right] = C\left( \| x_{i} - x_{j} \| \right)$. We can see that such covariance functions are invariant to translations. 

The squared exponential/radial basis function kernel is a commonly used stationary kernel, which has the following form:
\begin{equation}
    k^{SE}(x_{i},x_{j}) = \sigma^{2} \exp{\left[\sum_{d=1}^{D}-\frac{\left(x_{i,d} - x_{j,d} \right)^{2}}{l^{2}} \right]}
\end{equation}
, where we have written its definition in the anisotropic case. Intuitively, the lengthscale values $\{l_{1}^{2}, \cdots, l_{D}^{2} \}$ represent the strength along a particular dimension of input space by which successive values are strongly correlated with correlation invariably decreasing as the distance between points increases. Such a covariance function has the property of Automatic Relevance Determination (ARD) \citep{neal2012bayesian}. However, the usage of ARD to determine the relative importance of a given feature in comparison to others has been recently challenged \citep{paananen2019variable}. Lastly, the kernel variance $\sigma^{2}$ controls the variance of the process, more specifically the amplitude of function samples. 

A GP has the following joint distribution over finite subsets $\mathbb{X}_{1} \in \mathbb{X}$ with function values $f(X_{1})$. Analogously for $\mathbb{X}_{2}$, with their union being denoted as $x = \{x_{1}, \cdots, x_{n} \}$.
\begin{equation}
    p\begin{pmatrix} 
        f(x_{1})     \\
        f(x_{2})
        \end{pmatrix} = \mathcal{N}\left[
        \begin{pmatrix} 
            m(x_{1})     \\
            m(x_{2})
        \end{pmatrix}
        ,
        \begin{pmatrix} 
            k(x_{1},x_{1}),  k(x_{1},x_{2})   \\
            k(x_{2},x_{1}),  k(x_{2},x_{2})
        \end{pmatrix}        
        \right]
\end{equation}

The following observation model is used:
\begin{equation}
    p(y|f,x) = \prod_{i=1}^{N}p(y_{i} \mid f(x_{i}))
\end{equation}
, where given a supervised learning scenario, the dataset $D = \{x_{i},y_{i} \}_{i=1,\cdots,n}$ can be shorthand denoted as $D = \{x, y\}$. In the case of probabilistic regression, we make the assumption that the noise is additive, independent and Gaussian, such that the latent function $f(x)$ and the observed noisy outputs $y$ are defined by the following equation:
\begin{equation}
    y_{i} = f(x_{i}) +\epsilon_{i} \text{, where }~ \epsilon_{i} \sim \mathcal{N}\left(0, \sigma^{2}_{noise} \right)
\end{equation}

To train a GP for regression tasks, one performs Marginal Likelihood Maximization of Type 2 over the following equation:
\begin{equation}
    p(y) = \mathcal{N}\left(y \mid m, K+\sigma^{2}_{noise}\mathbb{I}_{n}\right) 
\end{equation}
by treating the kernel hyperparameters as point-mass. In the case of priors being placed on the hyperparameters, Markov Chain Monte Carlo (MCMC) methods can be used \citep{lalchand2020approximate}.

We are interested in finding the posterior $p\left(f(x^{*}) \mid y \right)$ since the goal is to predict for unseen data points $x^{*}$ which are different than the training set. We know that the joint prior over training and testing set latent functions is given by:
\begin{equation}
    p\begin{pmatrix} 
        f(x)     \\
        f(x^{*})
        \end{pmatrix} = \mathcal{N}\left[
        \begin{pmatrix} 
            m(x)     \\
            m(x^{*})
        \end{pmatrix}
        ,
        \begin{pmatrix} 
            k(x,x) &  k(x,x^{*})   \\
            k(x^{*},x) &  k(x^{*},x^{*})
        \end{pmatrix}        
        \right]
\end{equation}
and using marginalization properties of multivariate normals we can then write the following : $f(x^{*}) \sim \mathcal{N}\left(m(x^{*}), k(x^{*}, x^{*}) \right)$. With this in mind, and also reiterating that the marginal likelihood is also a multivariate normal given by $y \sim \mathcal{N}\left(m(x), k(x,x,) + \sigma^{2}_{noise}\mathbb{I}_{n} \right)$, we can then write their joint prior since both of them are Gaussian:
\begin{equation}
    p\begin{pmatrix} 
        y     \\
        f(x^{*})
        \end{pmatrix} = \mathcal{N}\left[
        \begin{pmatrix} 
            m(x)     \\
            m(x^{*})
        \end{pmatrix}
        ,
        \begin{pmatrix} 
            k(x,x) + \sigma^{2}_{noise}\mathbb{I}_{n} &  k(x, x^{*})   \\
            k(x^{*},x) &  k(x^{*},x^{*})
        \end{pmatrix}        
        \right]
\end{equation}
Now we can simply apply the conditional rule for multivariate Gaussians to obtain:
\begin{align}
    p\left(f(x^{*}) \mid y \right) &= \mathcal{N}( f(x^{*}) \mid m(x^{*}) +K_{f^{*}f}\left[K_{ff}+\sigma^{2}_{noise}\mathbb{I}_{n}\right]^{-1}\left[ y - m(x)\right], \\ & \nonumber \hspace{1cm} K_{f^{*}f^{*}} - K_{f^{*}f}\left[K_{ff}+\sigma^{2}_{noise}\mathbb{I}_{n}\right]^{-1}K_{ff^{*}})
\end{align}

The usage of GP in real-life datasets is hindered by the $\mathbb{O}(n^{3})$ time, $\mathbb{O}(n^{2})$ memory for training, where $n$ is the number of data points in the training set. In the next subsection we will see how to avert having to incur these expensive computational budgets.

\subsection{Variational Free Energy approximation} \label{sec:svgp_background}

\cite{titsias2009variational} introduced the first variational lower bound comprising a probabilistic regression model over inducing points. More specifically, the authors applied variational inference (VI) in an augmented probability space that comprises training set latent function values $F$ alongside inducing point latent function values $U$:
\begin{equation}
    p(y,F,U) = p(y \mid F)p(F \mid U)p(U)
\end{equation}
, where the conditional of the training set latent functions based on inducing points is taken to be exact $p(F \mid U) = \mathcal{N}\left(F \mid K_{fu}K_{uu}^{-1}U, K_{ff} - K_{fu}K_{uu}^{-1}K_{uf} \right)$. 

In terms of doing exact inference in this new model, respectively computing the posterior $p(f|y)$ and the marginal likelihood $p(y)$, it remains unchanged even with the augmentation of the probability space by $U$ as we can marginalize $p(F) = \int p(F,U)~dU$ due to the marginalization properties of Gaussian processes. Succinctly, $p(F)$ is not changed by modifying the values of $U$, even though $p(F \mid U)$ and $p(U)$  do indeed change. This translates into the fundamental difference between variational parameters $U$ and hyperparameters of the model $\{\sigma^{2}_{noise}, \sigma^{2},l_{1}^{2}, \cdots, l_{D}^{2} \}$, whereby the introduction of more variational parameters does not change the fundamental definition of the model before probability space augmentation.

In \cite{titsias2009variational}, the authors have derived a collapsed variational lower bound which enties using the entire training set to derive the lower bound to the log marginal likelihood, which is used for training the hyperparameters $\{\sigma^{2}_{noise}, \sigma^{2}, l_{1}^{2}, \cdots, l_{D}^{2}\}$ of the GP, respectively for deriving the optimal posterior for $U$. Doing exact inference in the VFE formulation of sparse GP (SGP) is of sound theoretical interest, however it becomes infeasible for even medium sized datasets. In this subsection we will see how we can train VFE formulated SGP using minibatches by keeping an explicit representation over $q(U)$.

Stochastic Variational Inference (SVI) \citep{hoffman2013stochastic} enables the application of VI for extremely large datasets, by virtue of performing inference over a set of global variables, which induce a factorisation in the observations and latent variables, such as in the Bayesian formulation of Neural Networks with distributions (implicit or explicit) over matrix weights. GP do no exhibit these properties, but by virtue of the approximate prior over testing and training latent functions for SGP approximations with inducing points $U$, which we remind here for the VFE scenario:
\begin{equation}
    p(f,f^{*}) \approx q(f,f^{*}) = \int p(f\mid U) p(f^{*}\mid U) p(U)~dU
\end{equation}
This translates into a fully factorized model with respect to observations at training and testing time, conditioned on the global variables $U$. One can easily see that by integrating out $U$ in the Titsias bound, this key property for SVI dissipates. In the remainder of this subsection, we will derive a lower bound to the Titsias bound that satisfies the SVI requirements.

Our goal is to approximate the true posterior distribution $p(F,U \mid y) = p(F \mid U,Y)p(U \mid Y)$ by introducing a variational distribution $q(F,U)$ and minimizing the Kullback-Leibler divergence:
\begin{equation}
    KL\left[ q(F,U) \| p(F,U\mid y)\right] = \int q(F,U) \log \frac{q(F,U)}{p(F,U\mid y)}~dF~dU
\end{equation}
where the approximate posterior is factorized as $q(F,U) = p(F\mid U)q(U)$ and $q(U)$ is an unconstrained variational distribution over $U$. Following the standard VI framework we need to maximize the following variational lower bound on the log marginal likelihood:
\begin{align}
    \log p(y) &\geq \int p(F\mid U)q(U) \log\frac{p(y \mid F)p(F\mid U)p(U)}{p(F\mid U)p(U)}~dF~dU \\
    &\geq \int q(U) \left[\int \log p(Y\mid F)p(F\mid U)~dF + \log\frac{p(U)}{q(U)}\right]~dU
\end{align}
We can now solve for the integral over $F$:
\begin{align}
    \int \log p(y|F) p(F|U)~dF &=  \mathbb{E}_{p(F|U)} \left[-\frac{n}{2}\log (2\pi \sigma^{2}_{noise}) - \frac{1}{2\sigma^{2}_{noise}}Tr\left[yy^{\top} - 2yF^{\top} + FF^{\top}\right] \right] \\
    &= -\frac{n}{2}\log (2\pi \sigma^{2}_{noise}) - \frac{1}{2\sigma^{2}_{noise}}Tr[yy^{\top}  - 2y\left(K_{fu}K_{uu}^{-1}U\right)^{\top} + \\ & \nonumber \left(K_{fu}K_{uu}^{-1}U\right)\left(K_{fu}K_{uu}^{-1}U\right)^{\top} + K_{ff} - Q_{ff} ] \\
    &= \log \mathcal{N}\left(y|K_{fu}K_{uu}^{-1}U, \sigma^{2}_{noise}\mathbb{I}_{n} \right) - \frac{1}{2\sigma^{2}_{noise}}Tr\left[ K_{ff} - Q_{ff} \right]
\end{align}
We can now rewrite our variational lower bound as follows:
\begin{equation}
    \log p(y) \geq \int q(U) \log \frac{\mathcal{N}\left(y\mid K_{fu}K_{uu}^{-1}U, \sigma^{2}_{noise}\mathbb{I}_{n} \right) p(U)}{q(U)}~dU - \frac{1}{2\sigma^{2}_{noise}}Tr\left[ K_{ff} - Q_{ff} \right]
\end{equation}
The variational posterior is explicit in this case, respectively $q(F,U) = p(F \mid U;X,Z)q(U)$, where $q(U)= \mathcal{N}(U\mid m_{U},S_{U})$. Here, $m_{U}$ and $S_{U}$ are free variational parameters. Due to the Gaussian nature of both terms we can marginalize $U$ to arrive at $q(F) = \int p(F\mid U)q(U) = \mathcal{N}(F\mid \tilde{U}(x), \tilde{\Sigma}(x))$, where:
\begin{align}
    \tilde{U}(x) &= K_{fu}K_{uu}^{-1}m_{U} \label{eqn:posterior_mean_svgp} \\
    \tilde{\Sigma}(x) &= K_{ff} - K_{fu}K_{uu}^{-1} \left[ K_{uu}-S_{U} \right] K_{uu}^{-1}K_{uf} \label{eqn:posterior_variance_svgp}
\end{align}
The lower bound can be re-expressed as follows:
\begin{align}
    \log p(y) &\geq \int q(U) \log \mathcal{N}\left(y|K_{fu}K_{uu}^{-1}U, \sigma^{2}\mathbb{I} \right) ~dU -KL\left[q(U) \| p(U)\right] -  \frac{1}{2\sigma^{2}_{noise}}Tr\left[ K_{ff} - Q_{ff} \right] \\
    &\geq \mathcal{N}\left(y|K_{fu}K_{uu}^{-1}m_{U}, \sigma^{2}_{noise}\mathbb{I}_{n} \right) -  \frac{1}{2\sigma^{2}_{noise}}Tr\left[ K_{fu}K_{uu}^{-1}S_{U}K_{uu}^{-1}K_{uf} \right]  \\ & \nonumber \hspace{1cm} -  \frac{1}{2\sigma^{2}_{noise}}Tr\left[ K_{ff} - Q_{ff} \right] - KL\left[q(U) \| p(U)\right] = \mathcal{L}_{SVGP}
\end{align}
, where we can easily see that the last equation factorized with respect to individual observations. This lower variational bound will be denoted as the Sparse Variational GP (SVGP) or uncollapsed SVGP bound and will constitute a stepping-stone in training our proposed model.

\subsection{Approximation of Gaussian Processes with noisy inputs}

In subsection \ref{sec:intro_gp} we have seen that the predictive distribution of the output at a unseen noise-free input is Gaussian. In \cite{girard2004approximate}, the authors propose several approximation techniques devising approximations of the non-Gaussian predictive distribution when the input is noisy (or a random variable in a more general scenario as we shall later see). In this subsection we will attempt to make a summary of aforementioned work, keeping in mind that this approximation techniques will play a key role in analyzing properties of deep GP and how they propagate uncertainty through the hierarchy.

We now consider the case where the input is corrupted by noise $\epsilon_{x^{*}} \sim \mathcal{N}\left( 0, \Sigma_{x^{*}}\right)$ such that $x^{*} = \mu_{x^{*}} + \epsilon_{x^{*}}$. Hence, our aim is to make predictions at $x^{*} \sim \mathcal{N}\left( \mu_{x^{*}}, \Sigma_{x^{*}} \right)$, which involves the following integral:
\begin{equation}
    p\left(y^{*} \mid D, \mu_{x^{*}}, \Sigma_{x^{*}} \right) = \int p\left(y^{*} \mid D, x^{*} \right) p\left(x^{*} \mid \mu_{x^{*}}, \Sigma_{x^{*}} \right)~dx^{*}
\end{equation}
where $p\left(y^{*} \mid D, x^{*} \right) = \frac{1}{\sqrt{2\pi \sigma^{2}(x^{*})}} \exp{\left[ -\frac{1}{2}\frac{\left(y^{*} - \mu(x^{*}) \right)^{2}}{\sigma^{2}(x^{*})} \right]}$, respectively $\mu(x^{*}) = K_{f^{*}f}K_{ff}^{-1}y$ and $\sigma^{2}(x^{*}) = K_{f^{*}f^{*}} - Q_{f^{*}f^{*}}$. In this subsection, for notation purposes, we will consider $K_{x_{i},x_{j}} = K^{SE}(x_{i}, x_{j}) + \delta_{i,j}\sigma^{2}_{noise}$.
One can notice that $p\left(y^{*} \mid D, x^{*} \right)$ is a non-linear function of $x^{*}$, so we cannot integrate out $x^{*}$. We will now consider approximations to the solution of this integral.

\paragraph{\RNum{1} Monte Carlo Integration} uses random sampling of a function to numerically compute an estimate of its integral.
\begin{equation}
    p^{MC}\left( y^{*} \mid D, \mu_{x^{*}}, \Sigma_{x^{*}} \right) \approx \frac{1}{T}\sum\limits_{t=1}^{T} p\left( y^{*} \mid D, x_{t}^{*} \right) = \frac{1}{T}\sum\limits_{t=1}^{T}\mathcal{N}\left(y^{*} \mid \mu(x^{*}), \sigma^{2}(x^{*}) \right)
\end{equation}
It will result in a mixture of T Gaussians with equal mixing weights and as $T \to \infty$ we have $p^{MC}\left( y^{*} \mid D, \mu_{x^{*}}, \Sigma_{x^{*}} \right) \approx p\left( y^{*} \mid D, \mu_{x^{*}}, \Sigma_{x^{*}} \right)$. We will later see that the case $T=1$ will be default method to computing this type of integrals within the framework of Deep Gaussian Processes.

\paragraph{\RNum{2} Gaussian Approximation} involves approximating the non-Gaussian predictive equation with a Gaussian:
\begin{equation}
    p\left( y^{*} \mid D, \mu_{x^{*}}, \Sigma_{x^{*}} \right) = \mathcal{N}\left(  y^{*} \mid m(u_{x^{*}}, \Sigma_{x^{*}}), v(u_{x^{*}}, \Sigma_{x^{*}}) \right) \label{eqn:gaussian_approximation}
\end{equation}
We can now explicitly write the first two moments:
\begin{align}
    m(u_{x^{*}}, \Sigma_{x^{*}}) &= \int y^{*} \left[ \int p\left(y^{*} \mid D, x^{*} \right) p\left( x^{*} \mid u_{x^{*}}, \Sigma_{x^{*}}\right) dx^{*}\right] dy^{*} \\
    &= \int \left[ \int y^{*}  p\left(y^{*} \mid D, x^{*} \right) dy^{*} \right] p\left( x^{*} \mid u_{x^{*}}, \Sigma_{x^{*}}\right) dx^{*} \\
    &= \int \mu(x^{*}) p\left( x^{*} \mid u_{x^{*}}, \Sigma_{x^{*}}\right) dx^{*} \\
    &= \mathbb{E}_{x^{*}}\left[ \mu(x^{*})\right] \label{eqn:gaussian_approximation_mean}
\end{align}
\begin{align}
    v(u_{x^{*}}, \Sigma_{x^{*}}) &= \int \left(y^{*}\right)^{2} \left[ \int p\left(y^{*} \mid D, x^{*} \right) p\left( x^{*} \mid u_{x^{*}}, \Sigma_{x^{*}}\right) dx^{*}\right] dy^{*} -  m(u_{x^{*}}, \Sigma_{x^{*}})^{2} \\
    &= \int \left[ \int \left(y^{*}\right)^{2}  p\left(y^{*} \mid D, x^{*} \right) dy^{*} \right] p\left( x^{*} \mid u_{x^{*}}, \Sigma_{x^{*}}\right) dx^{*}  -  m(u_{x^{*}}, \Sigma_{x^{*}})^{2}  \\
    &= \int \left[\sigma^{2}(x^{*}) + \mu(x^{*})^{2} \right] p\left( x^{*} \mid u_{x^{*}}, \Sigma_{x^{*}}\right) dx^{*}  -  m(u_{x^{*}}, \Sigma_{x^{*}})^{2} \\
    &=  \mathbb{E}_{x^{*}}\left[\sigma^{2}(x^{*})\right] +  \mathbb{E}_{x^{*}}\left[\mu(x^{*})^{2} \right]  -  m(u_{x^{*}}, \Sigma_{x^{*}})^{2} \\
    &=  \mathbb{E}_{x^{*}}\left[\sigma^{2}(x^{*})\right] +  V_{x^{*}}\left[\mu(x^{*}) \right] \label{eqn:gaussian_approximation_variance} 
\end{align}
, where we now consider just predictive distributions over a single data point $x^{*}$, hence we update predictive equations to correspond to the univariate prediction case:
\begin{align}
    \mu(x^{*}) &= K_{x^{*}f}K_{ff}^{-1}y \\
    \sigma^{2}(x^{*}) &= K_{x^{*}x^{*}} - \sum\limits_{i,j=1}^{N} K_{ij}^{-1}K(x^{*},x_{i})K(x^{*},x_{j})
\end{align}

To recap, we are interested in finding the first two moments of the Gaussian approximation in equation \eqref{eqn:gaussian_approximation}, which we now write explicitly:
\begin{align}
    m(u_{x^{*}}, \Sigma_{x^{*}}) &= \mathbb{E}_{x^{*}} \left[ K_{x^{*}f}K_{ff}^{-1}y  \right]\\
    v(u_{x^{*}}, \Sigma_{x^{*}}) &= \mathbb{E}_{x^{*}} \left[ K_{x^{*}x^{*}} - \sum\limits_{i,j=1}^{N} K_{ij}^{-1}K(x^{*},x_{i})K(x^{*},x_{j})  \right] \\ & \nonumber 
    + \mathbb{E}_{x^{*}} \left[ \left[K_{ff}^{-1}y\right]_{i,:}\left[K_{ff}^{-1}y\right]_{j,:} K_{x^{*},x_{i}}K_{x^{*},x_{j}}  \right] \\ & \nonumber
    - m(u_{x^{*}}, \Sigma_{x^{*}})^{2}
\end{align}
, where $\left[K_{ff}^{-1}y\right]_{i,:}$ takes the i-th row. Therefore, we need to compute the following integrals:
\begin{align}
    \Psi = \mathbb{E}_{x^{*}}\left[K_{x^{*}, x^{*}} \right] &= \int K(x^{*},x^{*}) p(x^{*})~dx^{*} \\
    \Psi_{i} = \mathbb{E}_{x^{*}}\left[K_{x^{*}, x_{i}} \right] &= \int K(x^{*},x_{i}) p(x^{*})~dx^{*} \\
    \Psi_{i,j} = \mathbb{E}_{x^{*}}\left[K_{x^{*}, x_{i}} K_{x^{*}, x_{j}}\right] &= \int K(x^{*},x_{i})K(x^{*},x_{j}) p(x^{*})~dx^{*} 
\end{align}
These integrals will have an analytic solution in the case of specific kernels. We will first focus on the general case, whereby we devise approximations for any arbitrary kernel, subsequently calculating the exact moments for squared exponential kernels.

\subsubsection{Approximate Moments} \label{sec:approximate_moments_girard}

Moment approximations involves approximating the integrand by a Taylor expansion. 
\paragraph{General case.} To illustrate this method, we take a simple scenario where $x$ is a random variable such that $\mathbb{E}_{x}\left[ x \right] = \mu_{x}$ and $V_{x}\left[ x \right] = \Sigma_{x}$. For a well-behaved function $f$ with $y = f(x)$ and sufficiently small $\sigma_{x} = \sqrt{\Sigma_{x}}$ we can consider the Taylor expansion of $f(x)$ at $\mu_{x}$ up to second order to be an accurate representation of $f(x)$:
\begin{equation}
    y = f(x) \approx f(\mu_{x}) + \left( x - \mu_{x} \right) f^{'}(\mu_{x}) + \frac{1}{2} \left( x - \mu_{x} \right)^{2}f^{''}(\mu_{x})
\end{equation}
Using the Taylor expansion, we can calculate the first two moments as follows:
\begin{align}
    \mathbb{E}_{x}\left[ y \right] &\approx \mathbb{E}_{x} \left[f(\mu_{x})\right] + \mathbb{E}_{x} \left[ \left( x - \mu_{x} \right) f^{'}(\mu_{x})\right] + \mathbb{E}_{x} \left[ \frac{1}{2} \left( x - \mu_{x} \right)^{2}f^{''}(\mu_{x}) \right] \\ 
    &\approx \mathbb{E}_{x} \left[f(\mu_{x})\right] + \frac{1}{2} \Sigma_{x} f^{''}(\mu_{x}) \\
    V_{x}\left[ y \right] &\approx \mathbb{E}_{x}\left[ y^{2} \right] - \mathbb{E}_{x}\left[ y \right]^{2} \\
    &\approx f(\mu_{x})^{2} + \Sigma_{x}f^{'}(\mu_{x})^{2} + f(\mu_{x}) f^{''}(\mu_{x})\Sigma_{x} -  f(\mu_{x})^{2} - f(\mu_{x}) f^{''}(\mu_{x})\Sigma_{x} \\
    &\approx \Sigma_{x}f^{'}(\mu_{x})^{2}
\end{align}

\paragraph{GP with noisy inputs.} We now apply the above results to equations \eqref{eqn:gaussian_approximation_mean} and \eqref{eqn:gaussian_approximation_variance}, obtaining:
\begin{align}
    m(u_{x^{*}}, \Sigma_{x^{*}}) &= \mathbb{E}_{x^{*}}\left[ \mu(x^{*})\right] \\ 
    &= \mu(\mathbb{E}_{x^{*}}(x^{*})) +\frac{1}{2}V_{x^{*}}(x^{*}) \mu^{''}(\mathbb{E}_{x^{*}}\left[ x^{*} \right]) \label{eqn:girard_app_mean}\\
    v(u_{x^{*}}, \Sigma_{x^{*}}) &=  \mathbb{E}_{x^{*}}\left[\sigma^{2}(x^{*})\right] +  \mathbb{E}_{x^{*}}\left[\mu(x^{*})^{2} \right]  -  m(u_{x^{*}}, \Sigma_{x^{*}})^{2} \\
     &= \sigma^{2}(\mathbb{E}_{x^{*}}\left[ x^{*} \right] ) + V_{x^{*}}\left[\frac{1}{2}\sigma^{2}(\mathbb{E}_{x^{*}}\left[ x^{*}\right])^{''} + \mu^{'}(\mathbb{E}_{x^{*}}\left[ x^{*}\right])^{2} \right] \label{eqn:girard_app_var}
\end{align}

\subsubsection{Exact Moments} \label{sec:exact_moments_girard}

We can calculate the exact moments in the case of squared exponential kernels of the following form:
\begin{equation}
    K\left(x_{i}, x_{j} \right) = \sigma^{2}\exp{\left[-\frac{1}{2}\left[x_{i} - x_{j} \right]^{T} W^{-1} \left[x_{i} - x_{j} \right] \right]}
\end{equation}
, where in most cases $W = diag\left[l_{1}^{2}, \cdots, l_{D}^{2}\right]$.

We remind ourselves that we need to compute the following equations:
\begin{align}
    m\left( \mu_{x^{*}}, \Sigma_{x^{*}} \right) &= \sum\limits_{i=1}^{N} \left[y K_{ff}^{-1} \right]_{i,:} \Psi_{i} \\
    v\left( \mu_{x^{*}}, \Sigma_{x^{*}} \right) &= \Psi - \sum\limits_{i,j=1}^{N} \left[ K_{ij}^{-1} - \left[y K_{ff}^{-1} \right]_{i,:}\left[y K_{ff}^{-1} \right]_{j:} \right] \Psi_{ij} -  m\left( \mu_{x^{*}}, \Sigma_{x^{*}} \right)^{2}
\end{align}
We now need to derive the analytic formulas for $\Psi$, $\Psi_{i}$ and $\Psi_{ij}$. We will make use of the following identity.
\begin{definition}[Product of Multivariate Gaussians]
    For two multivariate Gaussian distributions, $\mathcal{N}\left(x \mid a, A \right)$ and $\mathcal{N}\left(x \mid b, B \right)$, their product can be expressed as follows:
    \begin{equation}
        zN\left(x \mid c, C \right) = \mathcal{N}\left(x \mid a, A \right)\mathcal{N}\left(x \mid b, B \right)
    \end{equation}
    where $c = C\left[A^{-1}a + B^{-1}b \right]$ and $C = \left[ A^{-1} + B^{-1} \right]^{-1}$
\end{definition}

Armed with this toolkit we can proceed:
\begin{align}
    \Psi &= \mathbb{E}_{x^{*}}\left[ K\left(x^{*}, x^{*} \right)\right] = \sigma^{2} \\
    \Psi_{i} &= \mathbb{E}_{x^{*}}\left[ K\left(x^{*}, x_{i} \right)\right] \\
    &= \int K\left(x^{*}, x_{i} \right) \mathcal{N}\left( x^{*} \mid \mu_{x^{*}}, \Sigma_{x^{*}}  \right) dx^{*} \\
    &= c \int \mathcal{N}\left( x^{*} \mid x^{i}, W \right) \mathcal{N}\left(x^{*} \mid \mu_{x^{*}}, \Sigma_{x^{*}}  \right) dx^{*} \\ 
    &= c \mathcal{N}\left(\mu_{x^{*}} \mid x_{i}, W + \Sigma_{x^{*}} \right) \\
    \Psi_{ij} &= \mathbb{E}_{x^{*}}\left[ K\left(x^{*}, x_{i} \right) K\left(x^{*}, x_{j} \right)\right] \\
    &= c^{2} \int \mathcal{N}\left(x^{*} \mid x^{i}, W \right) \mathcal{N}\left(x^{*} \mid x^{j}, W \right)   \mathcal{N}\left(x^{*} \mid \mu_{x^{*}}, \Sigma_{x^{*}}  \right) dx^{*} \\
    &= c^{2} \mathcal{N}\left(x_{i} \mid x_{j}, 2W \right) \int \mathcal{N}\left(x^{*} \mid \frac{x_{i} + x_{j}}{2}, \frac{W}{2}\right) \mathcal{N}\left(x^{*} \mid \mu_{x^{*}}, \Sigma_{x^{*}}  \right) dx^{*} \\
    &= c^{2} \mathcal{N}\left(x_{i} \mid x_{j}, 2W \right)  \mathcal{N}\left(\mu_{x^{*}} \mid \frac{x_{i} + x_{j}}{2}, \Sigma_{x^{*}} + \frac{W}{2} \right)
\end{align}
where $c = \left( 2\pi \right)^{\frac{D}{2}} \mid W \mid^{\frac{1}{2}} \sigma^{2}$. We can now write the full explicit form for the two exact moments:
\begin{align}
    m\left(\mu_{x^{*}}, \Sigma_{x^{*}} \right) &= \sum\limits_{i=1}^{N} \left[y K_{ff}^{-1} \right]_{i,:} c \mathcal{N}\left(\mu(x^{*}) \mid x_{i}, W + \Sigma_{x^{*}} \right) \\
    v\left(\mu_{x^{*}}, \Sigma_{x^{*}} \right) &= \sigma^{2} - c^{2} \sum\limits_{i,j=1}^{N} \left[y K_{ff}^{-1} \right]_{i,:}\left[y K_{ff}^{-1} \right]_{j,:}  \mathcal{N} \left(x_{i} \mid x_{j}, 2W \right) \\
    & \nonumber \hspace{1cm} \mathcal{N}\left(\mu_{x^{*}} \mid \frac{x_{i} + x_{j}}{2}, \Sigma_{x^{*}} + \frac{W}{2} \right) - m\left(\mu_{x^{*}}, \Sigma_{x^{*}} \right)^{2}
\end{align}
We follow \cite{quinonero2003prediction} and re-write the explicit first two moments as:
\begin{align}
    m\left(\mu_{x^{*}}, \Sigma_{x^{*}} \right) &= q^{T}K_{ff}^{-1}y \label{eqn:girard_exact_mean} \\
    v\left(\mu_{x^{*}}, \Sigma_{x^{*}} \right) &= \sigma^{2} + Tr\left[ \left(K_{ff}^{-1}y y^{T} K_{ff}^{-T} - K_{ff}^{-1}\right) Q\right] -  Tr\left[y^{T}K_{ff}^{-1}y \right] \label{eqn:girard_exact_var}
\end{align}
, where we have: 
\begin{align}
    q_{i} &= \mid W^{-1}\Sigma_{x^{*}} + \mathbb{I} \mid^{-\frac{1}{2}} \exp{\left[-\frac{1}{2} \left( \mu_{x^{*}} - x_{i} \right)^{T} \left(\Sigma_{x^{*}} + W \right)^{-1} \left( \mu_{x^{*}} - x_{i} \right) \right]} \\
    Q_{ij} &= \mid W^{-1}\Sigma_{x^{*}} + \mathbb{I} \mid^{-\frac{1}{2}} \exp\left[-\frac{1}{2} \left( \frac{x_{i} + x_{j}}{2}-\mu_{x^{*}}  \right)^{T} \left(\Sigma_{x^{*}} + \frac{1}{2}W \right)^{-1} \left( \frac{x_{i} + x_{j}}{2}-\mu_{x^{*}} \right) \right] \\ & \nonumber \hspace{1cm} \exp\left[-\frac{1}{2} \left( x_{i} - x_{j}  \right)^{T} \left( 2W \right)^{-1} \left( x_{i} - x_{j} \right) \right]
\end{align}

\subsection{Deep Gaussian Processes} \label{sec:dgp}

We can view DGPs as a composition of functions  
$f_{L}(x) = f_{L} \circ ... \circ f_{1}(x)$ with $f_{l} = \mathcal{GP}\left(m_{l}, k_{l}\left(\cdot, \cdot \right) \right)$.
Assuming a general likelihood function, we can write the joint prior as:
\begin{equation}
 p\left(y, \{f_{l} \}_{l=1}^{L}; X \right)=\underbrace{p(y|f_{L})}_{\text{likelihood}}\underbrace{\prod_{l=1}^{L} p(f_{l}|f_{l-1})}_{\text{prior}}
\end{equation}
with $p\left(f_{l} \mid f_{l-1}\right) \sim \mathcal{GP}\left(m_{l}(f_{l-1}), k_{l}\left(f_{l-1}, f_{l-1} \right) \right)$, where we use squared exponential kernels 
$k^{SE}_{l}(f_{l,i},f_{l,j}) = \sigma^{2}_{l} \exp{\left[\sum_{d=1}^{D_{l}}-\frac{\left(f_{l,i,d} - f_{l,j,d} \right)^{2}}{l^{2}_{l,d}} \right]}$, where $D_{l}$ represents the number of dimensions of $F_{l}$ and we introduce layer specific kernel hyperparameters $\theta_{l} = \{\sigma^{2}_{l}, l^{2}_{l,1}, \cdots, l^{2}_{l,D_{l}}\}$.

Analytically integrating this Bayesian hierarchical model is intractable as it requires integrating Gaussians present in a non-linear way. To enable faster inference over our model we can augment each layer $l$ with $M_{l}$ inducing locations $Z_{l-1}$, respectively inducing values $U_{l}$ resulting in the following augmented joint prior $p\left(y, \{f_{l} \}_{l=1}^{L}, \{U_{l} \}_{l=1}^{L}; X, \{Z_{l} \}_{l=0}^{L-1} \right)$:
\begin{equation}
 \underbrace{p(y|f_{L})}_{\text{likelihood}}\underbrace{\prod_{l=1}^{L} p(f_{l}|f_{l-1},U_{l};Z_{l-1})p(U_{l})}_{\text{prior}}
\end{equation}
 To perform stochastic variational inference \citep{hoffman2013stochastic} we introduce an approximate posterior: $q\left( \{U_{l} \}_{l=1}^{L}\right) = \prod\limits_{l=1}^{L} \mathcal{N}\left(U_{l} \mid m_{U_{l}}, S_{U_{l}} \right)$. Using a similar derivation as for sparse variational GPs (SVGP), we can arrive at the evidence lower bound:
\begin{align}
    \log{p(y)} &\geq \mathbb{E}_{q(\prod\limits_{l=1}^{L} q\left(f_{l} \mid f_{l-1} \right))}\left[\log{p\left( y \mid f_{L}\right)} \right] - \\ & \nonumber \hspace{1cm} \sum\limits_{l=1}^{L}KL\left[q(U_{l}) \| p(U_{l}) \right] = \mathcal{L}_{DGP}
\end{align}
, where $q\left(f_{l} \mid f_{l-1} \right) = \mathcal{N}\left(f_{l} \mid \tilde{U_{l}}(f_{l-1}), \tilde{\Sigma_{l}}(f_{l-1}) \right)$, respectively:
\begin{align}
    \tilde{U_{l}}(f_{l-1}) &= m_{l}(f_{l-1}) + K_{fu}K_{uu}^{-1}\left[m_{U_{l}} - m_{l}(Z_{l-1})\right] \label{eqn:svgp_posterior_mean} \\
    \tilde{\Sigma_{l}}(f_{l-1}) &= K_{ff} - K_{fu}K_{uu}^{-1}\left[K_{uu} - S_{U_{l}} \right]K_{uu}^{-1}K_{uf} \label{eqn:svgp_posterior_variance}
\end{align}
This composition of functions is approximated via Monte Carlo integration as in \cite{salimbeni2017doubly}.

\subsection{Wasserstein-2 distance kernels on Gaussian measures}

The Wasserstein space on $\mathbb{R}$ can be defined as the set $W_{2}(\mathbb{R})$ of probability measures on $\mathbb{R}$ with a finite moment of order two. We denote by $\Pi(\mu,\nu)$  the  set  of  all  probability measures $\Pi$ over the product set $\mathbb{R}\times\mathbb{R}$ with marginals $\mu$ and $\nu$, which are probability  distributions  in $W_{2}(\mathbb{R})$. The transportation cost between two measures $\mu$ and $\nu$ is defined as $T_{2}(\mu,\nu) = \inf_{\pi\in\Pi(\mu,\nu)}\int[x-y]^{2}d\pi(x,y)$. This transportation cost allows us to endow the set $W_{2}(\mathbb{R})$ with a metric by defining the quadratic Wasserstein distance between
$\mu$ and $\nu$ as $W_{2}(\mu,\nu) =T_{2}(\mu,\nu)^{1/2}$. More details on Wasserstein spaces and their links with optimal transport problems can be found in \cite{villani2008optimal}.

\begin{theorem}[Theorem \RNum{4}.1. in \cite{bachoc2017gaussian}] \label{thm:positive_definite}
Let $k_{W} : W_{2}(\mathbb{R}) \times W_{2}(\mathbb{R}) \rightarrow \mathbb{R}$ be the Wasserstein-2 RBF kernel defined as following:
\begin{equation}
k^{W_{2}}(\mu,\nu) = \sigma^{2} \exp \frac{-W_{2}^{2}(\mu,\nu)}{l^{2}}
\end{equation}
then $k^{W_{2}}(\mu,\nu)$ is a positive definite kernel for any $\mu,\nu \in  W_{2}(\mathbb{R})$, respectively $\sigma^{2}$ is the kernel variance, $l^{2}$ being the lengthscale. 
\end{theorem}
For completeness a proof following the layout in \cite{thi2019distribution} is provided in Appendix \ref{apd:pos_def_kernels}.

Multiplication of positive definite kernels results again in a positive definite kernel, hence we arrive at the automatic relevance determination kernel based on Wasserstein-2 distances: 
\begin{equation}
    k^{W_{2}}([\mu_{d}]_{d=1}^D,[\nu_{d}]_{d=1}^D) = \sigma^{2} \exp \sum_{d=1}^D \frac{-W_{2}^{2}(\mu_{d},\nu_{d})}{l_{d}^{2}}\label{eqn:wasserstein_kernel}    
\end{equation}

\paragraph{Wasserstein-2 Distance between Gaussian distributions} 

Gaussian measures fulfill the condition of finite second order moment, thereby being a clear example of probability measures for which we can compute Wasserstein metrics. The Wasserstein-2 distance between two multivariate Gaussian distributions $\mathcal{N}( m_{1},\Sigma_{1})$ and $\mathcal{N}( m_{2},\Sigma_{2})$, which have associated Gaussian measures and implicitly the Wasserstein metric is well defined for them, has been shown to have the following form
$\parallel m_{1} - m_{2} \parallel_{2}^{2}+
\Tr\Big[\Sigma_{1} + \Sigma_{2} - 2\Big(\Sigma_{1}^{1/2} \Sigma_{2}\Sigma_{1}^{1/2}\Big)^{1/2}\Big]$ \citep{Dowson1982TheFD}, which in the case of univariate Gaussians simplifies to $| m_{1} - m_{2} |^{2}+|\sqrt{\Sigma_{1}} - \sqrt{\Sigma_{2}} |^{2}$. This last formulation will be used throughout this paper.

\subsection{Varifold theory or kernels between generalized measure spaces} \label{apd:varifold_theory}

Varifolds were first introduced in Geometric Measure Theory  (GMT) \citep{almgren1966plateau}, in which shapes like submanifolds are represented in certain spaces of generalized measures such as currents, varifolds or normal cycles. This general framework was adapted to suit requirements for computational anatomy matching in \cite{charon2013varifold} and without loss of generality, throughout the remainder of this subsection we will focus solely on surfaces embedded in $\mathbb{R}^{3}$ and how to match them.

We commence by introducing elementary definitions from GMT. Subsequently introduced methods are capable of working with smooth surfaces alongside piecewise smooth or discrete objects. The umbrella term of rectifiable subsets in GMT encompasses the aforementioned scenarios. We denote $\mathbb{M} \in \mathbb{R}^{3}$ a rectifiable surface if for almost every $x \in \mathbb{M}$, there exists a tangent space $\mathbb{T}_{x}\mathbb{M} \in \mathbb{R}^{2}$.

To make the abstract mathematical framework more intuitive, we consider the problem of partial matching between shapes \citep{antonsanti2021partial}. Hence, the aim is to find the optimal deformation to register a source $S$ to a target shape $T$, where $S,T \in \mathbb{R}^{m}$, hence they are a finite union of $m$-dimensional submanifolds of the ambient space $E = \mathbb{R}^{d}$. We consider them to be m-rectifiable, which in the case of hyperspheres implies that any point $x \in S$ ($x \in T$) has a normal vector $\tau_{x}S \in S^{d-1}$ ($\tau_{x}T \in \mathbb{S}^{d-1}$).

\begin{definition}
    A varifold is a distribution on the product $\mathbb{R}^{k} \times \mathbb{S}^{k-1}$, namely a continuous linear form on a given space $W$ of smooth functions $\mathbb{R}^{k} \times \mathbb{S}^{k−1} \to \mathbb{R}$.
If $M$ is a surface, we associate the varifold $\{M\} \in W^{'}$, for any  function
$w \in W$, by the surface integral:
$\{M\}(w) = \int_{M} w\left( x, \tau_{x}T\right) dS(x)$
\end{definition}

\begin{proposition}[Proposition 4.1. in \cite{charon2013varifold}]
    Given positive-definite real kernel $k_{e}$ on $\mathbb{R}^{d}$ s.t. $k_{e}$ is continuous and for all $x\in \mathbb{R}^{d}$, $k_{e}(x,\cdot)$ vanishes at infinity, respectively a positive-definite real kernel $k_{t}$ defined on the manifold $\mathbb{S}^{d-1}$ that is also continuous. Then the Reproducing Kernel Hilbert Space (RKHS) $W$ associated to the positive-definite kernel $k_{e} \otimes k_{t}$ is continuously embedded into the space $C_{0}\left(\mathbb{R}^{d} \times \mathbb{S}^{d-1} \right)$, the space of continuous functions on $\mathbb{R}^{d} \times \mathbb{S}^{d-1}$ that decay to 0 at infinity.
\end{proposition}
Its dual space $W^{'}$ is the space of varifolds. The reproducing kernel of $W$ is given by $k_{e} \otimes k_{t}$, with $k_{e}(x,y) = \exp{\left(-\mid x - y \mid^{2}\right)} / \sigma_{w}^{2}$ and $k_{t}(u,v) = \exp{\langle u,v\rangle_{R}^{d}}$ and we associate the following canonical function $w_{S} \in W$:
\begin{equation}
    w_{S}(y, \tau) = \int_{S} k_{e}(y,x)k_{t}(\tau, \tau_{x}S)dx
\end{equation}
We can finally define the scalar product between canonical functions $w_{S}$ and $w_{T}$ as follows:
\begin{equation}
    \langle w_{S}, w_{T} \rangle = \int_{S}\int_{T} k_{e}(x,y) k_{t}(\tau_{x}S, \tau_{y}T) dx~dy \label{varifold_kernel}
\end{equation}

\begin{remark}
 Intuitively, the inner product associated to RKHS $W$ keeps track of both Euclidean differences between 3D points on the surface mesh, but also takes into account geometric information surrounding those points, as given by the points' surface normal vectors (Figure \ref{fig:varifold_hybrid_kernels}). 
\end{remark}

\begin{figure}[!htb]
    \centering
    \includegraphics[width=0.95\linewidth]{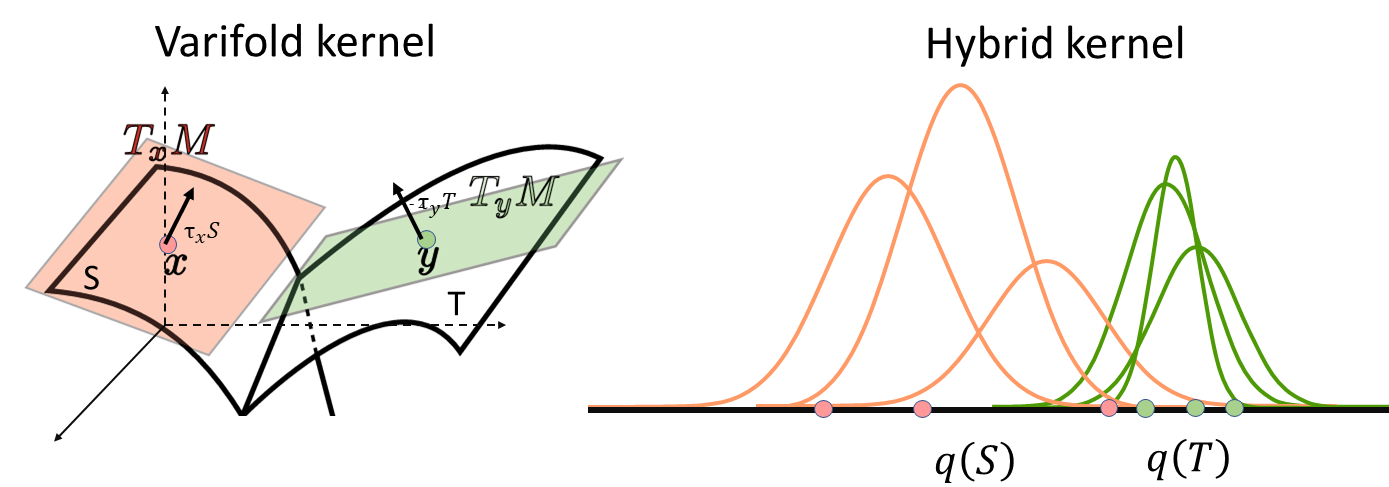}
    \caption[Similarities between Varifold kernel and Hybrid kernel.]{\textbf{Similarities between Varifold kernel and Hybrid kernel.} \textbf{Left:} Comparing two manifolds involves taking 3D points in Euclidean space (denoted by circles) and their associated surface normal vectors (denoted by directed arrows), subsequently computing their distance in their respective domains. We only show the comparison between two points for aesthetical reasons, however it should be noted that in practice a multitude of points and associated surface normal vectors are used; \textbf{Right:} Considering two data points $S$ and $T$, their predictive distribution at a certain hidden layer of a DDGP would consist of mixture of Gaussians (p.d.f. of mixture components are depicted either in pink or green). The hybrid kernel takes samples from mixture components and their underlying Gaussian distributions.}
    \label{fig:varifold_hybrid_kernels}  
\end{figure}

\section{Propagation of variance in DGP} \label{sec:theoretical_analysis_zero_mean_deepgp}

\cite{ustyuzhaninov2019compositional} have proved for noiseless data that the various stochastic layers in the hierarchy of a DGP collapse to deterministic transformations. In this section, we analyze what are the necessary theoretical requirements for a zero-mean DGP to maintain comparatively higher total uncertainty for an OOD data point throughout the hierarchy and draw some insights into what type of uncertainty (e.g., parametric/epistemic or non-parametric/distributional) is propagate forward in each case. We analyze two different case studies that might occur in practice. The first one involves inducing points $\{Z_{0}, \cdots, Z_{L} \}$ that are evenly spread around 0 with a maximum range of $\left[-3\sigma_{l}, 3\sigma_{l}\right]$ at each layer such that samples stemming from the previous layers are within the neighbourhood of said inducing points. Secondly, we will consider the scenario where we have two separate groups of inducing points $Z_{l,1} \subsetneq \left[-\infty, -c \right]$ and $Z_{l,2} \subsetneq \left[c, \infty \right]$ for  $c \in \mathbb{R^{+}}$. This situation can occur in practice for a binary classification task, where at the respective layer, the inducing points can reliably separate the two classes. We now detail what are the necessary conditions in the first scenario to predict higher total uncertainty for OOD data points in comparison to in-distribution points.

\begin{figure}[!htb]
  \centering
    \includegraphics[width=\linewidth]{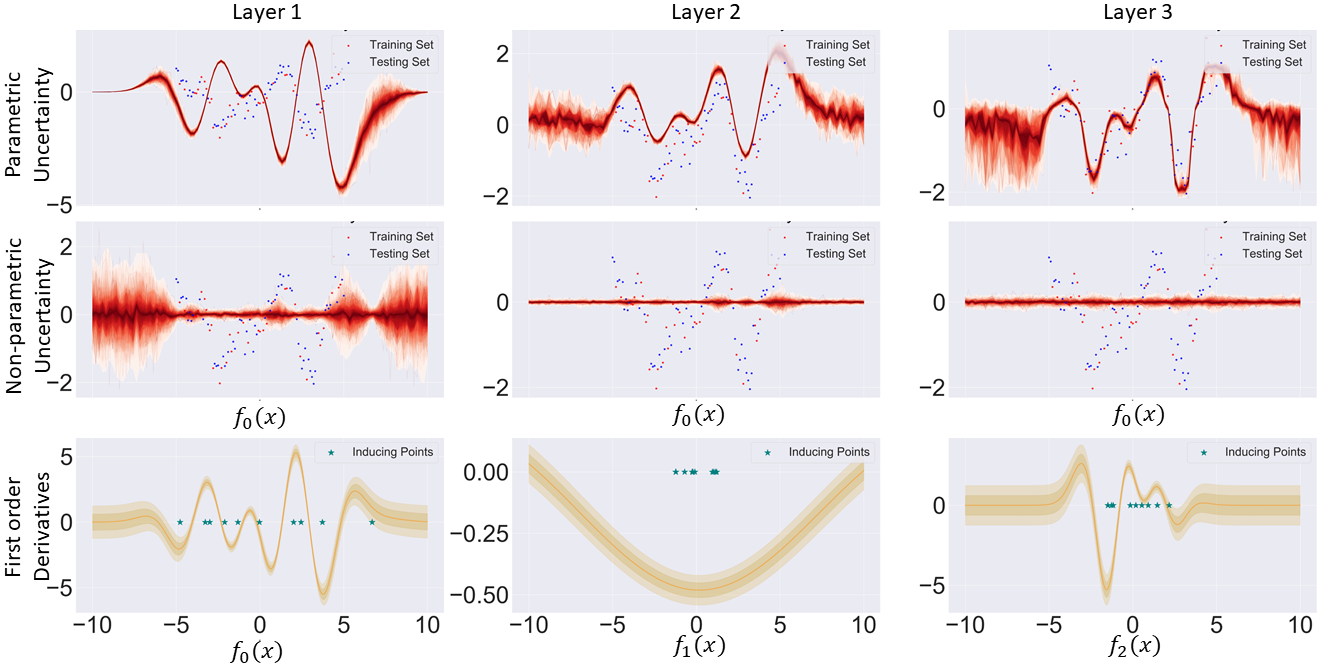}
    \caption[Collapse of non-parametric variance in DGP]{\textbf{Collapse of non-parametric variance in DGP.} Layer-wise decomposition of uncertainty into parametric and non-parametric for a zero-mean DGP, alongside first order derivatives.}
    \label{fig:toy_regression_derivatives}  
\end{figure}

\begin{proposition} \label{thm:outlier_req_dgp}
    We consider the approximate posterior DGP as a composition of functions $q_{L} = q_{L} \circ ... \circ q_{1}$ with $q_{l}$ being given by $
    q(F_{l}|F_{l-1}) = \int p(F_{l}|U_{l},F_{l-1})q(U_{l}) = \mathcal{N}(F_{l} \mid \tilde{U_{l}}(F_{l-1}), \tilde{\Sigma_{l}}(F_{l-1}))
    $. We consider an OOD data point $x_{out}$ in input space such that $\tilde{\Sigma_{1}}(x_{out}) = \sigma^{2}_{1}$ and $\tilde{U_{1}}(x_{out}) = 0$, respectively $x_{in-d}$ to be in-distribution, with $\tilde{\Sigma}_{1}(x_{in-d}) = V_{in-d} \leq \sigma^{2}_{1}$ and $\tilde{U}_{1}(x_{in-d}) = M_{in-d}$. We assume that $Z_{2}$ are equidistantly placed between $[-3\sigma_{1},3\sigma_{1}]$. The variance in the second layer of $x_{out}$ will be higher than $x_{in-d}$ if the following holds:
    \begin{equation}
     \frac{ \left(\frac{\partial \tilde{U}_{2}(F_{1})}{\partial F_{1}}\right)^{2}\Bigr|_{\substack{F_{1}=M_{in-d}}}}{\left(\frac{\partial \tilde{U}_{2}(F_{1})}{\partial F_{1}}\right)^{2}\Bigr|_{\substack{F_{1}=0}}} \leq \frac{\sigma^{2}_{1}}{V_{in-d}}
\end{equation}
\end{proposition}

\begin{proof}

We consider the approximate posterior DGP as a composition of functions $q_{L} = q_{L} \circ ... \circ q_{1}$ with a certain $q_{l}$ being given by:
\begin{equation}
q(F_{l}|F_{l-1}) = \int p(F_{l}|U_{l},F_{l-1})q(U_{l}) = \mathcal{N}(F|\tilde{U_{l}}(f_{l-1}), \tilde{\Sigma_{l}}(f_{l-1}))
\end{equation}

Following the layout introduced in \cite{girard2002gaussian} for obtaining Gaussian approximations of GP with uncertain inputs, we have the following adapted case for a two layer DGP, with the forthcoming analysis being easily extended to more layers: U
\begin{equation}
    q(F_{2})(x)  = \int p(F_{2}|F_{1}) q(F_{1}(x))~dF_{1}
\end{equation}
Using the framework described in subsection \ref{sec:approximate_moments_girard} (equations \eqref{eqn:girard_app_mean} and \eqref{eqn:girard_app_var}) and adapting to our case at hand, we have the following approximate moments for $q(F_{2})(x)$:
\begin{align}
    m(F_{2}(x)) &= \tilde{U}_{2}(\tilde{U}_{1}(x)) \\
    v(F_{2}(x)) &=  \tilde{\Sigma}_{2}(\tilde{U}_{1}(x)) + \tilde{\Sigma}_{1}(x) \Bigg[ \frac{1}{2} \frac{\partial^{2} \tilde{\Sigma}_{2}(F_{1})}{\partial^{2} F_{1}}\Bigr|_{\substack{F_{1}=\tilde{U}_{1}(x)}}  + \left(\frac{\partial \tilde{U}_{2}(F_{1})}{\partial F_{1}}\right)^{2}\Bigr|_{\substack{F_{1}=\tilde{U}_{1}(x)}} \Bigg] \label{eqn:approximation_variance}
\end{align}
, where for the purposes of our derivation here we only considered a first order Taylor expansion for the approximate mean. In \cite{ustyuzhaninov2019compositional}, the authors have empirically shown that $\lim_{M\to\infty} \frac{\partial^{2} \tilde{\Sigma_{2}}(F_{1})}{\partial^{2} F_{1}} = 0$, hence that term is dropped. Intuitively, it follows a sinusoidal path with the amplitude converging to 0 as the number of inducing points is increased.

We remind ourselves that in the proposition statement we made the following set of assumptions: inducing point locations for the second layer $Z_{1}$ are equidistantly placed in the interval $[-3\sigma_{1},3\sigma_{1}]$; the variance of the inducing point values $S_{U_{2}}$ are taken to be equal to simplify the problem.

For $x_{out}$ and $x_{in-d}$ using equation \ref{eqn:approximation_variance} we obtain:
\begin{align}
    v(F_{2}(x_{out})) &= \tilde{\Sigma}_{2}(0) + \sigma^{2}_{1}  \left(\frac{\partial \tilde{U}_{2}(F_{1})}{\partial F_{1}}\right)^{2}\Bigr|_{\substack{F_{1}=0}}
    \\
    v(F_{2}(x_{in-d})) &= \tilde{\Sigma}_{2}(M_{in-d}) + V_{in-d}  \left(\frac{\partial \tilde{U}_{2}(F_{1})}{\partial F_{1}}\right)^{2}\Bigr|_{\substack{F_{1}=M_{in-d}}}
\end{align}
The desired behaviour for the second layer is: 
\begin{equation}
    \tilde{\Sigma}_{2}(M_{in-d}) + V_{in-d} \left(\frac{\partial \tilde{U}_{2}(x)}{\partial x}\right)^{2}\Bigr|_{\substack{F_{1}=M_{in-d}}} \leq \tilde{\Sigma}_{2}(0) + \sigma^{2}_{1} \left(\frac{\partial \tilde{U}_{2}(x)}{\partial  x}\right)^{2}\Bigr|_{\substack{F_{1}=0}}    
\end{equation}

Since we have assumed that the inducing point locations are equidistantly placed in a grid with centre zero, which is a sufficiently similar situation to what one might observe in practice, then $\tilde{\Sigma}_{2}(0) \approx \tilde{\Sigma}_{2}(M_{in-d})$ as we have also assumed that the posterior variances of the inducing point values are equal. Hence, the inequality $ v(F_{2}(x_{in-d})) \leq v(F_{2}(x_{out}))$ holds if and only if:
\begin{equation}
     \frac{ \left(\frac{\partial \tilde{U}_{2}(F_{1})}{\partial F_{1}}\right)^{2}\Bigr|_{\substack{F_{1}=M_{in-d}}}}{\left(\frac{\partial \tilde{U}_{2}(F_{1})}{\partial F_{1}}\right)^{2}\Bigr|_{\substack{F_{1}=0}}} \leq \frac{\sigma^{2}_{1}}{V_{in-d}}
\end{equation}
\end{proof}

\begin{remark}
    Intuitively, this inequality can be easily satisfied if $\left(\frac{\partial \tilde{U}_{2}(F_{1})}{\partial F_{1}}\right)^{2}\Bigr|_{\substack{F_{1}=0}}$ reaches its absolute peak at $F_{1}=0$ since $1 \leq \frac{\sigma^{2}_{1}}{V_{in-d}}$ holds in general. However, this requirement is not guaranteed to be satisfied in practice. On a toy regression task we can observe that the overall uncertainty is correctly predicted as higher for OOD points due to high absolute values of the derivative at 0 (i.e., where OOD points get mapped in hidden layers on average) in hidden layers' space (Figure \ref{fig:toy_regression_derivatives}).
\end{remark}

In the first investigated case, we have thus established the necessary condition for the reliable propagation of total variance so as to guarantee that OOD data points have higher total uncertainty compared to in-distribution data points at each layer in the hierarchy. However, a reasonable question is whether DGPs in this case are \emph{distance-aware} (see Definition \ref{def:distance_awareness}) at each layer? From Figure \ref{fig:toy_regression_derivatives} we can see that the non-parametric uncertainty (which stands as a proxy for distributional uncertainty; the two terms will be used interchangeably throughout the remainder of the paper) collapses to zero across the whole spectrum of data points in the hidden layers. Therefore, DGPs in this scenario are not \emph{distance-aware}, effectively rendering the non-parametric part of layer-specific SVGP obsolete, thus resulting in a fully parametric DGP from the first hidden layer up to the output layer (for in-depth insight into why this pathology occurs we refer the reader to Figure \ref{fig:visual_explanation_collapse_of_variance} and accompanying text). The second scenario for DGPs will investigate the case where the non-parametric uncertainty for an OOD data point does get correctly propagated throughout the layers, thus making use of the non-parametric part of SVGP in each layer of the hierarchy. We lay out the necessary conditions in the following proposition. Before we proceed, we remind ourselves that a SVGP can be decomposed as $F(\cdot) \sim \underbrace{K_{fu}K_{uu}^{-1}U}_{G(\cdot)} + \underbrace{\mathcal{N}\left(0, K_{ff \cdot u} \right)}_{H(\cdot)}$.

\begin{proposition} \label{thm:safe_outlier_propagation_dgp}
    Under the assumption that for any OOD point in input space $x_{out}$ we have $G_{l}(x_{out}) = K_{fu}K_{uu}^{-1}U_{l} \approx 0$ at each layer in a DGP, we can instead focus solely on the non-parametric components of DGP as a composition of $p(H_{L}) = p(H_{L} \mid H_{L-1}) \circ ... \circ p(H_{1} \mid H_{0})$ with $p\left(H_{l} \mid H_{l-1} \right) = \mathcal{N}\left(H_{l} \mid 0, K_{ff \cdot u} \right)$ and $H_{0} = X$.  We consider two inducing points in the first hidden layer $\{Z_{1,1} =-c, Z_{1,2} = c\}$ with $c \geq 0$ and an OOD data point $x_{out}$ in input space such that $p(H_{1}(x_{out})) \sim \mathcal{N}\left(0, \sigma^{2}_{1} \right)$. The approximate variance of the non-parametric component pertaining to the second layer will satisfy $v\left(H_{2}(x_{out})\right) \geq K_{ff \cdot u}$ if the following holds $l^{2}_{2} \geq 2c^{2}$, where the Schur complement is evaluated at $F_{1} = 0$. 
\end{proposition}

\begin{proof}
    We commence by writing our Bayesian hierarchical formulation so that the link with approximation of Gaussian Processes' moments with noisy inputs (see section \ref{sec:approximate_moments_girard}) is made clearer. Under our set of assumptions we have:
    \begin{align}
        p(H_{1} \mid H_{0}) &= \mathcal{N}\left(H_{1} \mid 0, K_{ff} - K_{fu}K_{uu}^{-1}K_{uf}\right) \\
        &= \mathcal{N}\left(H_{1} \mid 0, \sigma^{2}_{1} \right) \\
        p(H_{2} \mid H_{1}) &= \int \mathcal{N}\left(H_{2} \mid 0, K_{ff} - K_{fu}K_{uu}^{-1}K_{uf}\right) \mathcal{N}\left(H_{1} \mid 0, \sigma^{2}_{1} \right)~dH_{1}
    \end{align}
    We can approximate the non-Gaussian predictive equation using the framework of exact moments introduced in subsection \ref{sec:exact_moments_girard} since we are only considering squared exponential kernels. We adapt equation \eqref{eqn:girard_exact_var} to suit our scenario:
    \begin{align}
        v\left(0, \sigma^{2}_{1}\right) = \sigma^{2}_{2} -  Tr\left[K_{uu}^{-1} Q\right]
    \end{align}
    where we have:
    \begin{align}
        K &= \begin{pmatrix} 
        \sigma^{2}_{2} & \sigma^{2}_{2} \exp{\left[ -\frac{2c^{2}}{l^{2}}\right]} \\
        \sigma^{2}_{2} \exp{\left[ -\frac{2c^{2}}{l^{2}}\right]} & \sigma^{2}_{2}
    \end{pmatrix} \\
        Q &= \frac{1}{\sqrt{\frac{2\sigma^{2}_{1}}{l^{2}} +1}}
    \begin{pmatrix}
        \exp{\left[ - \frac{c^{2}}{l^{2} + 2\sigma^{2}_{1}}\right]} & \exp{\left[-\frac{c^{2}}{l_{2}} \right]} \\
        \exp{\left[-\frac{c^{2}}{l_{2}} \right]} & \exp{\left[ - \frac{c^{2}}{l^{2} + 2\sigma^{2}_{1}}\right]}
    \end{pmatrix}
    \end{align}
    
We can now write the exact variance of the Gaussian approximation of $p(H_{2})$ as:
\begin{equation}
    v\left(0, \sigma^{2}_{1}\right) = \sigma^{2}_{2} - \frac{2\left[ \exp{\left[-\frac{c^{2}}{l^{2}+2\sigma^{2}_{1}}\right]} - \exp{\left[-\frac{3c^{2}}{l^{2}} \right]}\right]}{\sigma^{2}_{2}\left[1 - \exp{\left(- \frac{4c^{2}}{l^{2}} \right)} \right] \sqrt{\frac{2\sigma^{2}_{1}}{l_{2}}+1}}
\end{equation}
which has the following derivative with respect to $\sigma^{2}_{1}$:
\begin{equation}
    \frac{\partial v\left(0, \sigma^{2}_{1}\right)}{\partial \sigma^{2}_{1}} =   \sigma^{2}_{2}\left[1 - \exp{\left(- \frac{4c^{2}}{l^{2}} \right)} \right] 2 \exp{\left[ -\frac{c^{2}}{l^{2}+ 2\sigma^{2}_{1}} \right]}  \left[ \frac{2\sigma^{2}_{1} +l^{2}}{l^{2}}\right]^{-\frac{1}{2}} \left[2c^{2} \left[l^{2} + 2\sigma^{2}_{l-1} \right] - \frac{1}{2\sigma^{2}_{1}+l^{2}}  \right]
\end{equation}
We first consider the derivative evaluated at $\sigma^{2}_{1} = 0$:
\begin{equation}
    \frac{\partial v\left(0, \sigma^{2}_{1}\right)}{\partial \sigma^{2}_{1}}\Bigr|_{\substack{\sigma^{2}_{1}=0}} =   \sigma^{2}_{2}\left[1 - \exp{\left(- \frac{4c^{2}}{l^{2}} \right)} \right] \left[ - \frac{c^{2}}{l^{2}} \right] \left[ \frac{2c^{2}}{l^{4}} - \frac{1}{l^{2}}\right]
\end{equation}
The first two terms are always positive in the case that $K_{uu}$ is invertible, leaving us with just the last term. Hence, we have that:
\begin{equation}
    \frac{\partial v\left(0, \sigma^{2}_{1}\right)}{\partial \sigma^{2}_{1}}\Bigr|_{\substack{\sigma^{2}_{1}=0}} \geq 0 \iff l^{2} \geq 2c^{2}
\end{equation}
Similarly, one can investigate whether the derivative is positive at any other value except 0:
\begin{equation}
    \frac{\partial v\left(0, \sigma^{2}_{1}\right)}{\partial \sigma^{2}_{1}}\Bigr|_{\substack{\sigma^{2}_{1}=\tilde{\sigma^{2}_{1}}}} \geq 0 \iff \tilde{\sigma^{2}_{1}} + l^{2} \geq 2c^{2}
\end{equation}
One can easily see that at any value of $\sigma^{2}_{1}$ the derivative will be positive only if $l^{2} \geq 2c^{2}$. Hence, the minimum values will be at $v(0,0) = K_{ff} - Q_{ff}$, where the Schur complement is evaluated at point $F_{1} = 0$.
\end{proof}

\begin{remark}
    Reliable propagation of non-parametric uncertainty for OOD data points in DGP can only occur if the parametric posterior at each layer $\mathcal{N}\left(G_{l} \mid K_{fu}K_{uu}^{-1}m_{U_{l}}, K_{fu}K_{uu}^{-1}S_{U_{l}}K_{uu}^{-1}K_{uf} \right)$ manages to reliably map in-distribution points stemming from the previous layer in the intervals $I_{in-d}^{1} = [-\infty, -3\sigma_{l-1}]$, respectively  $I_{in-d}^{2} = [3\sigma_{l-1}, \infty]$. This scenario can occur for example in simple binary classification tasks, where the parametric posteriors at each layer have learned to reliably separate the two classes in hidden layer features. Considering that OOD data points from the previous layer have a predictive distribution $\mathcal{N}\left(0,\sigma^{2}_{l-1} \right)$, upon sampling $95\%$ of samples from this distribution will be in the interval $I_{OOD} = [-3\sigma_{l-1},3\sigma_{l-1}]$. Under an assumption of well-optimized inducing points' locations, $Z_{l-1}$ will invariably find themselves in either $I_{in-d}^{1}$ or $I_{in-d}^{2}$, with no locations in $I_{OOD}$. Using notation from Proposition \ref{thm:safe_outlier_propagation_dgp}, $c = 3\sigma_{l-1}$ and the approximate variance $v\left(0, \sigma^{2}_{l-1} \right) \geq K_{ff \cdot u}$ where the Schur complement is evaluated at $F_{l-1} = 0$ only if $l_{l}^{2} \geq 2c^{2}$. 
\end{remark}

To summarize our results so far, in this section we have shown the necessary conditions for DGPs to maintain a comparatively higher total uncertainty for OOD points in comparison to in-distribution points. Through empirical examples, we have seen that the non-parametric variance collapses to zero even in toy regression datasets, thus making the DGP rely solely on its parametric components. Next, we have proposed a scenario in which the posterior DGP can reliably propagate non-parametric variance for OOD points. However, this latter scenario is not guaranteed to occur in practice. In the next section, we propose a variant of DGPs which are guaranteed to propagate non-parametric variance for OOD points.

\section{Distributional Deep Gaussian Processes}

In this section we commence with a motivation behind the hybrid kernel (Euclidean $\&$ Wasserstein space), subsequently introducing the generative process and inference framework for our newly introduced model. Moreover, we derive the theoretical requirements of correctly propagating outliers through the hierarchy of this probabilistic formulation and show that it's easier satisfied in comparison to DGPs.

\subsection{Desiderata for kernels in DGP}

\paragraph{\RNum{1} Maintenance of outlier status.} While there are many definitions of what an outlier constitutes \citep{ruff2021unifying}, for the purposes of this work we follow the most basic one, respectively \textit{``An anomaly is an observation that deviates considerably from some concept of normality.''} More concretely, it can be formalised as follows: our data resides in $X \in \mathbb{R}^{D}$, an anomaly/outlier is a data point $x \in X$ that lies in a low probability region under $\mathcal{P}$ such that the set of anomalies/outliers is defined as $A  =\{x \in X \mid p(x) \leq \xi \}, ~ \xi \geq 0$, with $\xi$ is a threshold under which we consider data points to deviate sufficiently from what normality constitutes. GP constitute an easy plug-in method to define what normality constitutes, or any other kernel-based algorithm for that matter, due to the mathematical formalism embedded in the kernel.

\begin{definition}[Definition 1 in \cite{liu2020simple}] \label{def:distance_awareness}
    We consider the predictive distribution for unseen point $p\left( y^{*} \mid x^{*} \right)$ at testing time, for model trained on $\mathbb{X}_{in-d} \in \mathbb{X}$, with the data manifold being equipped with metric $\| \cdot \|_{X}$. Then, we can affirm that $p\left( y^{*} \mid x^{*} \right)$ is \emph{distance-aware} if there exists a summary statistic $u\left( x^{*} \right)$ of $p\left( y^{*} \mid x^{*} \right)$ that embeds the distance between $\mathbb{X}_{in-d}$ and $x^{*}$:
    \begin{equation}
        u\left( x^{*} \right) = v\left[ \mathbb{E}_{x \sim \mathbb{X}_{in-d}} \left[ \| x^{*} - x \|_{X}^{2}\right]\right]
    \end{equation}
    where $v$ is a monotonic function that increases with distance.    
\end{definition}

\paragraph{GP predictive variance as distributional uncertainty.}
GPs are clearly \emph{distance-aware} provided we use a translation-invariant kernel. The summary statistics for an unseen point is given by $u(x^{*}) = K_{f^{*}f^{*}} - K_{f^{*}f}K_{ff}^{-1}K_{ff^{*}}$, which is monotonically increasing as a function of distance. Throughout this paper, we will use the non-parametric variance of sparse GPs as a proxy for distributional uncertainty, which will be used to assess if inputs are in or outside the distribution.

Definition \ref{def:distance_awareness} does not make any assumptions related to the architecture of the model from which the predictive distribution stems. In practice we would have the following composition to arrive at the logits $logit(x^{*}) = f \circ enc\left( x^{*}\right)$, where $enc\left( \cdot \right)$ represents a network that outputs the representation learning layer and $f\left( \cdot \right)$ is the output layer. In \cite{liu2020simple} the authors propose the following two conditions to ensure that the composition is \emph{distance-aware}:
\begin{itemize}
    \item $f(\cdot)$ is \emph{distance-aware}
    \item $\mathbb{E}_{x \sim \mathbb{X}_{in-d}} \left[ \| x^{*} - X \|_{X}^{2}\right] \approx \mathbb{E}_{x \sim \mathbb{X}_{in-d}} \left[ \| enc(x^{*}) - enc(X) \|_{enc(X)}^{2}\right]$
\end{itemize}
The last condition means that distances between data points in input space should be correlated with distances in the representation learning layer, which is equipped with a $\| \cdot \|_{enc(X)}$ metric. Whereas we have seen that for example a Gaussian Process satisfies the \emph{distance-aware} condition for the last layer predictor, we are still left with the question on how to maintain distances in the representation learning layer correlated to distances in the input layer.

As we have previously stated, GPs are \emph{distance-aware}. Thus, they can reliably notice departures from the training set manifold. For SVGP we decompose the model uncertainty into two components:
\begin{align}
    h(\cdot) &=  \mathcal{N}(h|0, K_{ff} - K_{fu}K_{uu}^{-1}K_{uf}) \label{eqn:distributional_uncertainty_svgp} \\  
    g(\cdot) &= \mathcal{N}(g|K_{fu}K_{uu}^{-1}m_{U}, K_{fu}K_{uu}^{-1} S_{U} K_{uu}^{-1}K_{uf}) \label{eqn:epistemic_uncertainty_svgp}
\end{align}
The $h(\cdot)$ variance captures the shift from within to outside the data manifold and will be denoted as \emph{distributional uncertainty} (interchangeably also denoted as non-parametric variance). The variance $g(\cdot)$ is coined as \emph{epistemic uncertainty} (interchangeably also denoted as parametric variance or within-data uncertainty) encapsulating uncertainty present inside the data manifold. To capture the overall uncertainty in $h(\cdot)$, thereby also capturing the spread of samples from it, we can calculate its differential entropy as:
\begin{equation}
    h(h) = \frac{n}{2}\log{2\pi} + \frac{1}{2}\log{\mid K_{ff} - K_{fu}K_{uu}^{-1}K_{uf} \mid } + \frac{1}{2}n  \label{eqn:differential_entropy_formula}
\end{equation}
In practice we only use the diagonal terms of the Schur complement, hence the log determinant term will considerably simplify. Intuitively, if terms on the diagonal of the Schur complement have higher values, so will the distributional differential entropy. This OOD measure in logit space will be used throughout the rest of the paper for measuring departures from training set manifold. Distributional uncertainty as defined in equation \eqref{eqn:distributional_uncertainty_svgp} can be used as proxy for OOD detection due to its \emph{distance-awareness} properties. Therefore we desire DGP with similar OOD detection properties as it shallow counterpart at every layer in its hierarchy.

\paragraph{\RNum{2} Smoothness in hidden layers.}
Imposing smoothness in neural networks' hidden layers has been linked to increased generalization, robustness to input perturbations or reliable uncertainty estimates \citep{rosca2020case}. With the latter property in mind, recent work \citep{van2021feature, liu2020simple} have proposed bi-Lipschitz constraints for extracting features to be fed into a GP. This paradigm is based on the notion that objects which were close in previous layers should also be close going forward in the hierarchy, thus avoiding \emph{feature-collapse} \citep{van2021feature}. Consequently, kernels in DGP should provide a certain guarantee of smoothness, respectively it should guard against \emph{feature-collapse}.

\subsection{Hybrid kernel} \label{sec:hybrid_kernel}

The term \emph{Distributional Gaussian Process} was first introduced in \citep{bachoc2017gaussian} to describe a shallow GP that operates on measures using a Wasserstein-2 based kernel ( equation \eqref{eqn:wasserstein_kernel}) and for purposes of differentiating it from subsequent models it will be coined as \emph{DistGP-Bahoc} throughout the remainder of the paper. We first investigate prior function space properties if the latter formulation is extended to the multi-layered case and then introduce our proposed kernel and generative process.

Considering the case of Euclidean input data, the fist layer is governed by a GP $p(F_{1})= \mathcal{N}(0,K_{F_{0}F_{0}})$. We immediately see that the prior has the capability to capture correlations between data points. Adding a new layer given by \emph{DistGP-Bahoc} operating directly on $p(F_{1})$, where we only take diagonal terms of its covariance matrix to compute $K^{W_{2}}_{F_{1,i},F_{1,j}} =  \sigma^{2}_{2} \exp \frac{-W_{2}^{2}(\mathcal{N}(0, \sigma^{2}_{1}),\mathcal{N}(0, \sigma^{2}_{1}))}{l^{2}_{2}} = \sigma^{2}_{2}$. Therefore, correlations between different data points are ignored as we arrive at the prior on $p\left( F_{2} \right)$:
\begin{equation}
p(F_{2})= \mathcal{N}\left[ \begin{pmatrix} 
        0       \\
        \vdots  \\
        0
        \end{pmatrix}    
    ,
        \begin{pmatrix} 
    \sigma^{2}_{2} & \dots  & \sigma^{2}_{2} \\
    \vdots     & \ddots & \vdots     \\
    \sigma^{2}_{2} & \dots  & \sigma^{2}_{2} 
        \end{pmatrix}\right]
\end{equation}
Consequently, this type of construction will exhibit an overly-correlated prior which will cause samples to collapse to set of finite values, thereby not satisfying Desiderata \RNum{2}. However, we desire a generative process that is without pathologies in the zero-mean function case.

Sampling hidden layers' features in DGP is of utmost importance towards introducing correlations between data points. On the other hand, the Wasserstein-2 kernel (equation \eqref{eqn:wasserstein_kernel}) directly applied on hidden layers' resulting multivariate Gaussians is attractive as it takes into account differences between variance terms. The hindsight is that differences between low and high variance terms in the multivariate normals would enable a finer separation of data points which were OOD compared to in-distribution as detected by the previous layer. With these observations in mind, we argue that both stochastic (sampling hidden layer features in DSVI framework for DGPs) and deterministic versions (Wasserstein-2 kernels directly applied on hidden layer GP) are important from entirely complementary viewpoints. We seek to combine these two formulations in a single one that preservers both key properties, we can now write the generative process of this new probabilistic framework coined Distributional Deep Gaussian Processes (DDGP) for 2 layers:
\begin{align}
    F_{1} &\sim \mathcal{N}\left(0, K_{ff} \right) \\
    F_{1}^{sth} &= m\left(F_{1}\right) + \sqrt{v\left(F_{1}\right)} \epsilon, ~ \epsilon \sim \mathcal{N}\left(0, \mathbb{I}_{n} \right) \\
    F_{1}^{det} &= \mathcal{N}\left(m(F_{1}), diag\left[ v(F_{1}) \right] \right) \\
    F_{2} &\sim \mathcal{N}\left(0, k_{hybrid}\left(\{F_{1}^{sth}, F_{1}^{det} \}, \{F_{1}^{sth},F_{1}^{det}\} \right) \right)
\end{align}
More intuitively, this generative process implies keeping track of a \emph{stochastic}, respectively \emph{deterministic} component of the same SVGP at any given hidden layer. In Figure \ref{fig:schematic_dgp_ddgp} we provide a schematic of the differences in the generative process of DGPs and DDGPs.

\begin{figure}[htb]
  \centering
    \includegraphics[width=\linewidth]{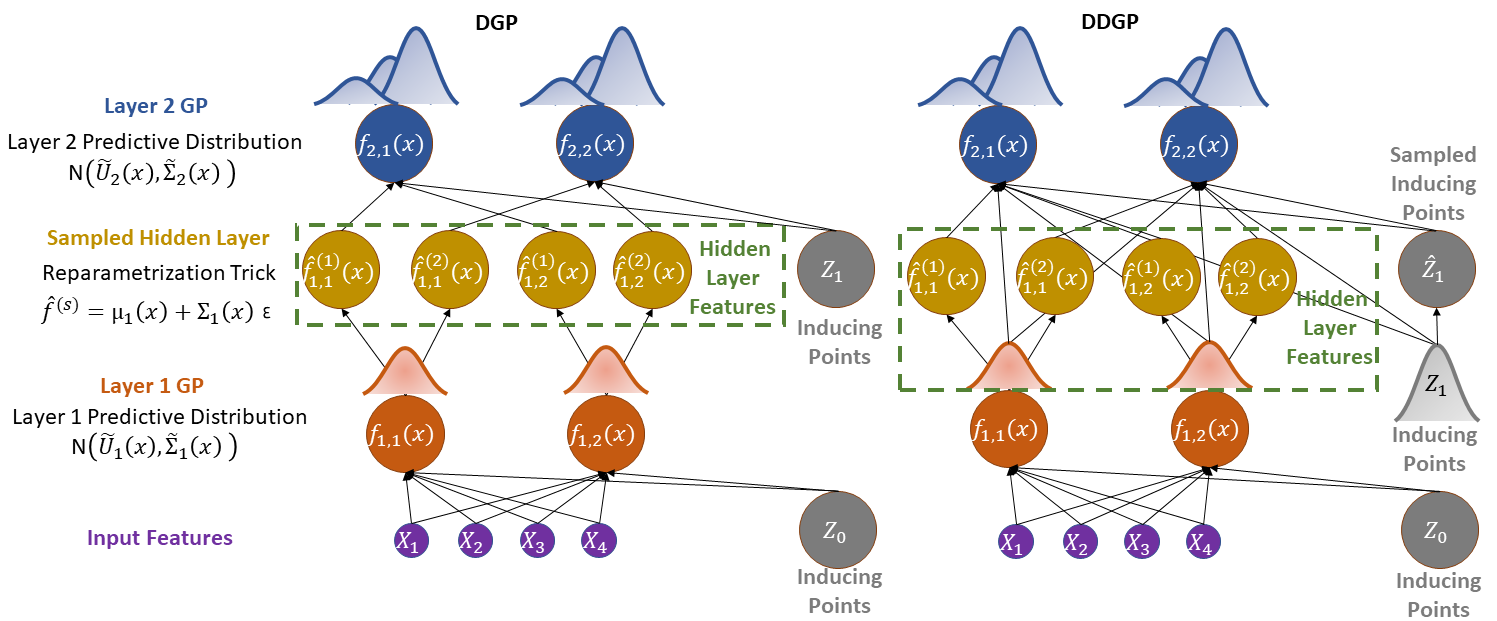}
    \caption[Conceptual differences between DGPs and DDGPs.]{\textbf{Conceptual differences between DGPs and DDGPs.} DGPs operate in Euclidean space of sampled hidden layer features. DDGPs also operate on Wasserstein space of underlying hidden layer GPs, with inducing locations following a Gaussian distribution.}
    \label{fig:schematic_dgp_ddgp}  
\end{figure}

We took inspiration from Varifold theory (see subsection \ref{apd:varifold_theory} for details) and devised a kernel operating on both Euclidean and Wasserstein-2 space, thereby incorporating both stochastic and deterministic pathways. Considering again two data points in the prior space of $F_{1}$, respectively $ p\begin{pmatrix}
    F_{1,i} \\ F_{1,j}
\end{pmatrix}
\sim \mathcal{N}\left[
\begin{pmatrix}
0 \\
0
\end{pmatrix},
\begin{pmatrix}
\sigma_{1}^{2} & K^{E}_{i,j} \\
K^{E}_{j,i} & \sigma_{1}^{2}
\end{pmatrix}
\right]$ and their associated samples $\{f_{1,i}, f_{1,j} \}$. To compute Wasserstein-2 distances using the previous layer's MVN, we take the diagonal of the covariance matrix and we denote $\mu_{i} = \mathcal{N}(0,\sigma_{1}^{2})$; $\mu_{j} = \mathcal{N}(0,\sigma_{1}^{2})$. We can now introduce the hybrid kernel:
\begin{equation}
k^{H}\left(\mu_{i},\mu_{i}\right) = k^{SE}(x_{i},x_{j})\exp \sum_{d=1}^D\frac{-W_{2}^{2}(\mu_{i,d},\mu_{j,d})}{l_{d}^{2}} \label{eqn:hybrid_kernel}
\end{equation}
Prior samples of $F_{2}$ will be correlated due to the Euclidean part of this kernel, whereas the Wasserstein-2 component will always be equal to 1, hence neither introducing or removing extra correlations between data points. To make this clear, we can explicitly calculate it:
\begin{align}
    p\begin{pmatrix}
        F_{2,i} \\ F_{2,j}
    \end{pmatrix}
    &\sim \mathcal{N}\left[
        \begin{pmatrix}
            0 \\
            0
        \end{pmatrix},
        \begin{pmatrix}
            \sigma_{2}^{2} \exp -\frac{-W_{2}^{2}(\mu_{i},\mu_{i})}{l^{2}} & K^{E}_{i,j} \exp -\frac{-W_{2}^{2}(\mu_{i},\mu_{j})}{l^{2}} \\ 
            K^{E}_{j,i} \exp -\frac{-W_{2}^{2}(\mu_{j},\mu_{i})}{l^{2}} & \sigma_{2}^{2} \exp -\frac{-W_{2}^{2}(\mu_{j},\mu_{j})}{l^{2}}
        \end{pmatrix}
        \right] \\
    &\sim \mathcal{N}\left[
        \begin{pmatrix}
            0 \\
            0
        \end{pmatrix},
        \begin{pmatrix}
            \sigma_{2}^{2} & K^{E}_{i,j}  \\ 
            K^{E}_{j,i} & \sigma_{2}^{2} 
        \end{pmatrix}
        \right]
\end{align}
since the Wasserstein-2 distance between two equal distributions ($\mu_{i}$ and $\mu_{j}$) is zero.
\begin{remark}
    In prior space, a hierarchical construction based on hybrid kernels (equation \eqref{eqn:hybrid_kernel}) for hidden layers is equivalent to a DGP. In posterior space, this no longer holds true with several different properties which we will outline subsequently.    
\end{remark}

\paragraph{Connections between Varifold kernel and Hybrid Kernel.} Whereas usually during training we only sample once from $q(F_{l})$ at each layer in order to obtain the samples that get used in the Euclidean part of the kernel, one could sample arbitrary times, thereby obtaining a collection of Gaussian distributions. In this case, we can re-express our hybrid kernel in the multi-sample scenario:
\begin{equation}
     \sum_{i=1}^{S} \sum_{j=1}^{S} k^{SE}(x_{i}, y_{j}) \exp \sum_{d=1}^D\frac{-W_{2}^{2}(\mu_{i,d}, \mu_{j,d})}{l_{d}^{2}} 
\end{equation}
We can observe similarities with the varifold kernel in equation \eqref{varifold_kernel}, where we compute a distance between submanifolds, taking into consideration Euclidean distances between discrete points on the manifolds' surface, combined with information pertaining to the points' geometry. For hybrid kernels as applied to DDGPs, we compare two sets of multivariate Gaussians, taking into account Euclidean differences between points sampled from the underlying distributions, but also considering geometric differences between these distributions as captures by the Wasserstein-2 distance. A visual depiction of these different but similar objectives is provided in figure \ref{fig:varifold_hybrid_kernels}.

\begin{figure}[htb]
    \centering
    \includegraphics[width=\linewidth]{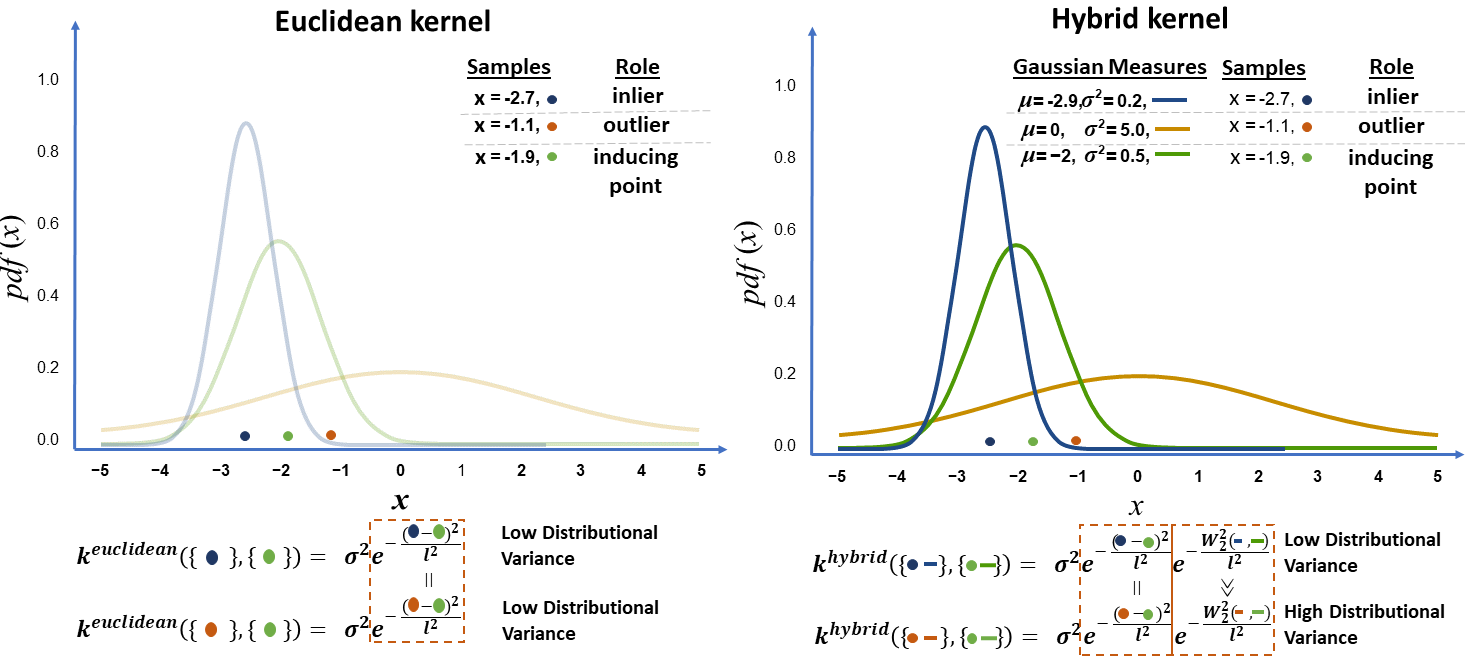}
    \caption[Conceptual difference between euclidean and hybrid kernel.]{\textbf{Conceptual difference between euclidean and hybrid kernel.} Explicit difference between moments in Wasserstein-2 distance ensures \emph{distance-awareness}.}
    \label{fig:schematic_hybrid_kernel}  
\end{figure}

\paragraph{Hybrid kernels and OOD detection.} The key motivation behind the introduction of hybrid kernels can be succinctly described by a simple toy example. Considering two-layered DGPs and DDGPs, we assume an in-distribution point $x_{in-d}$ to have low total variance in hidden layer $F_{1}$, respectively an OOD point $x_{ood}$ to have high total variance. In the DGP case, upon sampling from $q(F_{1}(x_{in-d}))$  and  $q(F_{1}(x_{ood}))$ we can end up with samples which are equally distant with respect to inducing points' location $Z_{1}$. If this occurs, then non-parametric variance (proxy for distributional uncertainty) will be equal for $x_{in-d}$ and $x_{ood}$ in $F_{2}$. Hence, what was initially flagged as OOD in the first hidden layer will be considered as in-distribution by the second hidden layer, thus not satisfying Desiderata \RNum{1}. Conversely, in the DDGP case and under the assumption that the variance of distributional inducing points' locations $Z_{2}$ is almost equal in distribution to the total variance of in-distribution points in $F_{1}$, the Wasserstein-2 component of the hybrid kernel will notice that there is a higher distance between the now distributional inducing point location and $x_{ood}$, as opposed to the distance of the former with $x_{in-d}$. Then, the non-parametric variance of $x_{ood}$ will be higher than that of $x_{in-d}$, thus satisfying Desiderata \RNum{1}. A visual depiction of this case study is illustrated in Figure \ref{fig:schematic_hybrid_kernel}.

\subsection{Evidence lower bound}
In this subsection we outline the generative process of constructing multi-layered architectures using the hybrid kernel, which we will coin as Distributional Deep Gaussian Processes (DDGP).
First, we remind ourselves that a DGP is defined as a stack of shallow GP acting as the prior:
\begin{equation}
 \underbrace{p(Y|F_{L})}_{\text{likelihood}}\underbrace{\prod_{l=1}^{L} p(F_{l}|F_{l-1},U_{l};Z_{l-1})p(U_{l})}_{\text{prior}}
\end{equation}
, where for brevity of notation we denote $F_{0} = X$. The SVGPs between hidden layers are treated as being noiseless. As the prior is analytically intractable to integrate, \cite{salimbeni2017doubly} have suggested to sample from each hidden layer of the DGP in order to obtain unbiased stochastic gradients. The difference between DGP and DDGP resides in the fact that for $l \geq 2$ we not only sample from the marginal of the data points $q(F_{l};F_{l-1})$, but we also explicitly use this marginal in the computation of the next layer's kernel. More concretely, we have the following joint density prior $p(Y,\{F_{l},U_{l}\}_{l=1}^{L})$:
\begin{equation}
    \underbrace{p(Y|F_{L})}_{likelihood}\underbrace{\prod_{l=2}^{L} p(F_{l}|F_{l-1},U_{l};Z_{l-1})p(U_{l})}_{ \text{Wasserstein \& Euclidean space}}  \underbrace{ p(F_{1}|U_{1};X)p(U_{1})}_{\text{Euclidean space}}
\end{equation}
We introduce a factorized posterior between layers and dimensions of the following form :  $q(\{F_{l}\}_{l=1}^{L}, \\ \{U_{l}\}_{l=1}^{L}) =  p(F_{L}|U_{L};Z_{L-1})\prod_{l=1}^{L}q(U_{l})$, where $q(U_{l})$ is taken to be a multivariate Gaussian with mean $m_{U_{l}}$ and variance $S_{U_{l}}$. 
The posterior over $q(F_{1})$ is similar to the DGP case, with the difference residing in $l\geq2$, where we have the following conditional DistGP:
\begin{equation}
 p(F_{l}|U_{l};Z_{l-1},F_{l-1}) = \mathcal{N}(F_{l}|K_{fu}^{H}\inv{K^{H}_{uu}}U_{l}, K_{ff}^{H} - Q_{ff}^{H}) \label{eqn:conditional_hybrid_gp}
\end{equation}
After sampling from $q(F_{l-1})$ and $Z_{l-1}\sim \mathbf{N}(M_{Z_{l-1}},\Sigma_{Z_{l-1}})$, we proceed to integrate out $U_{l}$, arriving at $q(F_{l})$:
\begin{equation}
    \mathcal{N}(F_{l}|K_{fu}^{H}\inv{K^{H}_{uu}}m_{U_{l}}, K_{ff}^{H} - Q_{ff}^{H} + K_{fu}^{H}\inv{K^{H}_{uu}}S_{U_{l}}\inv{K^{H}_{uu}}K_{uf}^{H})    
\end{equation}
where in the computation of the hybrid kernels we utilise $q(F_{l-1})$ for computing the Wasserstein-2 distance, respectively samples from $q(F_{l-1})$ in the Euclidean part. It is worth highlighting the fact that for $l \geq 2$, $Z_{l}$ are MVN. Further derivations will lead us to the DDGP ELBO:
\begin{equation}
    \mathcal{L}_{DDGP} = \mathbb{E}_{\{q(F_{l},U_{l}\}_{l=1}^{L}} p(Y|F_{L}) - \sum_{l=2}^{L} \mathbb{E}_{Z_{l}}\left[KL\left[q(U_{l}) \| p(U_{l})\right]\right] - KL\left[q(U_{1}) \| p(U_{1})\right]
\end{equation}

\subsection{Propagation of variance in DDGP} \label{sec:propagation_outlier_ddgp}

In this subsection, we investigate what are the necessary requirements for reliable propagation of non-parametric variance through the hierarchy for OOD data points.

\begin{proposition} \label{thm:outlier_req_ddgp}
    We consider the approximate posterior DDGP as a composition of functions $q_{L} = q_{L} \circ ... \circ q_{1}$ with $q_{l}$ being given by $
    q(F_{l}|F_{l-1}) = \int p(F_{l}|U_{l},F_{l-1})q(U_{l}) = \mathcal{N}(F_{l} \mid \tilde{U_{l}}(F_{l-1}), \tilde{\Sigma_{l}}(F_{l-1}))
    $. We consider an OOD data point $x_{out}$ in input space such that $\tilde{\Sigma_{1}}(x_{out}) = \sigma^{2}_{1}$ and $\tilde{U_{1}}(x_{out}) = 0$, respectively $x_{in}$ to be in-distribution, with $\tilde{\Sigma}_{1}(x_{in-d}) = V_{in-d} \leq \sigma^{2}_{1}$ and $\tilde{U}_{1}(x_{in-d}) = M_{in-d}$. We assume that $M_{Z_{1}}$ are equidistantly placed between $[-3\sigma_{1},3\sigma_{1}]$ and $\Sigma_{Z_{1}} \leq \sigma^{2}_{1}$. The approximate variance in the second layer of $x_{out}$ will be higher than $x_{in-d}$ if the following hold:
    \begin{equation}
        \sigma^{2}_{1} >> \Sigma_{Z_{2}} ~ \textit{and} ~
        V_{in-d} \approx \Sigma_{Z_{2}} \nonumber
    \end{equation}
\end{proposition}

\begin{proof}
From the proof of Proposition \ref{thm:outlier_req_dgp} we remind the desired behaviour for the second layer:
\begin{equation}
    \tilde{\Sigma}_{2}(M_{in-d}) + V_{in-d} * \left(\frac{\partial \tilde{U}_{2}(F_{1})}{\partial F_{1}}\right)^{2}\Bigr|_{\substack{F_{1}=M_{in-d}}}
    \leq   \tilde{\Sigma}_{2}(0) + \sigma^{2} * \left(\frac{\partial \tilde{U}_{2}(F_{1})}{\partial F_{1}}\right)^{2}\Bigr|_{\substack{F_{1}=0}}  
\end{equation}
In contrast to zero mean DGP, where under current assumptions we obtained $\tilde{\Sigma}_{2}(M_{in-d}) \approx \tilde{\Sigma}_{2}(0)$, this is not the case for DDGP. Expanding on the variance terms $\tilde{\Sigma}_{2}(M_{in-d})$ and $\tilde{\Sigma}_{2}(0)$ looking at their non-parametric variance:
\begin{equation}
    \sigma^{2}_{2} - K_{M_{in-d}u}^{H}\inv{K_{uu}^{H}}K_{uM_{in-d}}^{H} \leq \sigma^{2}_{2} - K_{0u}^{H}\inv{K_{uu}^{H}}K_{u0}^{H}
\end{equation}
To simplify the notation we can consider just one inducing point, lying halfway between 0 and $M_{in-d}$ at $Z_{m}$ and with $Z_{var} << \sigma^{2}_{1}$. Using the hybrid kernel formulation and re-arranging the terms in the above equation considering the matrices are scalars we arrive at: 
\begin{align}
    \sigma^{2}_{2}  &\exp{\frac{-W_{2}^{2}(\mathcal{N}(0,\sigma^{2}_{1}),\mathcal{N}(Z_{m},Z_{var}))}{l^{2}_{2}}} \exp{\frac{-(0-Z_{m})^{2}}{l^{2}_{2}}}  \leq \\ & \nonumber \sigma^{2}_{2} \exp \frac{-W_{2}^{2}(\mathcal{N}(M_{in-d},V_{in-d}),\mathcal{N}(Z_{m},Z_{var}))}{l^{2}_{2}}  \exp \frac{-(M_{in-d}-Z_{m})^{2}}{l^{2}_{2}}    
\end{align}
which after eliminating redundant terms and taking the logarithm we obtain:
\begin{equation}
    W_{2}^{2} (\mathcal{N}(0,\sigma^{2}_{1}), \mathcal{N}(Z_{m},Z_{var}) ) +  (0-Z_{m})^{2}  \geq   W_{2}^{2}(\mathcal{N}(M_{in-d},V_{in-d}), \mathcal{N}(Z_{m},Z_{var})) +  (M_{in-d}-Z_{m})^{2}   
\end{equation}
The above inequality holds if and only if $\sigma^{2}_{1}>>Z_{var}$ and $V_{in-d} \approx Z_{var}$, thus proving our proposition.

\end{proof}

\begin{remark}
    In opposition to DGP which require a specific pattern of first order derivatives surrounding 0, DDGP rely less on this constraint as the above inequality is already providing a significant difference in non-parametric variance between outlier and inlier points, thus satisfying Desiderata \RNum{1}. For any layer $l$, $\Sigma_{Z_{l}}$ will converge towards values of the parametric posterior variance of in-distribution points, since the inducing points are encouraged to gravitate towards the in-distribution manifold at layer $l$. In well optimized scenarios, the parametric posterior variance will be reduced compared to the prior variance ($\sigma_{l-1}^{2}$). Considering this observations, the conditions laid out in Proposition \ref{thm:outlier_req_ddgp} will be satisfied in practice.
\end{remark}

\subsection{\emph{Feature-collapse} in (D)DGP} 

\cite{dunlop2018deep} provide an in-depth analysis of function space properties DGP have, which DDGPs are equivalent to in prior space (see subsection \ref{sec:hybrid_kernel} ). We are interested to find what are the necessary requirements, given a few constraints, of the function space values to collapse in the posterior for in-distribution data points.

We now introduce some notation conventions. For $D \in \mathcal{R}^{d}$ and $\{ D_{l} \in \mathcal{R}^{d_{l}} \}_{l=1}^{L}$, where $d_{l}$ is the number of dimensions in the l-th layer of the hierarchy. We consider the functions $f_{l} : D_{l-1} \to \mathcal{R}^{d_{l}}$, with a DGP being defined as composition of functions: $f_{n}\left( \cdots f_{2}\left(f_{1}\left(x\right)\right)\right)$.

 For a DGP at layer $n$ we assume the following posterior conditional formulation for two data points $x,~x^{*} \in D_{in-d}$, respectively:
\begin{equation}
p\begin{pmatrix}
    f_{n}(x) \\
    f_{n}(x^{*})
\end{pmatrix} =
\mathcal{N}
    \left[    \begin{pmatrix} 
        K_{xu}K_{uu}^{-1}U_{n} \\
        K_{x^{*}u}K_{uu}^{-1}U_{n}
    \end{pmatrix}    
    ,
    \begin{pmatrix} 
        K_{xx} - Q_{xx}  & K_{xx^{*}} - Q_{xx^{*}} \\
        K_{x^{*}x} - Q_{x^{*}x}   & K_{x^{*}x^{*}} - Q_{x^{*}x^{*}}
    \end{pmatrix}
    \right]
\end{equation}
For subsequent derivations, we denote $\mu(x) = \mathcal{N}\left(K_{xu}K_{uu}^{-1}U_{n} , K_{xx} - Q_{xx} \right)$.

\begin{proposition} \label{thm:posterior_function_dgp}
    We assume $f_{0}$ to be bounded on bounded sets almost-surely. For a DGP, if at each layer $l$ we have satisfied the following inequality $\frac{\sigma^{4}_{l}}{4l_{l}^{4}} \| K_{uu}^{-1}U_{l}\|_{2}^{2} \left(f_{l}(x) - f_{l}(x^{*}) \right)^{2} < 1 $, alongside $Q_{ff} \approx K_{ff}$, given two in-distributions points $x$ and $x^{*}$ we have the following result: 
    \begin{equation}
        P\left( \| f_{n}(x) - f_{n}(x^{*}) \|_{2} \to 0 \mid \{U_{l}\}_{l=1}^{n} \right) = 1 \nonumber
    \end{equation}    
\end{proposition}

\begin{proof}

In subsequent derivations, we assume that the set of inducing points $\{ Z_{l} \}_{l=1}^{L}$ are sufficiently large and well-optimized so that for points $x \in D$ we have $K_{ff} \approx K_{fu}K_{uu}^{-1}K_{uf}$.

We commence the derivation:
\begin{align}
    \mathbb{E}\left[  \| f_{n}(x) - f_{n}(x^{*})\|_{2}^{2} \mid  f_{n-1}, U_{n} \right] & = \sum_{d=1}^{D_{l}} \mathbb{E}\left[ \mid f_{n}^{d}(x) - f_{n}^{d}(x^{*}) \mid^{2} \mid f_{n-1}, U_{n} \right]  \\ &
    = \sum_{d=1}^{D_{l}} \mid \left(K_{fu} - K_{f^{*}u} \right)K_{uu}^{-1}U_{n}^{d} \mid^{2} 
\end{align} 
We can apply Cauchy-Schwarz inside the brackets to obtain the following inequality:
\begin{align}
    & \leq \sum_{d=1}^{D_{l}} \| K_{fu} - K_{f^{*}u}\|_{2}^{2}  \| K_{uu}^{-1}U_{n}^{d}\|_{2}^{2}  \\ &
    \leq \sum_{d=1}^{D_{l}}  \| \sigma^{2}_{n}\exp{-\frac{\left( x - Z_{n} \right)^{2}}{2l_{n}^{2}}} - \sigma^{2}_{n}\exp{-\frac{\left( x^{*} - Z_{n} \right)^{2}}{2l_{n}^{2}}} \|_{2}^{2}  \| K_{uu}^{-1}U_{n}^{d}\|_{2}^{2}    
\end{align}
We assume that $x = x^{*} + h $ and $x^{*}$ is a high-density data point with regards to $Z_{n}$, respectively $\mid\mid  x - Z_{n}\mid\mid_{2}^{2} \approx 0$. Now we can re-arrange the last equation as:
\begin{align}
    & \leq \sum_{d=1}^{D_{l}} \| \sigma^{2}_{n}\exp{-\frac{\left( h + x^{*} - Z_{n} \right)^{2}}{2l_{n}^{2}}} - \sigma^{2}_{n} \|_{2}^{2}  \| K_{uu}^{-1}U_{n}^{d}\|_{2}^{2}  \\ &
    \leq \sum_{d=1}^{D_{l}} \sigma^{4}_{n} \mid \exp{-\frac{\left( h \right)^{2}}{2l_{n}^{2}}} - 1 \mid^{2} \| K_{uu}^{-1}U_{n}^{d}\|_{2}^{2}    
\end{align}
We have the following general result $1 - \exp{-x} < x,~ \text{for} ~ x>-1$ which we apply in our case since $\frac{\left( h \right)^{2}}{2l^{2}} > -1$, we obtain:
\begin{equation}
    \mathbb{E}\left[  \| f_{n}(x) - f_{n}(x^{*}) \|_{2}^{2} \mid f_{n-1},U_{n}\right] \leq \sum_{d=1}^{D_{l}}  \frac{\sigma^{4}_{n}}{4l_{n}^{4}} \| K_{uu}^{-1}U_{n}^{d}\|_{2}^{2} \left( h \right)^{2} \left( h \right)^{2}    
\end{equation}
which we can re-express with regards to previous layer features to make subsequent derivations easier to follow:
\begin{equation}
    \mathbb{E}\left[  \| f_{n}(x) - f_{n}(x^{*}) \|_{2}^{2} \mid f_{n-1},U_{n}\right] \leq \sum_{d=1}^{D_{l}}  \frac{\sigma_{n}^{4}}{4l_{n}^{4}} \| K_{uu}^{-1}U_{n}^{d}\|_{2}^{2} \left( h \right)^{2} \left( f_{n-1}(x) - f_{n-1}(x^{*}) \right)^{2}    
\end{equation}
We can now apply induction and the tower property of conditional expectation to arrive at:
\begin{equation}
    \mathbb{E}\left[  \| f_{n}(x) - f_{n}(x^{*}) \|_{2}^{2} \mid \{U_{l}\}_{l=1}^{n}\right]  \leq \prod_{l=1}^{n} \frac{\sigma_{l}^{4}}{4l_{l}^{4}} \| K_{uu}^{-1}U_{l}\|_{2}^{2} \left( h \right)^{2} \left( f_{0}(x) - f_{0}(x^{*}) \right)^{2}    
\end{equation}

By Markov's inequality, for any $\epsilon > 0$ we have that:
\begin{equation}
    P\left( \| f_{l+1}(x) - f_{l+1}(x^{*})  \|_{2} \geq \epsilon \mid \{U_{l}\}_{l=1}^{n} \right) \leq \frac{1}{\epsilon^{2}} \prod_{l=1}^{n} \frac{\sigma_{l}^{4}}{4l_{l}^{4}} \| K_{uu}^{-1}U_{l}\|_{2}^{2} \left( h \right)^{2}     \left( f_{0}(x) - f_{0}(x^{*}) \right)^{2}    
\end{equation}

We can apply the first Borel-Cantelli lemma to obtain:
\begin{equation}
    P\left( \cap_{l=1}^{\infty} \cup_{m=l}^{\infty} \|  f_{m}(x) - f_{m}(x^{*}) \|_{2} \geq \epsilon \mid \{U_{l}\}_{l=1}^{n} \right) = 0
\end{equation}
Lastly, we can express the following:
\begin{align}
    P\left( \|  f_{n}(x) - f_{n}(x^{*}) \|_{2} \to 0 \mid \{U_{l}\}_{l=1}^{n} \right) &= P( \cap_{k=1}^{\infty} \cup_{l=1}^{\infty} \cap_{m=l}^{\infty} \|  f_{m}(x) - f_{m}(x^{*}) \|_{2}  \leq \frac{1}{k} \mid \{U_{l}\}_{l=1}^{n} ) \\
    &= 1 - P( \cup_{k=1}^{\infty} \cap_{l=1}^{\infty} \cup_{m=l}^{\infty} \|  f_{m}(x) - f_{m}(x^{*}) \|_{2}   \geq \frac{1}{k} \mid \{U_{l}\}_{l=1}^{n} ) \\
    & ~\geq 1 - \sum\limits_{k=1}^{\infty} P( \cap_{l=1}^{\infty} \cup_{m=l}^{\infty} \|  f_{m}(x) - f_{m}(x^{*}) \|_{2}   \geq \frac{1}{k} \mid \{U_{l}\}_{l=1}^{n} ) \\
    &= 1
\end{align}
From which we obtain the proof of our proposition, respectively $ P\left( \|  f_{n}(x) - f_{n}(x^{*}) \|_{2} \to 0 \mid \{U_{l}\}_{l=1}^{n} \right) = 1$
\end{proof}

\begin{remark}
    From the constraint $\frac{\sigma_{l}^{4}}{4l_{l}^{4}} \| K_{uu}^{-1}U_{l}\|_{2}^{2} \left(f_{l}(x) - f_{l}(x^{*}) \right)^{2} < 1 $ to ensure \emph{feature-collapse}, we can see that in the converse case, of avoiding \emph{feature-collapse}, this can be assured by either increasing the kernel variance $\sigma^{2}_{l}$ or $\| K_{uu}^{-1}U_{l}\|_{2}^{2}$, which are both responsible for the amplitude of sampled functions. Moreover, another pathway to avoid \emph{feature-collapse} is to increase the lengthscale $l^{2}_{l}$. Intuitively, this will ensure that unseen test points close to an inducing point will reliably have their predicted mean almost equal to that respective inducing point's value. Conversely, for lower lengthscale values, the predicted mean for an unseen points will be a weighted version of inducing points' output values across a larger distance, hence favouring similar values for unseen points, which will eventually lead to \emph{feature-collapse}.
\end{remark}

Now we are interested in deriving similar constraints for DDGPs to suffer from \emph{feature-collapse} and possibly seeing if there are any differences.

\begin{proposition} \label{thm:posterior_function_ddgp}
    We assume $f_{0}$ to be bounded on bounded sets almost-surely. For a DDGP, if at each layer $l$ except the first one we have satisfied the following inequality $ \frac{\sigma^{4}_{l}}{4l^{4}_{l}} \| K_{uu}^{-1}U_{l}\|_{2}^{2} [ f_{l}(x) - f_{l}(x*) + \frac{\| m(f_{l}(x)) - m(f_{l}(x^{*})) \|_{2}^{2} + \| v(f_{l}(x)) - v(f_{l}(x^{*})) \|_{2}^{2}}{f_{l}(x) - f_{l}(x*)} ]^{2}  \leq 1$, given two in-distributions points $x$ and $x^{*}$ s.t. $Q_{ff} \approx K_{ff}$ we have the following result: 
    \begin{equation}
        P\left( \|  f_{n}(x) - f_{n}(x^{*}) \|_{2} \to 0 \mid \{U_{l}\}_{l=1}^{n} \right) = 1 \nonumber
    \end{equation}    
\end{proposition}

\begin{proof}

 We adapt our previous proof for Proposition \ref{thm:posterior_function_dgp} to suit the hybrid kernel specific to DDGPs. We have the following upper bound: 
\begin{align}
    \mathbb{E}\left[  \| f_{n}(x) - f_{n}(x^{*}) \|_{2}^{2} \mid f_{n-1},U_{n}\right] \leq  & \sum_{d=1}^{D_{n}} \| \sigma^{2}_{n}\exp{-\frac{\left( x - Z_{n} \right)^{2}+W_{2}^{2}(\mu(x), \mu(Z_{n}))}{2l_{n}^{2}}} \\ & \nonumber - \sigma^{2}_{n}\exp{-\frac{\left( x^{*} - Z_{n} \right)^{2}+W_{2}^{2}(\mu(x^{*}), \mu(Z_{n}))}{2l_{n}^{2}}} \|_{2}^{2}  \| K_{uu}^{-1}U_{n}^{d}\|_{2}^{2}   
\end{align}
We assume that $x = x^{*} + h $ and $x^{*}$ is a high-density data point with regards to $Z_{n}$, respectively $\mid\mid  x - Z_{n}\mid\mid_{2}^{2} \approx 0$. For their first two moments we have equivalent assumptions $m(x) = m(x^{*}) + m(h) $ and $v(x) = v(x^{*}) + v(h)$, with $\mu(x^{*}) \stackrel{d}{\approx} \mu(Z_{n})$ We can re-arrange the last equation as:
\begin{equation}
    \leq \sum_{d=1}^{D_{n}}  \| \sigma^{2}_{n}\exp{-\frac{\left( h + x^{*} - Z_{n} \right)^{2} + \| m(h) \|_{2}^{2} + \| v(h) \|_{2}^{2}}{2l_{n}^{2}}} - \sigma^{2} \|_{2}^{2}  \| K_{uu}^{-1}U_{n}^{d}\|_{2}^{2} 
\end{equation}
We have the following general result $1 - \exp{-x} < x,~ \text{for} ~ x>-1$ which we apply in our case since $\frac{\left( h \right)^{2} + \mid m(h) \mid^{2} + \mid v(h) \mid^{2}}{2l^{2}} > -1$:
\begin{equation}
    \mathbb{E}\left[  \| f_{n}(x) - f_{n}(x^{*}) \|_{2}^{2} \mid f_{n-1},U_{n}\right] \leq \sum_{d=1}^{D_{n}}  \frac{\sigma_{n}^{4}}{4l_{n}^{4}} \| K_{uu}^{-1}U_{n}^{d}\|_{2}^{2} \left( h^{2} + \| m(h) \|_{2}^{2} + \| v(h) \|_{2}^{2} \right)^{2} \label{eqn:cite_this_now}    
\end{equation}
We can see that we can re-express the final term of above equation as follows:
\begin{align}
    \left( h^{2} + \underbrace{\| m(h) \|_{2}^{2} + \| v(h) \|_{2}^{2}}_{c} \right)^{2} &= \left( h^{2} + c \right)^{2} \\
    &= h^{4} + 2h^{2}c + c^{2} \\
    &= h^{2}\left[h^{2} + 2c + \frac{c^{2}}{h^{2}} \right] \\
    &= h^{2} \left[ h + \frac{c}{h} \right]^{2}
\end{align}
We can apply this side derivation to equation \eqref{eqn:cite_this_now}:
\begin{equation}
    \mathbb{E}\left[  \| f_{n}(x) - f_{n}(x^{*}) \|_{2}^{2} \mid f_{n-1},U_{n}\right] \leq \sum_{d=1}^{D_{n}}  \frac{\sigma_{n}^{4}}{4l_{n}^{4}} \| K_{uu}^{-1}U_{n}^{d}\|_{2}^{2} h^{2} \left[ h + \frac{\| m(h) \|_{2}^{2} + \| v(h) \|_{2}^{2}}{h} \right]^{2} h^{2}
\end{equation}
We can now apply induction and the tower property of conditional expectation to arrive at:
\begin{equation}
    \mathbb{E}\left[  \| f_{n}(x) - f_{n}(x^{*}) \|_{2}^{2} \mid \{U_{l}\}_{l=2}^{n}\right]  \leq \prod_{l=2}^{L} \frac{\sigma_{l}^{4}}{4l_{l}^{4}} \| K_{uu}^{-1}U_{l}^{d}\|_{2}^{2} \left[ h + \frac{\| m(h) \|_{2}^{2} + \| v(h) \|_{2}^{2}}{h} \right]^{2}  \left( f_{1}(x) - f_{1}(x^{*}) \right)^{2}    
\end{equation}
Using the upper bound on $\left( f_{1}(x) - f_{1}(x^{*}) \right)^{2}$ from the proof of Proposition \ref{thm:posterior_function_dgp}, we obtain:
\begin{align}
    \mathbb{E}\left[  \| f_{n}(x) - f_{n}(x^{*}) \|_{2}^{2} \mid \{U_{l}\}_{l=1}^{n}\right]  & \leq \frac{\sigma^{4}_{1}}{4l_{1}^{4}} \| K_{uu}^{-1}U_{1}^{d}\|_{2}^{2} \prod_{l=2}^{L} \frac{\sigma_{l}^{4}}{4l_{l}^{4}} \| K_{uu}^{-1}U_{l}^{d}\|_{2}^{2} \left[ h + \frac{\| m(h) \|_{2}^{2} + \| v(h) \|_{2}^{2}}{h} \right]^{2} \left( f_{0}(x) - f_{0}(x^{*}) \right)^{2}    
\end{align}
Similar derivations to the proof of Proposition \ref{thm:posterior_function_dgp} can be applied to arrive at the final result.
\end{proof}

\begin{remark}
    The constraint for DDGPs to suffer from \emph{feature-collapse} $ \frac{\sigma^{4}_{l}}{4l_{l}^{4}} \| K_{uu}^{-1}U_{l} \|_{2}^{2} [ f_{l}(x) - f_{l}(x*) + \frac{\| m(f_{l}(x)) - m(f_{l}(x^{*}))  \|_{2}^{2} + \| v(f_{l}(x)) - v(f_{l}(x^{*})) \|_{2}^{2}}{f_{l}(x) - f_{l}(x*)} ]^{2}  \leq 1$, given two in-distributions points $x$ and $x^{*}$ s.t. $Q_{ff} \approx K_{ff}$ is similar to the one for DGPs expect the final term in the brackets. We can see that for any given value of $f_{l}(x) - f_{l}(x^{*})$, since $\| m(f_{l}(x)) - m(f_{l}(x^{*}))  \|_{2}^{2} + \| v(q(f_{l}(x)) - v(f_{l}(x^{*})) \|_{2}^{2}$ is always positive, it will result in $[ f_{l}(x) - f_{l}(x*) + \frac{\| m(f_{l}(x)) - m(f_{l}(x^{*}) ) \|_{2}^{2} + \| v(f_{l}(x)) - v(f_{l}(x^{*})) \|_{2}^{2}}{f_{l}(x) - f_{l}(x*)} ]^{2} \geq \left[f_{l}(x) - f_{l}(x*)\right]^{2}$. This means that DDGPs are more suitable to guard against \emph{feature-collapse} than DGPs ceteris paribus. Moreover, we posit that the dependence on moment differences between two data points will embed DDGPs with \emph{smoother} properties in comparison to DGPs, since the moments are computed using squared exponential kernels, which will produce smoother functions. With these two properties in mind, we can affirm that DDGPs satisfy Desiderata \RNum{2}.
\end{remark}

\begin{figure}[!htb]
    \centering
    \includegraphics[width=\linewidth]{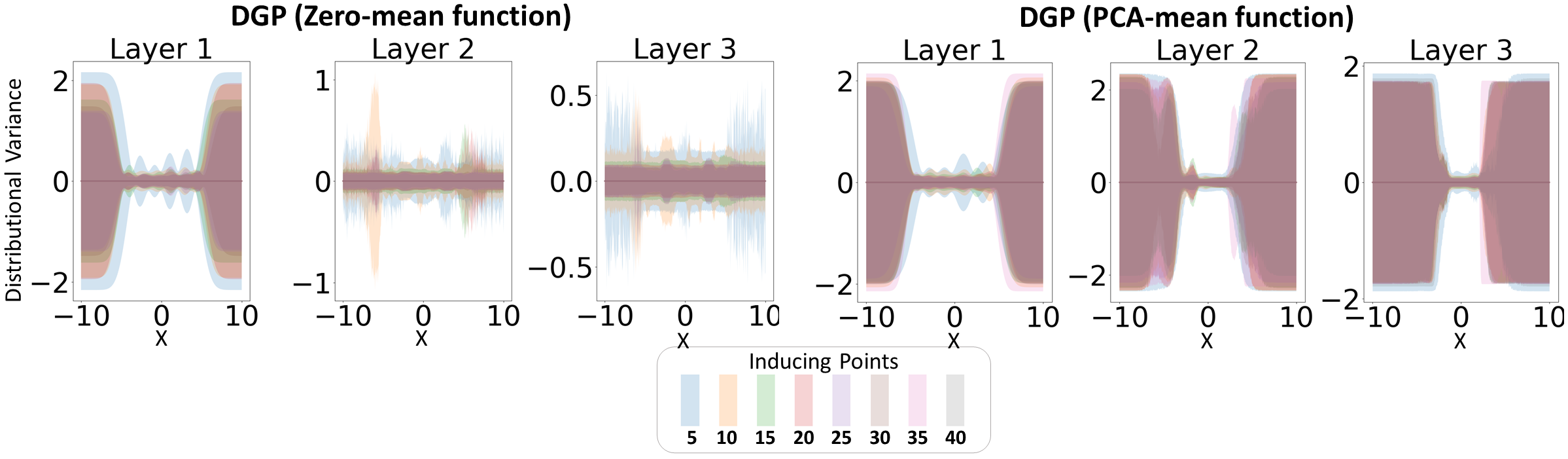}
    \caption[Progressive collapse of distributional uncertainty in DGPs]{\textbf{Progressive collapse of distributional uncertainty in DGPs.} Distributional uncertainty of DGPs trained on a toy dataset with data between -5.0 and 5.0 with a varying number of inducing points. X axis denotes input space.}
    \label{fig:visual_collapse_of_distributional_variance_num_inducing}  
\end{figure}

\section{Experiments}

\paragraph{Collapse of non-parametric variance in DGP.}

\cite{duvenaud2014avoiding} highlighted a certain pathology when stacking GP, respectively the increasingly more non-injective mappings as the network increases in number of layers. In this section we highlight another pathology of DGP with zero-mean functions, namely the progressive collapse of non-parametric uncertainty for larger number of inducing points. This erroneous behaviour is also present in relatively small networks with just two hidden layers (Figure \ref{fig:visual_collapse_of_distributional_variance_num_inducing}). 

To understand what is causing this pathology, we take a simple case study of a DGP with two hidden layers trained on a toy regression dataset (Figure \ref{fig:visual_explanation_collapse_of_variance}). Taking a clear outlier in input space, say the data point situated at -7.5, it gets correctly identified as an outlier in the mapping from input space to hidden layer space as given by its distributional variance. However, its outlier property dissipates in the next layer after sampling, as it gets mapped to regions where the next GP assigns inducing point locations. This is due to points inside the data manifold getting confidently mapped between -2.5 and 2.5 in hidden layer space. Consequently, what was initially correctly identified as an outlier will now have its final distributional uncertainty close to zero. Adding further layers, will only compound this pathology.

\begin{figure}[!htb]
  \centering
    \includegraphics[width=\linewidth]{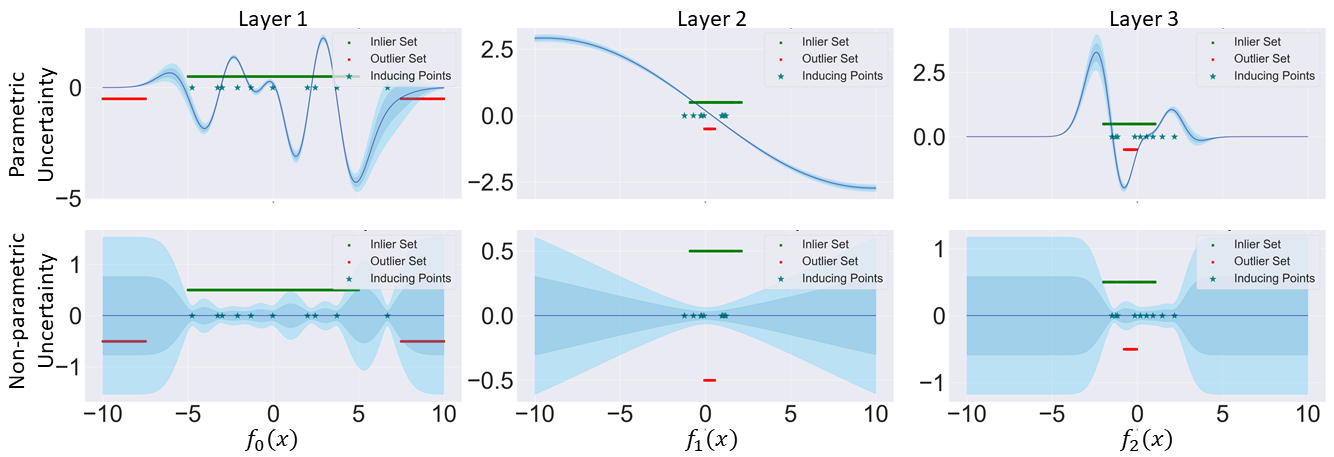}
    \caption[Rationale behind collapse of non-parametric uncertainty.]{\textbf{Rationale behind collapse of non-parametric uncertainty.} Layer-wise decomposition of uncertainty into parametric and non-parametric for a zero-mean DGP. OOD points in input space (red dots) get sampled in hidden layers close to in-distribution points (green dots), hence close to inducing point locations.}
    \label{fig:visual_explanation_collapse_of_variance}  
\end{figure}

\paragraph{Interplay between PCA mean functions and non-parametric uncertainty.}

In this subsection we investigate the effect that PCA mean function has on distributional uncertainty. In the case that the hidden layers' size is equal or larger than input size, then DGPs will retain their OOD capabilities. Conversely, if the hidden layers' size is considerably lower than input size (such as image classification) we posit that it will negatively affect OOD capabilities (Figure \ref{fig:visual_collapse_of_distributional_variance_num_inducing}). We use the ``banana'' dataset and train our proposed models with 2 hidden layers and 10 inducing points. The DGP distributional uncertainty is guided by the direction of the PCA projection of the data, thereby widely extrapolating low uncertainty regions to spatial locations where there is no data. For DDGP the PCA projection only slightly influences distributional variance in the last layer (Figure \ref{fig:banana_dataset}). 

\begin{figure}[!htb]
    \centering
    \subfigure[DGP (PCA mean function)]{\includegraphics[width=0.48\textwidth]{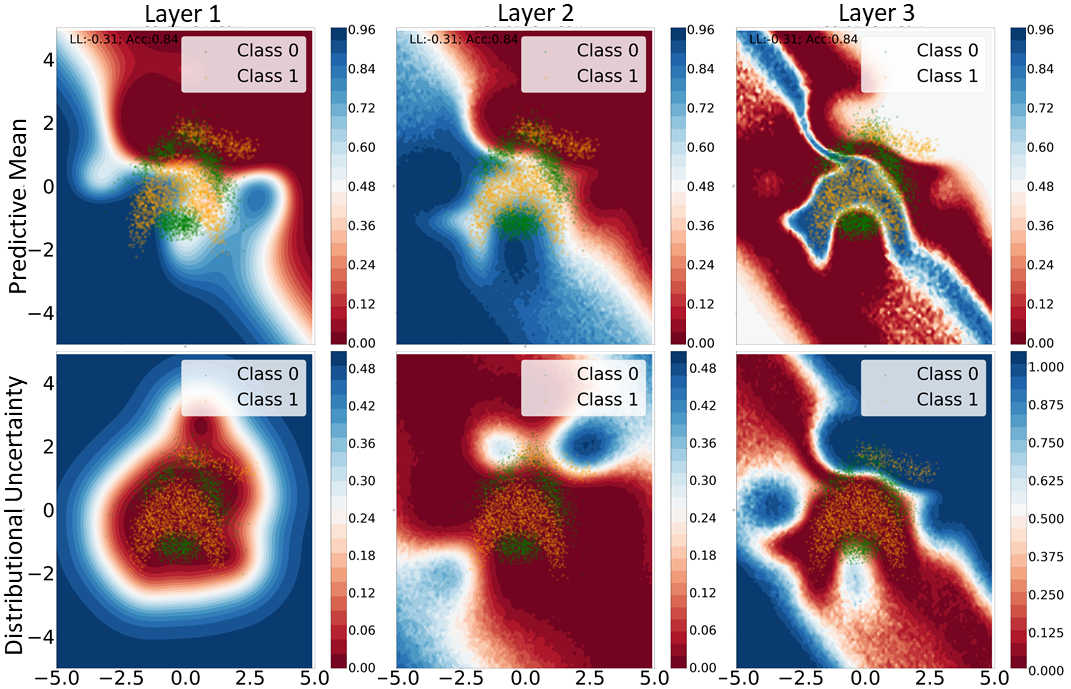}}
    \quad
    \subfigure[DDGP (PCA mean function)]{\includegraphics[width=0.48\textwidth]{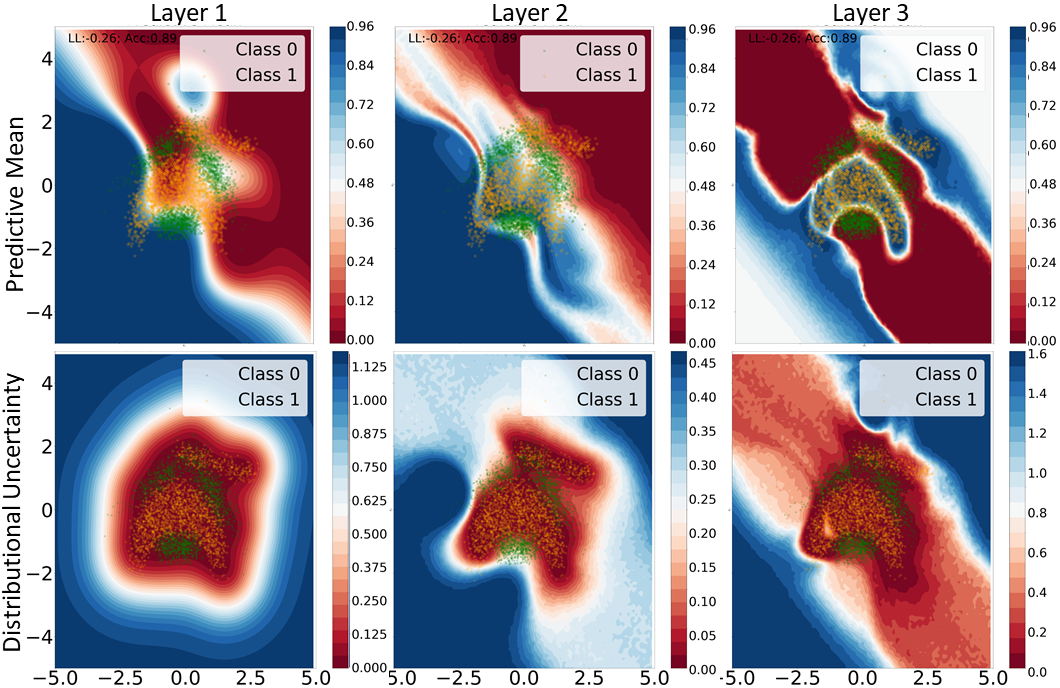}}
    \caption[PCA mean functions influence distributional uncertainty in DGPs.]{\textbf{PCA mean functions influence distributional uncertainty in DGPs.} Layer-wise moments of specified models trained on the ``banana'' dataset. Distributional uncertainty in hidden layers is lower on the data manifold compared to outside the manifold for DDGPs. Conversely, DGPs widely extrapolate low distributional uncertainty outside the data manifold.}
    \label{fig:banana_dataset}  
\end{figure}

\paragraph{Smoothness properties.}

From Proposition \ref{thm:posterior_function_dgp} and  \ref{thm:posterior_function_ddgp} we would expect that DDGPs exhibit smoother hidden layers features. Using our previous experiment on the ``banana'' dataset, we keep track of correlations between two focus points with respect to full input space (Figure \ref{fig:smoothness_banana}). The increased ``locality'' of features in layer 3 is striking in the case of DDGPs. Moreover, we notice even in layer 2 a slight difference in smoothness between the two models.

\begin{figure}[!htb]
  \centering
    \includegraphics[width=\linewidth]{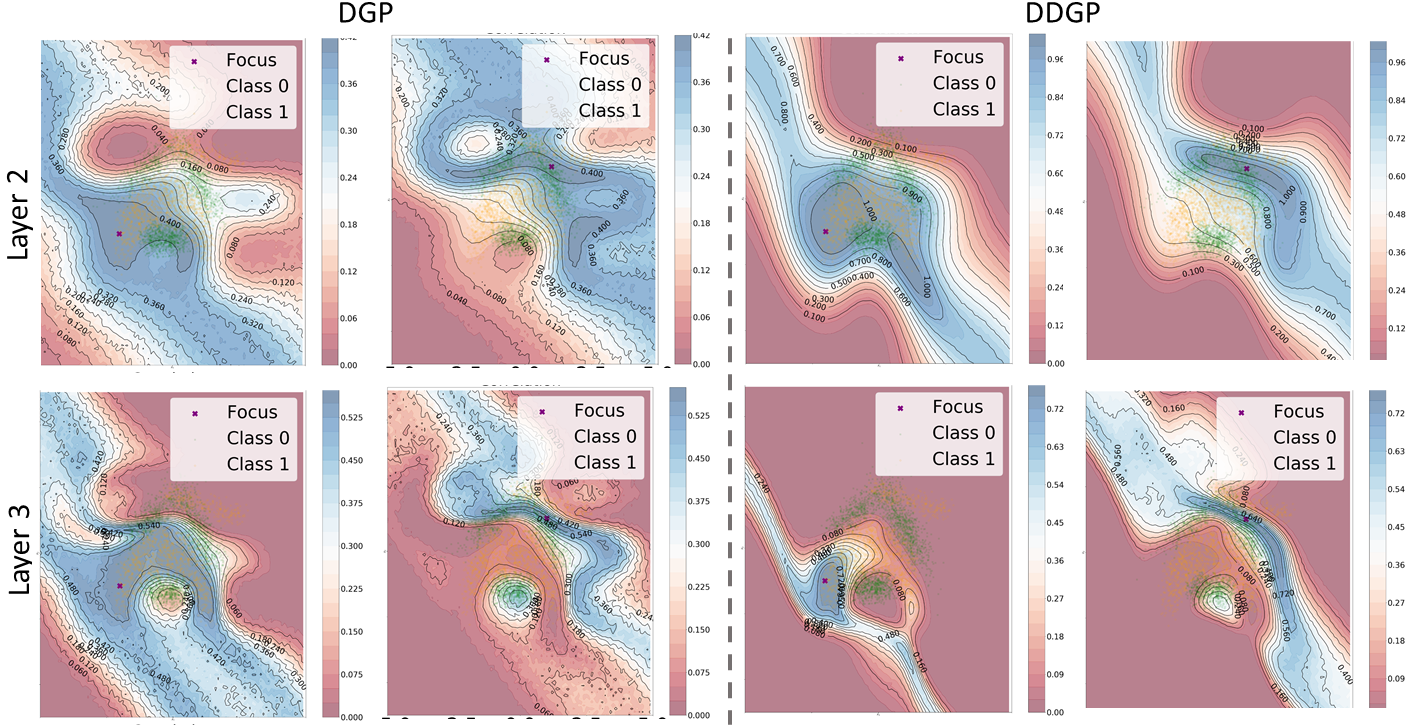}
    \caption[Smoothness comparison between DGPs and DDGPs]{\textbf{Smoothness comparison between DGPs and DDGPs.} Correlation of two focus data points (purple dots) with respect to full input space spectrum at hidden layers for DGPs (PCA mean function) and DDGPs (PCA mean function) trained on ``banana'' dataset. DDGPs (right panel) exhibit smoother and more ``local'' correlations.}
    \label{fig:smoothness_banana}  
\end{figure}

\paragraph{Performance on real-life data.} Besides testing whether DDGPs surpass DGPs in terms of OOD detection, it is also of vital importance to see whether DDGPs manage to obtain calibrated predictive distributions on medium and large scale datasets. For this we evaluate our models on datasets from the UCI machine learning dataset repository. All of the experiments were run with the same initializations of variational parameters and kernel hyperparameters, with the goal of comparing DDGPs with DGPs for 2,3, respectively 4 layers. We used 100 inducing points at each layer, with 5 dimensions per hidden layer and PCA mean functions. Further details regarding training and initialization are given in Appendix \ref{apd:uci_details}.

\begin{figure}[!htb]
    \centering
    \subfigure[Test log-likelihood regression datasets]{\includegraphics[width=0.48\textwidth]{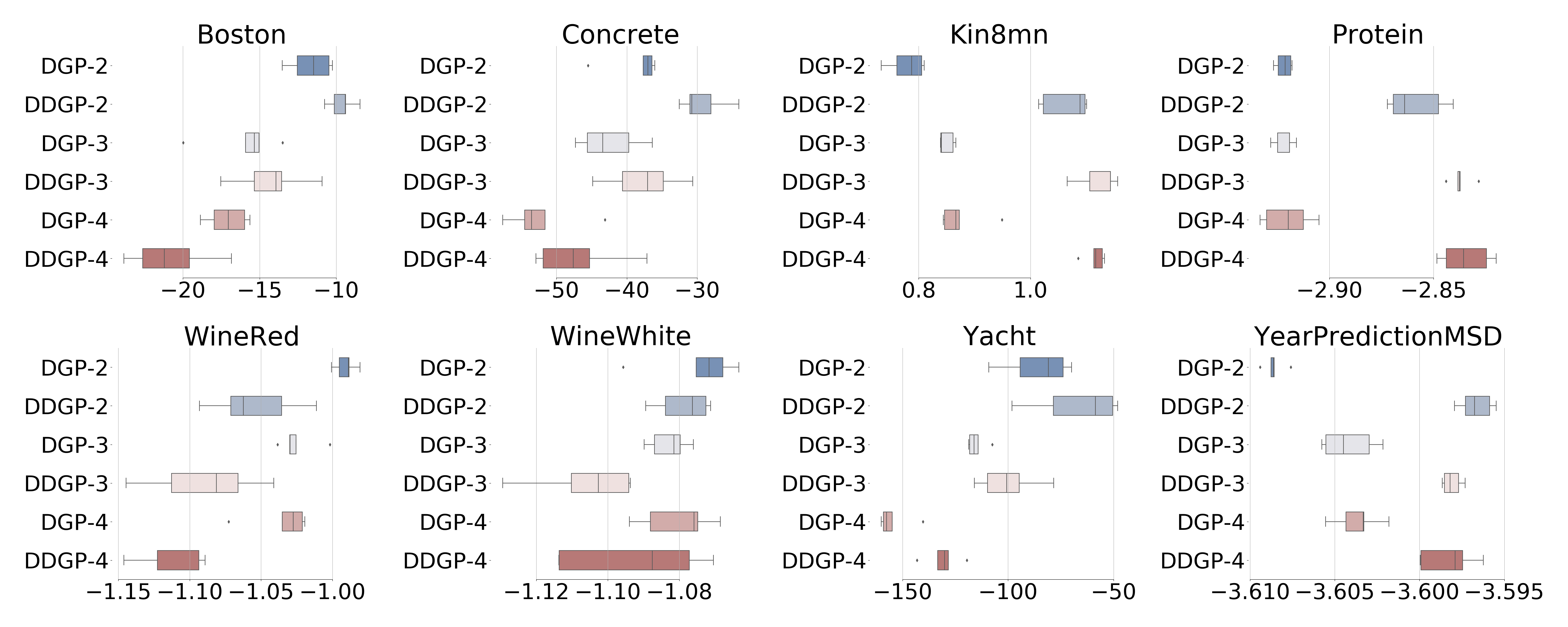}}
    \quad
    \subfigure[Test log-likelihood classification datasets]{\includegraphics[width=0.48\textwidth]{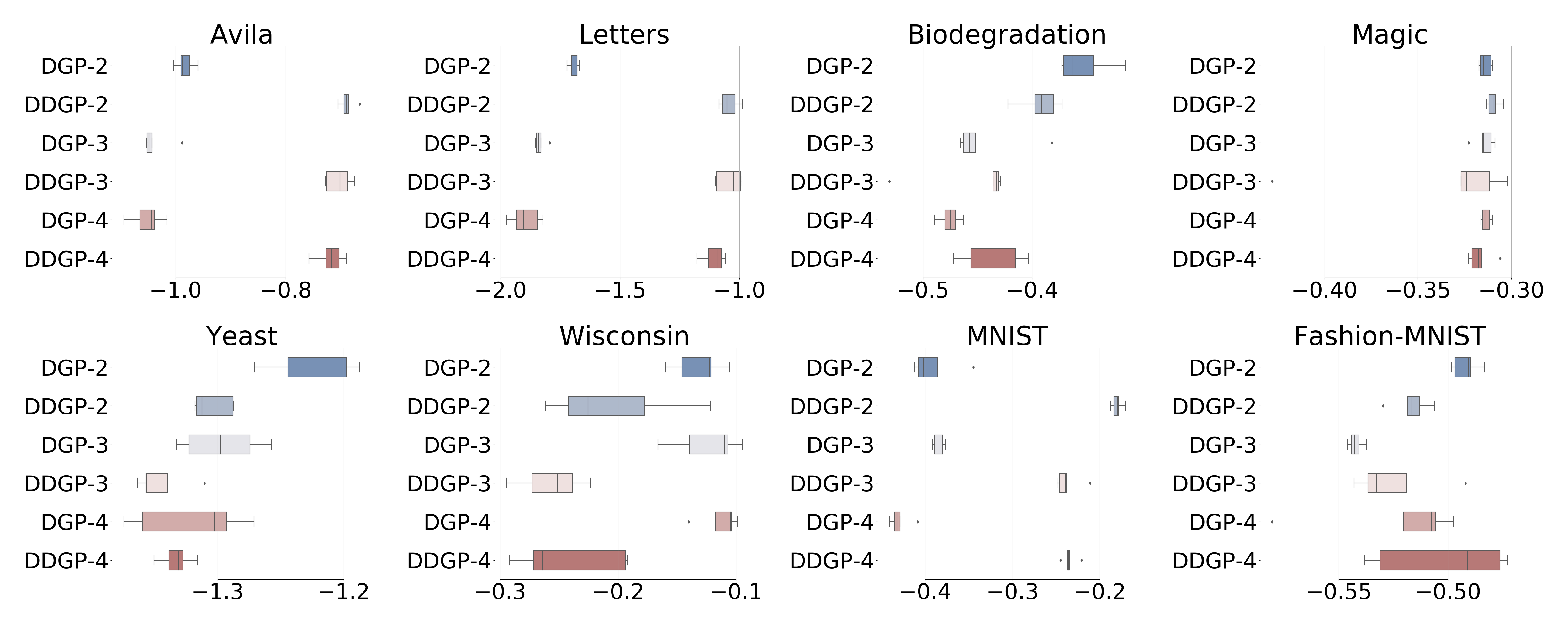}}
    \caption[Results on UCI datasets.]{\textbf{Results on UCI datasets.} \textbf{Top:} Test log-likelihood values on regression datasets. \textbf{Bottom:} Test log-likelihood values on classification datasets. \textbf{All subplots}: 10 different runs of each model with different initialization seeds are taken into the composition of each boxplot. Higher values (to the right) indicate better model fit.}
    \label{fig:results_uci_datasets}
\end{figure}

In terms of predictive log-likelihood values on test data, DDGPs generally surpass their Euclidean only counterpart on regression datasets, with diminished performance seen on ``WineRed'' and ``WineWhite''. In terms of predictive testing set results on classification datasets, the behaviour is roughly similar, with noticeable improved performance of DDGPs on ``Avila'',  ``Letters'' and ``MNIST''. Conversely, the same can be said for DGPs on ``Wisconsin''. On remaining classification datasets the behaviour is similar (Figure \ref{fig:results_uci_datasets}).

\paragraph{Out-of-distribution detection.}\label{sec:outlier_image}

In this subsection we adapt DDGP to the architecture introduced in Deep Convolutional GP (DCGP) \citep{blomqvist2018deep}, obtaining the Distributional Deep Convolutional GP (DistDCGP). The goal is to investigate their OOD detection capabilities on real data. To quantify the departure from normative data, we need a scalar measure encompassing the uncertainty present across all classes. For this we use the distributional differential entropy introduced in equation \eqref{eqn:differential_entropy_formula}. This has the effect to measure the overall distributional uncertainty in logit space, with higher values representing more uncertainty.

In these experiments we assess the capacity of our model to detect dataset shift by training it on MNIST and looking at the uncertainty measures computed on the testing set of MNIST and NotMNIST and SVHN. The hypothesis is that we ought to see higher differential entropy for distributional uncertainty for the digits stemming from a wide array of fonts present in NotMNIST and SVHN as none of the fonts are handwritten. We additionally compare with FashionMNIST, to discern if models are capable of flagging anomalous inputs when deployed in a user-facing system.

\begin{table}[htb]
\begin{center}
 \begin{tabular}{r r |c  c | c  c | c  c } 
 \toprule
 Model & \makecell{ Hidden \\ Layers } &  \multicolumn{2}{c}{\makecell{ MNIST vs. NotMNIST \\ AUC }}  & \multicolumn{2}{c}{\makecell{ MNIST vs. SVHN \\ AUC }} & \multicolumn{2}{c}{\makecell{ MNIST vs. FashionMNIST \\ AUC }} \\ [0.5ex] 
 \midrule
 \multicolumn{2}{c}{Inducing Points} & 50 & 100 & 50 & 100 & 50 & 100 \\
 \midrule
DCGP     & 1 & $0.71$          & $0.69$          & $\textbf{1.0}$ & $\textbf{1.0}$   & $\textbf{0.98}$ & $0.98$ \\ 
DistDCGP & 1 & $\textbf{0.87}$ & $\textbf{0.95}$ & $\textbf{1.0}$ & $\textbf{1.0}$   & $\textbf{0.98}$ & $\textbf{1.0}$  \\ 
 \midrule
DCGP     & 2 & $0.64$          & $0.59$          & $0.97$         & $0.99$           & $0.94$          & $0.92$ \\ 
DistDCGP & 2 & $\textbf{0.97}$ & $\textbf{0.84}$ & $\textbf{1.0}$ & $\textbf{1.0}$   & $\textbf{1.0}$  & $\textbf{0.98}$  \\ 
 \midrule
DCGP     & 3 & $0.64$          & $0.70$          & $0.97$         & $0.89$           & $0.94$          & $0.84$ \\ 
DistDCGP & 3 & $\textbf{0.95}$ & $\textbf{0.86}$ & $\textbf{1.0}$ & $\textbf{1.0}$   & $\textbf{0.96}$ & $\textbf{0.92}$ \\ 
\bottomrule
\end{tabular}
\caption[OOD detection results on MNIST using standard architectures.]{\textbf{OOD detection results.} Performance of OOD detection based on distributional uncertainty differential entropy. Models with different inducing point numbers in hidden layers are trained on MNIST (normative data).}
\label{tab:results_ood}
\end{center}
\end{table}

We notice an abrupt loss in AUC scores on SVHN and FashionMNIST for 3 layers DCGP when transitioning from 50 to 100 inducing points. This reinforces the observed pathology that increasing inducing points results in loss of OOD capabilities. On NotMNIST we notice a substantial discrepancy between the two in favor of DistDCGP (Table \ref{tab:results_ood}). Overall, DistDCGP manages to surpass DCGP in any given comparison regardless of network architecture. Further figures showcasing differences in their distributional differential entropy histograms are provided in Appendix \ref{apd:additional_ood_detection}.

\section{Discussion }


We have shown under a set of assumptions that total uncertainty can be correctly propagated under the DSVI framework introduced in \cite{salimbeni2017doubly} only in the case that the first order derivative of the predictive mean function attains maximum values at the point where previously identified OOD points get mapped to. Intuitively, total uncertainty for OOD points will be higher compared to in-distribution points, if the parametric component of the current layer's SVGP provides a diverse set of function samples when evaluated at points where the OOD point got mapped in the previous layer. 

The collapse of distributional uncertainty as we increase the number of inducing points in the hidden layers is caused by the progressive mapping of data points which in lower layers were flagged as OOD to regions close to inducing points in the next layer in the case that in-distribution data points are mapped in intervals centred around zero in hidden layer space. This leads to the deactivation of $h(\cdot)$ (\ref{eqn:distributional_uncertainty_svgp}) (non-parametric component in SVGP) in all layers, except the first. Hence, in these situations DGPs rely just on the variance stemming from $g(\cdot)$ (\ref{eqn:epistemic_uncertainty_svgp}) (parametric component of SVGP). 


The PCA based projections stemming from the training data as introduced in \cite{salimbeni2017doubly} help to solve these aforementioned issues. However, these also come with a hidden cost, as the PCA projections will wildly extrapolate the regions of low distributional uncertainty in directions perpendicular to the eigenvectors in the case of high-dimensional datasets that are projected down in hidden layer space (see experiments on ``banana" dataset). Our proposed framework solves this problem as the Wasserstein-2 part of our hybrid kernel explicitly takes into account differences between the variance of hidden layer SVGPs. Intuitively, even if the mean component of the SVGP is dominated by the PCA projection, which could coincide with the projection values of the training set, if the variance component is higher then the Wasserstein-2 distance will highlight this domain shift with higher distributional uncertainty. Besides this intuition, we have also provided theoretical guarantees that total uncertainty is higher for OOD points compared to in-distribution points for DDGPs in the case that the inducing locations' variance is close in distribution to total uncertainty of in-distribution points at that hidden layer. This is a scenario which happens in practice as the variance of inducing locations will be optimized to match that of the in-distribution points. Future work should explore methods to explicitly enforce this match in distribution between inducing locations and in-distribution predictive distributions. Recent work \citep{ustyuzhaninov2019compositional, ober2021global} has proposed to propagate the initial inducing locations in input space through the hierarchy, thus obtaining predictive distributions for these ``global" inducing points at each layer. Under this construction, the distribution match between inducing locations and in-distribution points will be guaranteed at each layer in DDGPs. 

Besides theoretical guarantees on correct propagation of distributional uncertainty in DDGPs, we have also shown that DDGPs have theoretical guarantees to better guard against \emph{feature-collapse}. Moreover, the Wasserstein-2 distance based on predictive moments will cause DDGPs to be smoother compared to DGPs (see Figure \ref{fig:smoothness_banana}). Due to all these attractive properties of DDGPs, it is perhaps unsurprising that on OOD detection tasks our proposed model has managed to surpass baseline models. Additionally, we analyze their sensitivity to input perturbations (see Appendix \ref{apd:input_perturbation}). 

Moreover, we discover another pathology pertaining to DGPs, respectively that when deploying very wide (i.e., high-dimensional hidden layers) DGPs, samples can collapse to a thin hyper-sphere, thus requiring increasing inducing locations to cover it. In Appendix \ref{apd:fail_high_dimensions} we offer more details on this pathology, also showing that hybrid kernels can guard against it.

In conclusion, we have shown that DDGPs are better suited for outlier detection on both toy and real datasets, while also showing slight improvements at testing time in medium and large scale datasets. Future work should consider adapting DDGPs to safety-critical domains such as biomedical image segmentation  \citep{czolbe2021segmentation} or semantic segmentation \citep{franchi2020one}.

\bibliographystyle{unsrtnat}
\bibliography{references}  






\end{document}


\title{}
\appendix

\section{Experiments Details} \label{apd:uci_details}

\paragraph{Standard architecture.} For all UCI datasets we randomly selected 20$\%$ as the testing set, with the remainder being used for training. All implementations use the RBF kernel with automatic relevance determination for  DGPs, whereas for DDGPs the standard RBF kernel on Euclidean space is used just for the first layer, with the remaining layers using the hybrid kernel. In terms of model architecture, we use 5 hidden units for each hidden level, with 100 inducing points. All models are optimized until convergence with a mini-batch of size 32. Results are provided for 2,3 and 4 layers. 

\paragraph{Convolutional architecture.} In terms of model architecture, we use 3 hidden layers with 5 hidden units for each hidden level, with 250 inducing inputs. All models are optimized until convergence with a mini-batch of size 32 using the training set of MNIST. At testing time we compare the distributional differential entropy between testing set of MNIST (i.e., normative data) and OOD datasets (i.e., NotMNIST, SVHN and Fashion-MNIST) using AUC scores to discern if models are able to separate OOD points from in-distribution points.

\paragraph{Euclidean Kernels.}

All squared exponential kernels used a lengthscale hyperparameter per input dimension, initialized to 1.351. We initialize the variance of the kernel with the same value.

\paragraph{Hybrid Kernels.}

All hybrid kernels used a lengthscale hyperparameter per input dimension, initialized to 1.351. We initialize the variance of the kernel with the same value.

\paragraph{Euclidean inducing points \& DGP approximate posterior.}

For all experiments we used 100 inducing points for the first layer, which is mapping from input space to first hidden layer space. We initialize the inducing point locations to the k-means of the training data. For the approximate posterior of the inducing point values we 
initialized the mean by uniformly sampling between -2.0 and 2.0. The Cholesky decomposition of the variance was initialized to be the diagonal matrix multiplied by $1e^{-5}$.

\paragraph{Distributional inducing points \& DDGP approximate posterior.}

For all experiments we used 100 inducing points in the hidden layers of the multi-layered GP architectures. Since these inducing point are operating in Wasserstein space, the inducing point locations are actually distributions in this case. We initialize the mean of the inducing points locations by uniformly sampling between -2.0 and 2.0. The initial variance of the inducing point locations is taken to be the diagonal matrix multiplied by $1e^{-1}$. This low initialization of the variance term was chosen so that the DDGP will implicitly learn to produce low variance terms at all hidden layers for data points inside the data manifold. For the approximate posterior of the inducing point values, we take the same initialization as in the case of DGPs. 

\paragraph{Optimization on UCI datasets.}

All parameters were optimized using the Adam optimizer with a learning rate of 0.001. We used a batch size of 32 and trained for 50,000 iterations or until convergence.

\paragraph{Likelihood.} For regression tasks the likelihood variance was initialized to 1.0

\section{Proof of Theorem \ref{thm:positive_definite}} \label{apd:pos_def_kernels}

\cite{thi2019distribution} provides a detailed proof that a generalized radial basis function kernel with Wasserstein distances is positive definite. We recall here the main results.

\begin{theorem}[Schoenberg's Theorem]
    Let $F:\mathbb{R^{+}} \Rightarrow \mathbb{R^{+}}$ be strictly
    monotonically decreasing, and $\mathbb{K}$ a negative definite kernel.  Then $(x,y) \Rightarrow F(\mathbb{K}(x,y))$ is a positive definite kernel.
\end{theorem}

\begin{proposition}
    The function $W_{2}^{2H}$ is a negative definite kernel if and only if $0 < H \leq 1$. 
\end{proposition}
A complete proof of this proposition can be found in \cite{kolouri2016sliced}, where they prove this result for absolutely continuous distributions in $W_{2}(\Omega)$, where $\Omega \subset \mathbb{R}$ is a compact subset.
Theorem \ref{thm:positive_definite} follows from Schoenberg's Theorem and the above proposition.

\section{Isotropic kernel properties in (Distributional) Deep Gaussian Processes}\label{apd:fail_high_dimensions}

\paragraph{Point-mass inducing points are doomed to fail in high-dimensional spaces.} 

This section is written from a purely theoretical perspective in the hypothetical scenario that DeepConvGP architectures or more generally DeepGP architectures are scalable to high dimensions. Operating in this scenario, we can take the case of DGP, with the following analysis being easily extendable to DeepConvGP architectures. 
We first introduce a general result from the concentration of measure in high-dimensional spaces.
\begin{definition}
     Multivariate spherical Gaussian distributions with a different mean $u_{k}$ and  variance $\sigma_{k}^{2}$ for each dimension has the following p.d.f:
    \begin{equation}
        p(x) = \frac{1}{\left(2\pi\right)^{d/2}\prod_{k=1}^{K}\sigma_{k}} \exp{ \left[ - \sum_{k=1}^{K}\left(\frac{(x_{k}-u_{k})^{2}}{2\sigma_{k}^{2}} \right) \right]}
    \end{equation} 
\end{definition}

\begin{definition}[Gaussian Annulus Theorem] 
For a d-dimensional spherical Gaussian with unit variance in each direction, we have that:
\begin{equation}
    Pr(|x-\sqrt{d}| \leq \beta)\leq 3\exp{-c\beta^{2}}    
\end{equation}
where c is a fixed constant. This results in samples being concentrated in a thin annulus of size 1 with a radius $\sqrt{d}$, centred at zero.
\end{definition}

\begin{figure}[htb]
    \centering
    \includegraphics[width=0.95\linewidth]{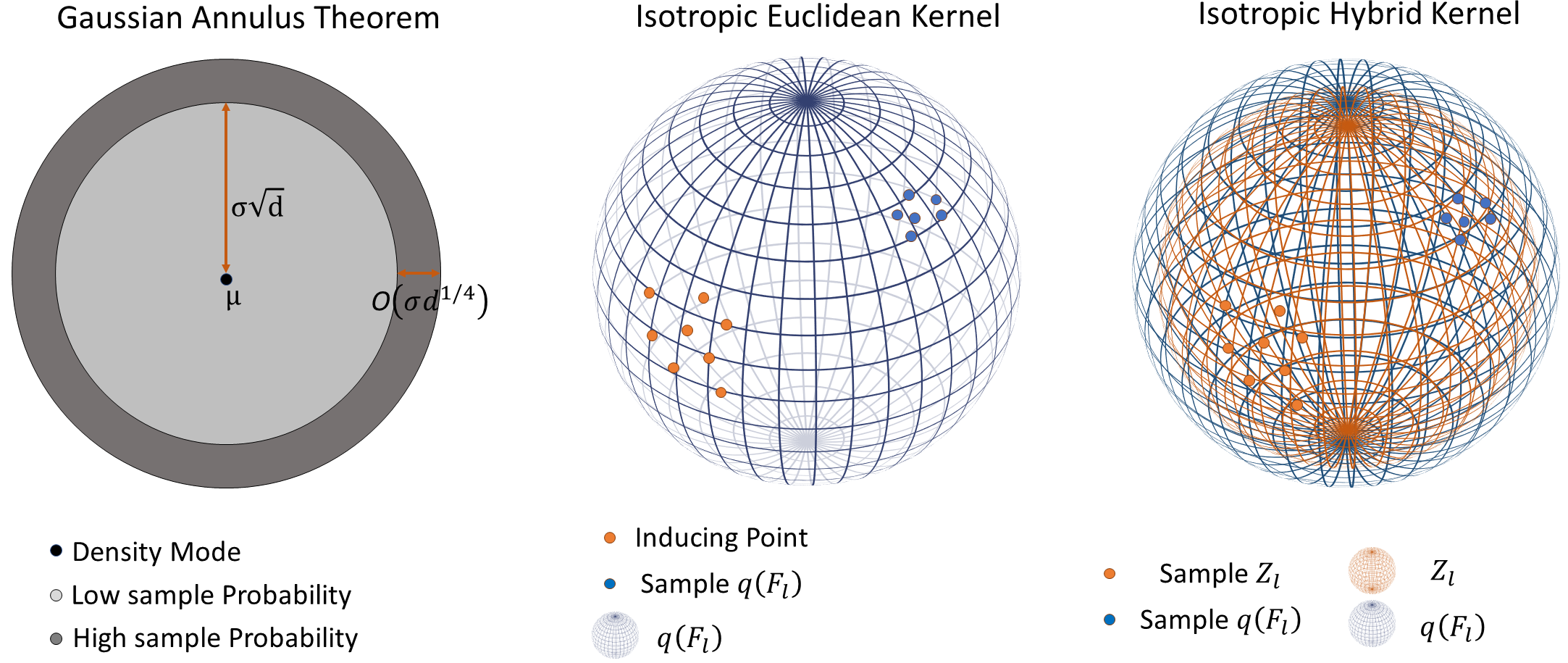}
    \caption[Gaussian Annulus Theorem and behaviour of (D)DGPs in high-dimensional spaces]{\textbf{Gaussian Annulus Theorem and behaviour of (D)DGPs in high-dimensional spaces.} Left: Visual depiction of Gaussian Annulus Theorem; Middle: Hidden layer samples in DGPs concentrate on a thin hypersphere, with inducing points having to cover the entire surface; Right: Hidden layer samples in DDGPs concentrate on a thin hypersphere, but due to distributional inducing points and Wasserstein-2 part of hybrid kernel this will not result in the necessity to cover entire hypersphere.}
    \label{fig:gaussian_annulus}  
\end{figure}

If we consider the posterior over the l-th layer of a DGP $q(F_{l}) \sim \mathcal{N}(\tilde{U_{l}}, \tilde{\Sigma_{l}})$ be of size $d_{l}$ where we consider the dimensionality to be sufficiently large so that the above theorem holds, with the added assumption that the d-dimensional $\tilde{\Sigma_{l}}$ values for data point $x$ are roughly equal. Then we can consider samples from $q(F_{l})$ as $\{F_{l}^{1} - \tilde(U_{l}), \cdots,  F_{l}^{S}- \tilde(U_{l}) \}$, which are individually spherical Gaussians of mean zero and $\tilde{\Sigma_{l}}$ variance. Thereby, by applying the Gaussian Annulus Theorem we can see that samples, which are used in the DSVI framework commonly used in DGPs \citep{salimbeni2017doubly}, concentrate around an annulus centred at $\tilde{U_{l}}$ with the radius being $\sqrt{d_{l}\Sigma_{l}}$. Hence, the Euclidean distance with respect to the centre increases for higher values of $d_{l}$.

\begin{remark}
    This will only constitute an issue with isotropic kernels, as the locations $Z_{l}$ will have to cover an increasing surface of the $d_{l}\text{-sphere}$ created by samples from $q(F_{l})$. This will constitute less of a problem for hybrid kernels (isotropic as well for Euclidean part) present in DDGPs, as the Wasserstein-2 distance between the multivariate normal stemming from $q(F_{l})$ and that of $Z_{l}$ will be reliable.
\end{remark}

\section{Additional Results on Out-of-Distribution Detection for DDGPs}

\paragraph{Deep Convolutional Gaussian Processes architecture.}

We use 250 inducing points at each layer. Every hidden layer is taken to have five channels. We use the Adam optimizer starting with a learning rate of 0.01 and decreasing by 0.1 every 20,000 iterations. We optimized all models trained on MNIST for 100,000 iterations.

\begin{figure}[!htb]
    \centering
    \includegraphics[width=\linewidth]{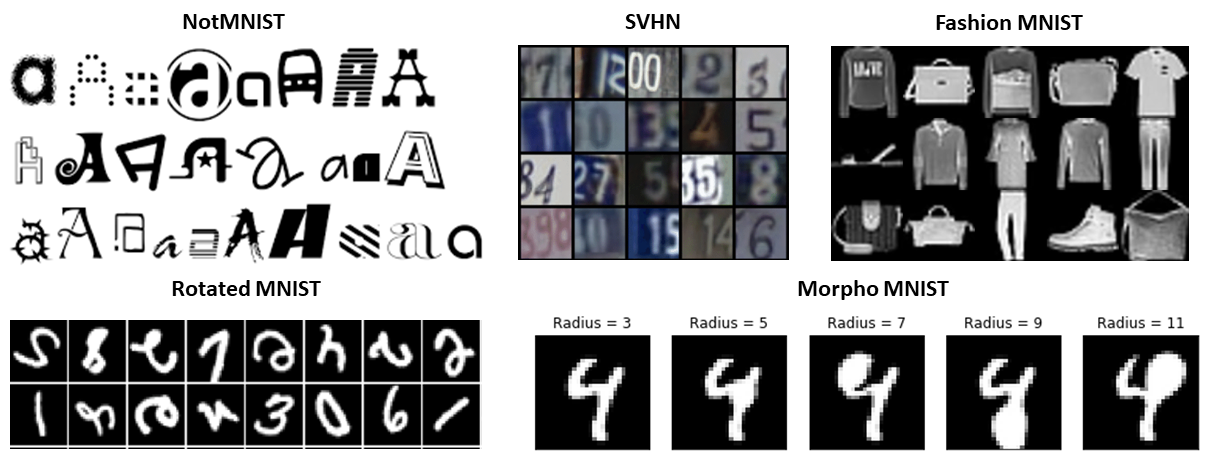}
    \caption[Examples of OOD datasets]{\textbf{Examples of OOD datasets.}}
    \label{fig:ood_examples}  
\end{figure}

Visual examples of OOD datasets used throughout thesis are given in Figure \ref{fig:ood_examples}.

\subsection{Sensitivity to input perturbation} \label{apd:input_perturbation}

\paragraph{Morpho-MNIST.}

For better understanding the morphological changes occurring with swelling of different radii we provide some examples for digit 4 (see Figure \ref{fig:ood_examples}). From Figure \ref{fig:morpho_mnist_ddgp} we can see that for 50 inducing points at each hidden layer, DGPs and DDGPs have roughly similar behaviour. However, for 100 inducing points we can notice that DGPs (i.e., for 2 and 4 layers) lose their ability to notice these morphological changes in the input by virtue of inspecting their distributional differential entropy. The same cannot be said for DDGPs, which regardless of inducing point numbers provide a steady upward trend in distributional differential entropy as swellings of higher intensity are applied on the original digit. This experiment goes hand in hand with our observation on toy datasets (see Figure \ref{fig:visual_collapse_of_distributional_variance_num_inducing}), where we have seen that by increasing the number of inducing points the distributional uncertainty in zero-mean function DGPs can collapse.

\begin{figure}[!htb]
    \centering
    \subfigure[50 inducing points]{\includegraphics[width=\textwidth]{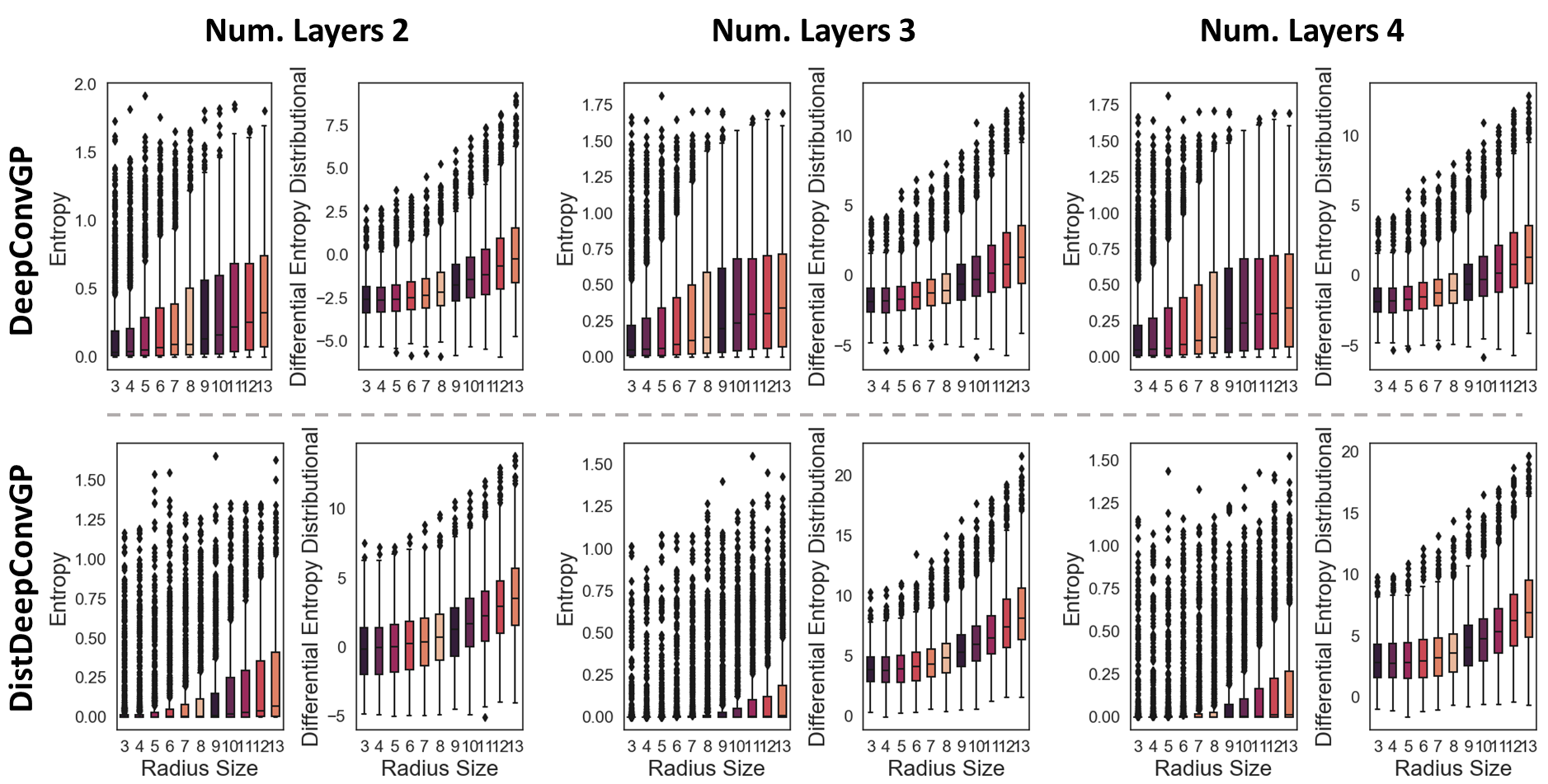}}
    \quad
    \subfigure[100 inducing points]{\includegraphics[width=\textwidth]{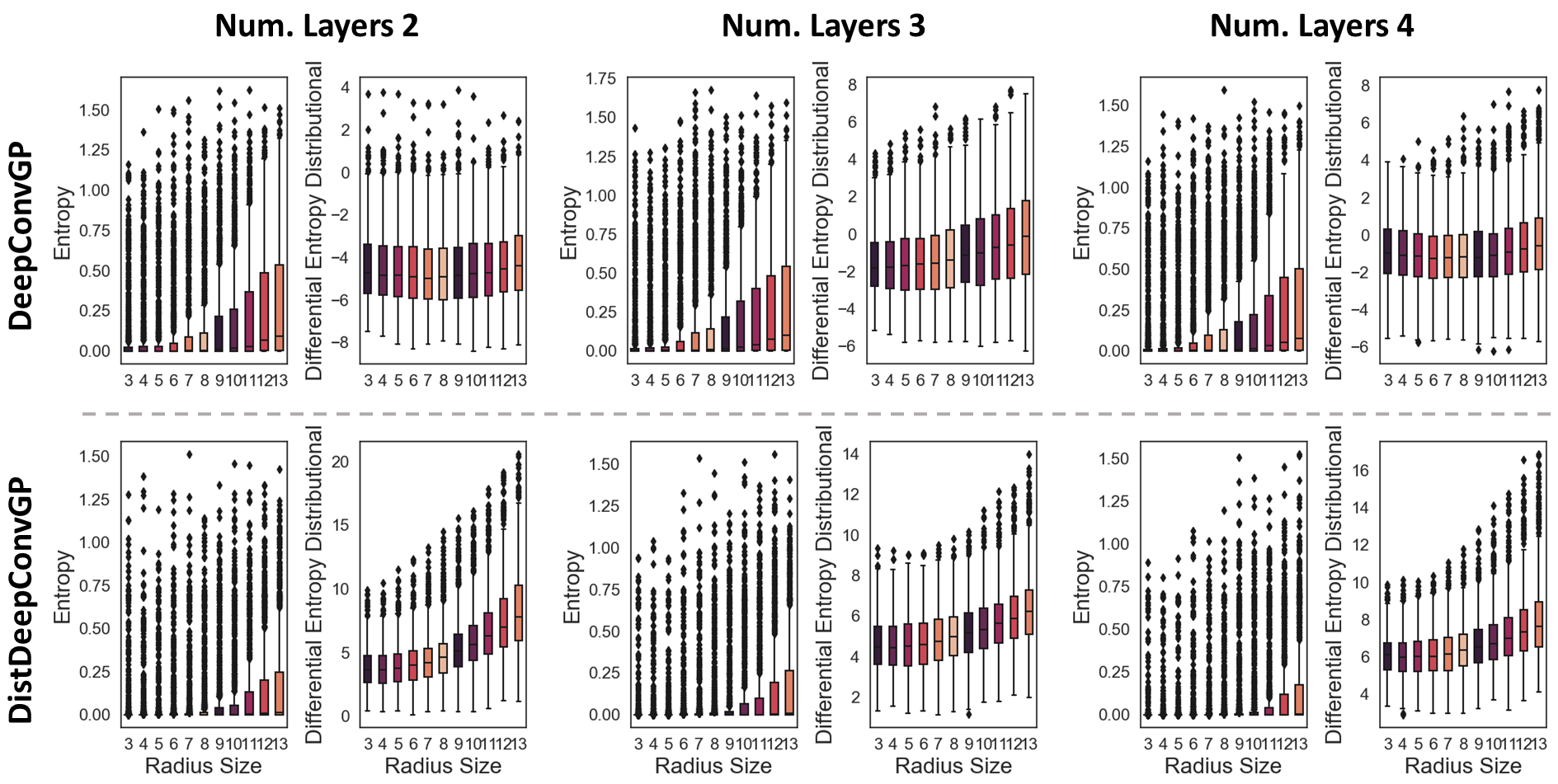}}
    \caption[Morpho-MNIST]{\textbf{Sensitivity to morphological deformations.} Histograms of predictive entropy and distributional differential entropy for varying degrees of morphological deformation applied on MNIST digits. Higher values indicated an increased uncertainty.}
    \label{fig:morpho_mnist_ddgp}  
\end{figure}

\begin{figure}[!htp]
    \centering
    \includegraphics[width=\linewidth]{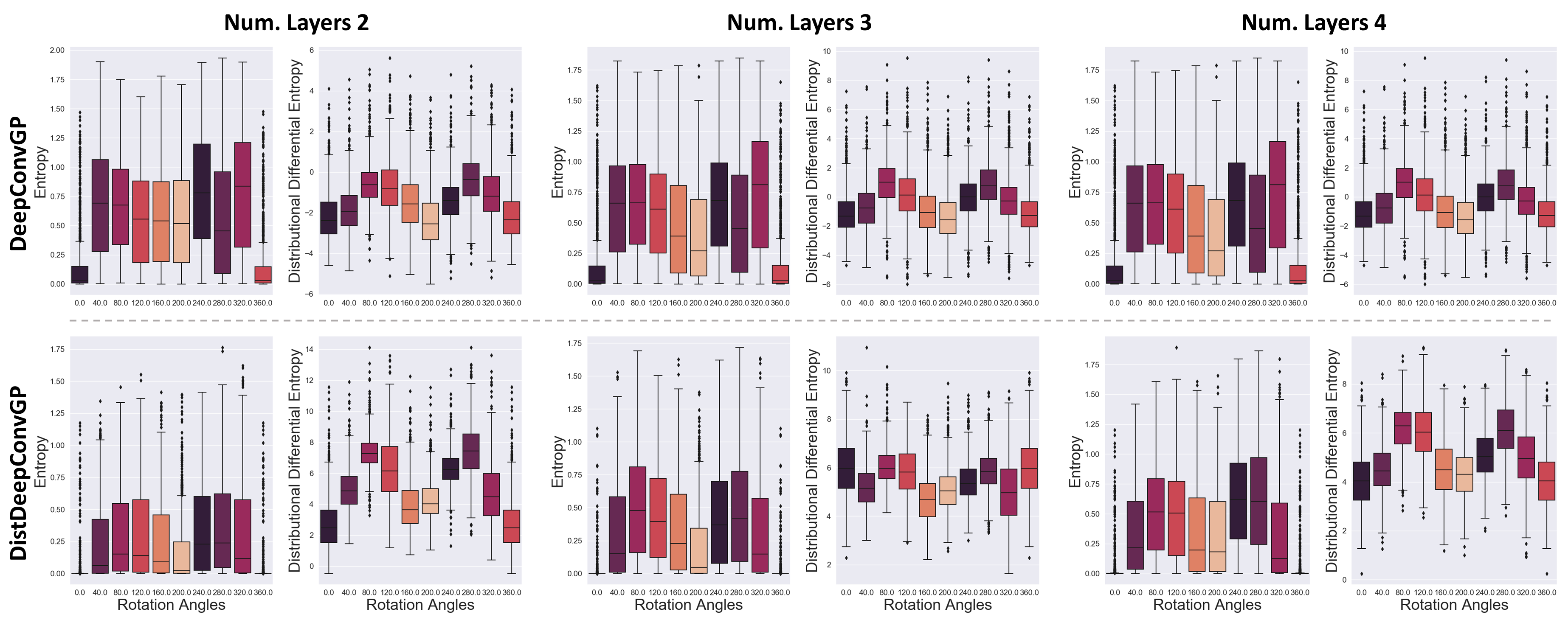}
    \caption[Rotated-MNIST]{\textbf{Sensitivity to image rotation.} Histograms of predictive entropy and distributional differential entropy for varying degrees of rotation applied to digit 8 images from MNIST. Higher values indicated an increased uncertainty. Models trained use 50 inducing points in the hidden layers.}
    \label{fig:rotated_mnist_num_inducing_50}  
\end{figure}

\paragraph{Rotated-MNIST.}

To further assess the sensitivity of our methods, we employ the experiments introduced in \cite{gal2016dropout} by successively rotating digits from MNIST. We expect to see an increase in both predictive entropy and distributional differential entropy as digits are rotated. For our experiment we rotate digit 6. When the digit is rotated by around 180 degrees the predictive entropy and distributional differential entropy should revert back closer to initial levels, as it will resembles digit 9. From Figure \ref{fig:rotated_mnist_num_inducing_50} we can notice that in terms of distributional differential entropy the same patterns are exhibited by both considered models, with slightly less overlapping distributions for DistDeepConvGP.

\subsection{Additional Figures for OOD detection} \label{apd:additional_ood_detection}

In this subsection we provide additional figures pertaining to the OOD detection experiments summarized in Table \ref{tab:results_ood}.

\begin{figure}[!htb]
    \centering
    \subfigure[50 inducing points]{\includegraphics[width=\textwidth]{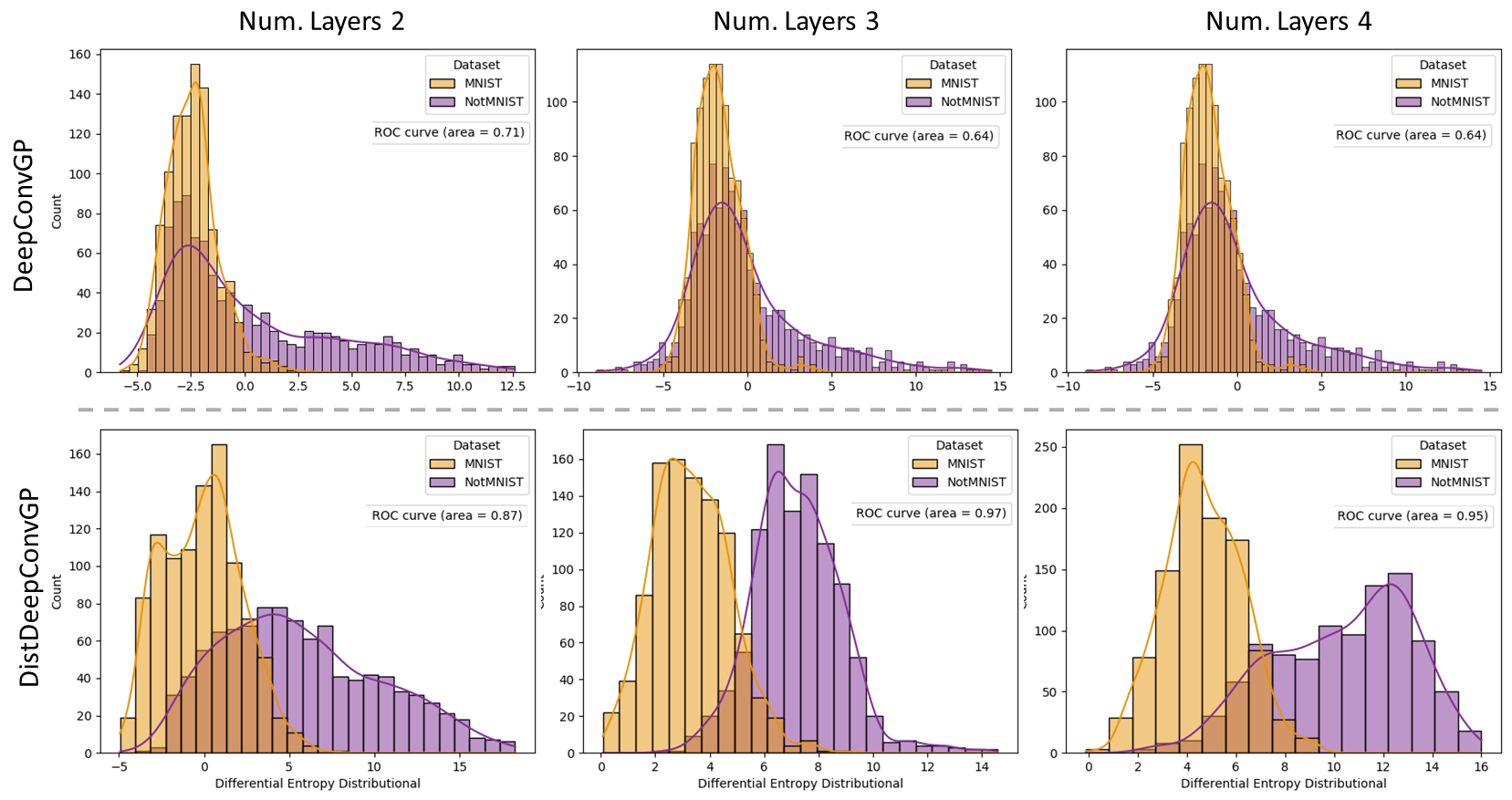}}
    \quad
    \subfigure[100 inducing points]{\includegraphics[width=\textwidth]{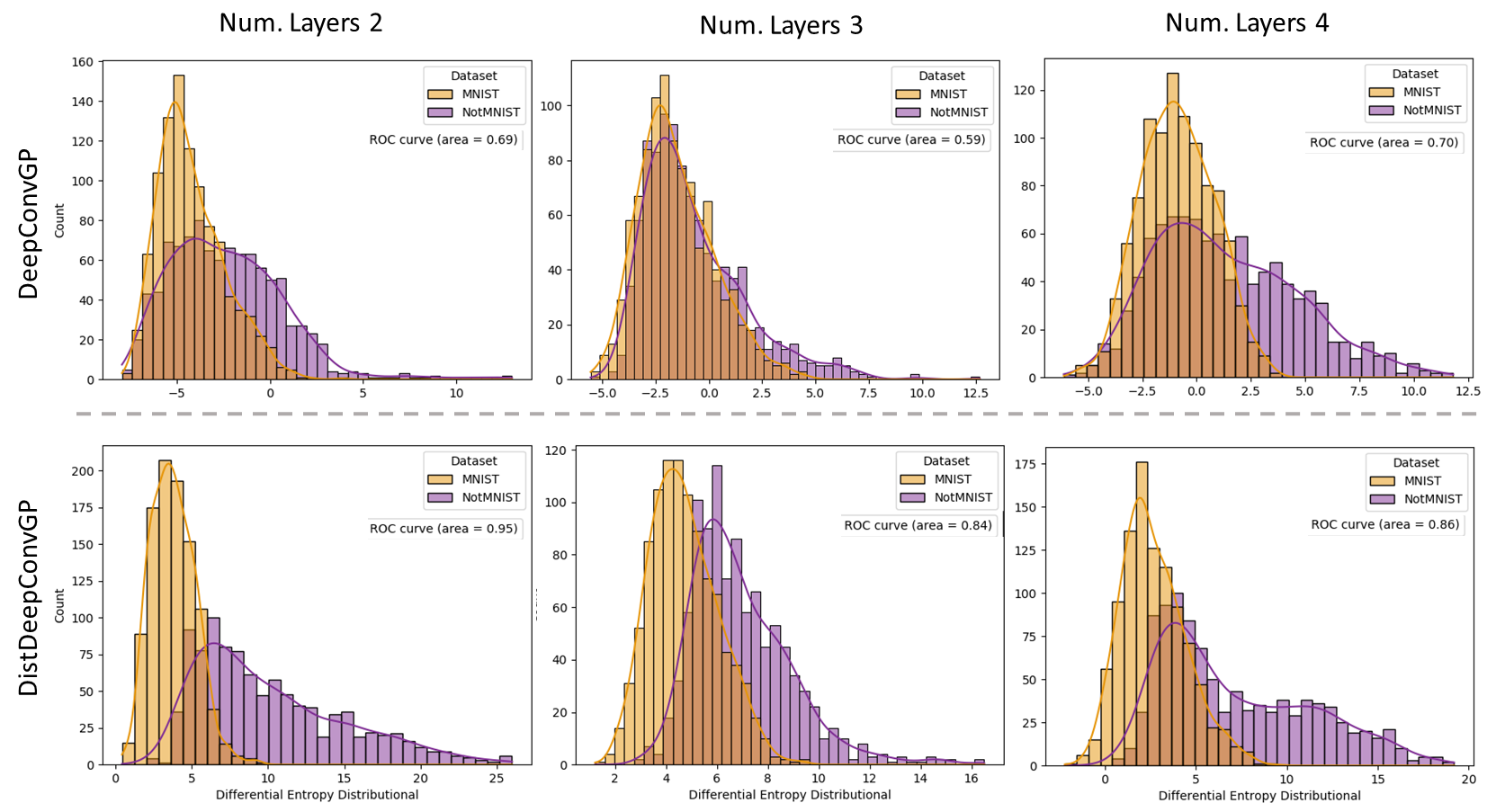}}
    \caption[OOD detection MNIST vs. NotMNIST]{\textbf{OOD detection MNIST vs. NotMNIST.} Histograms of distributional differential entropy computed on MNIST and NotMNIST. Higher values indicate an increased uncertainty.}
    \label{fig:not_mnist_appendix} 
\end{figure}

\begin{figure}[!htb]
    \centering
    \subfigure[50 inducing points]{\includegraphics[width=\textwidth]{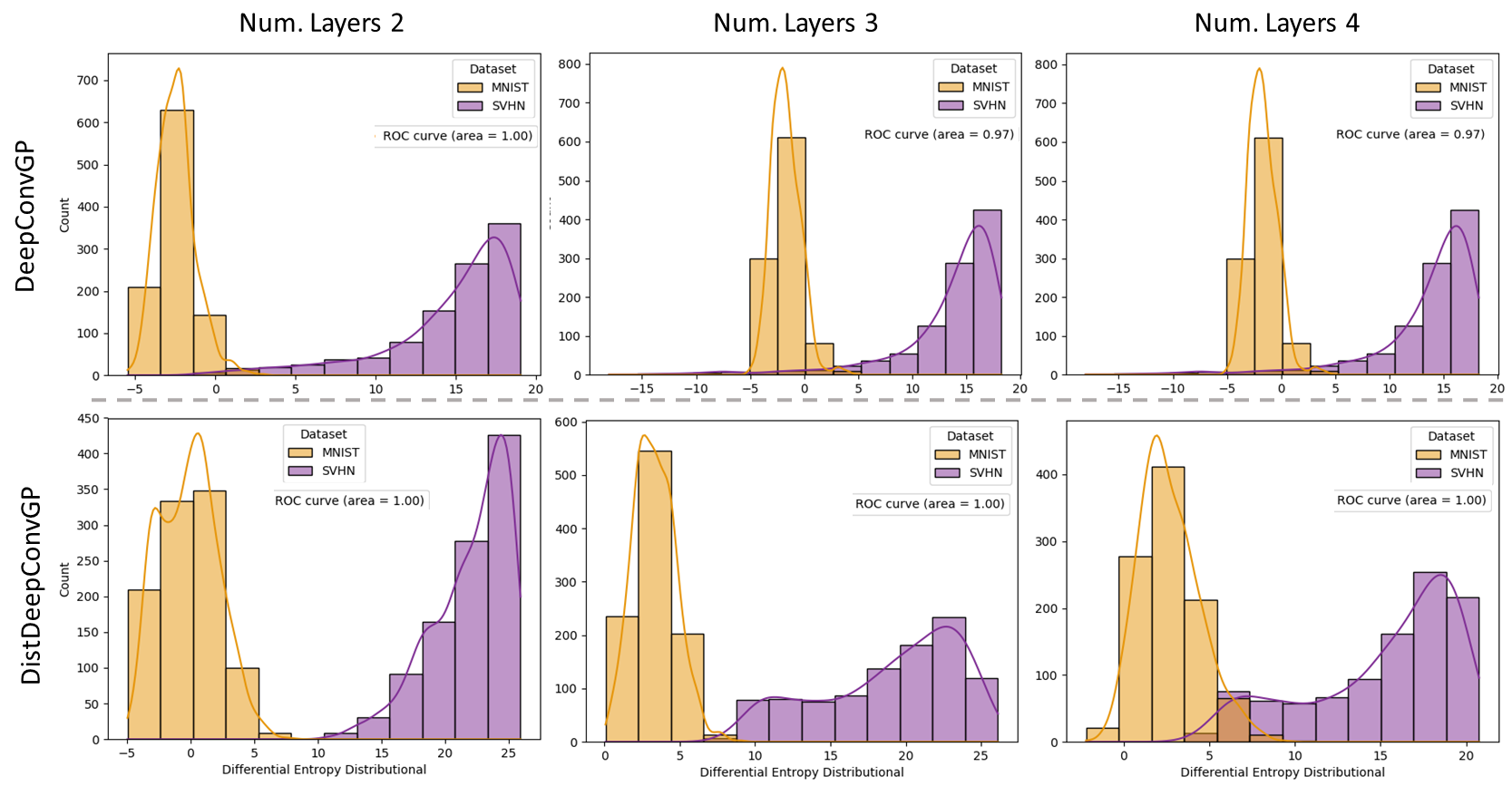}}
    \quad
    \subfigure[100 inducing points]{\includegraphics[width=\textwidth]{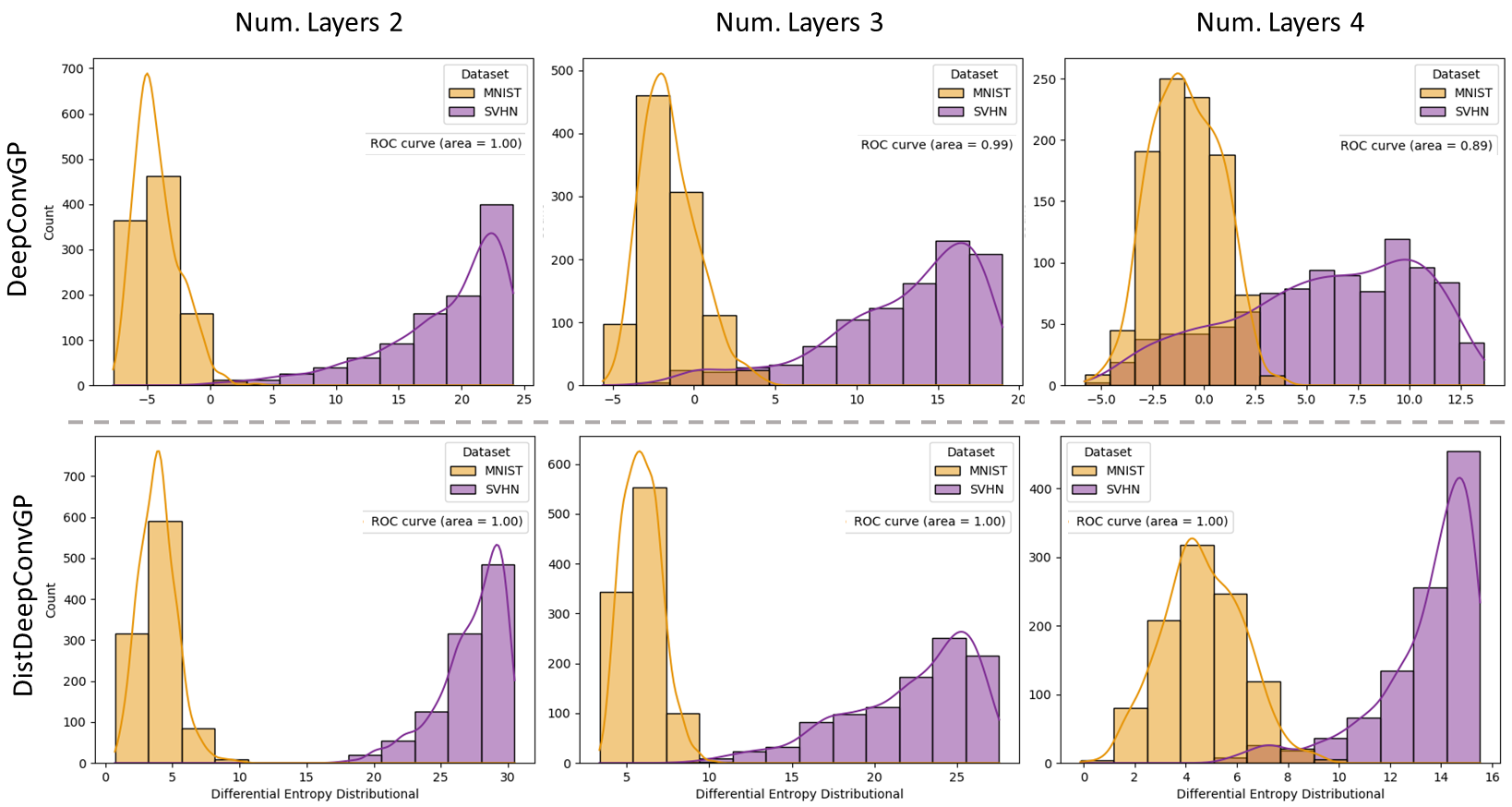}}
    \caption[OOD detection MNIST vs. SVHN]{\textbf{OOD detection MNIST vs. SVHN.} Histograms of distributional differential entropy computed on MNIST and SVHN. Higher values indicate an increased uncertainty.}
    \label{fig:svhn_appendix} 
\end{figure}

\begin{figure}[!htb]
    \centering
    \subfigure[50 inducing points]{\includegraphics[width=\textwidth]{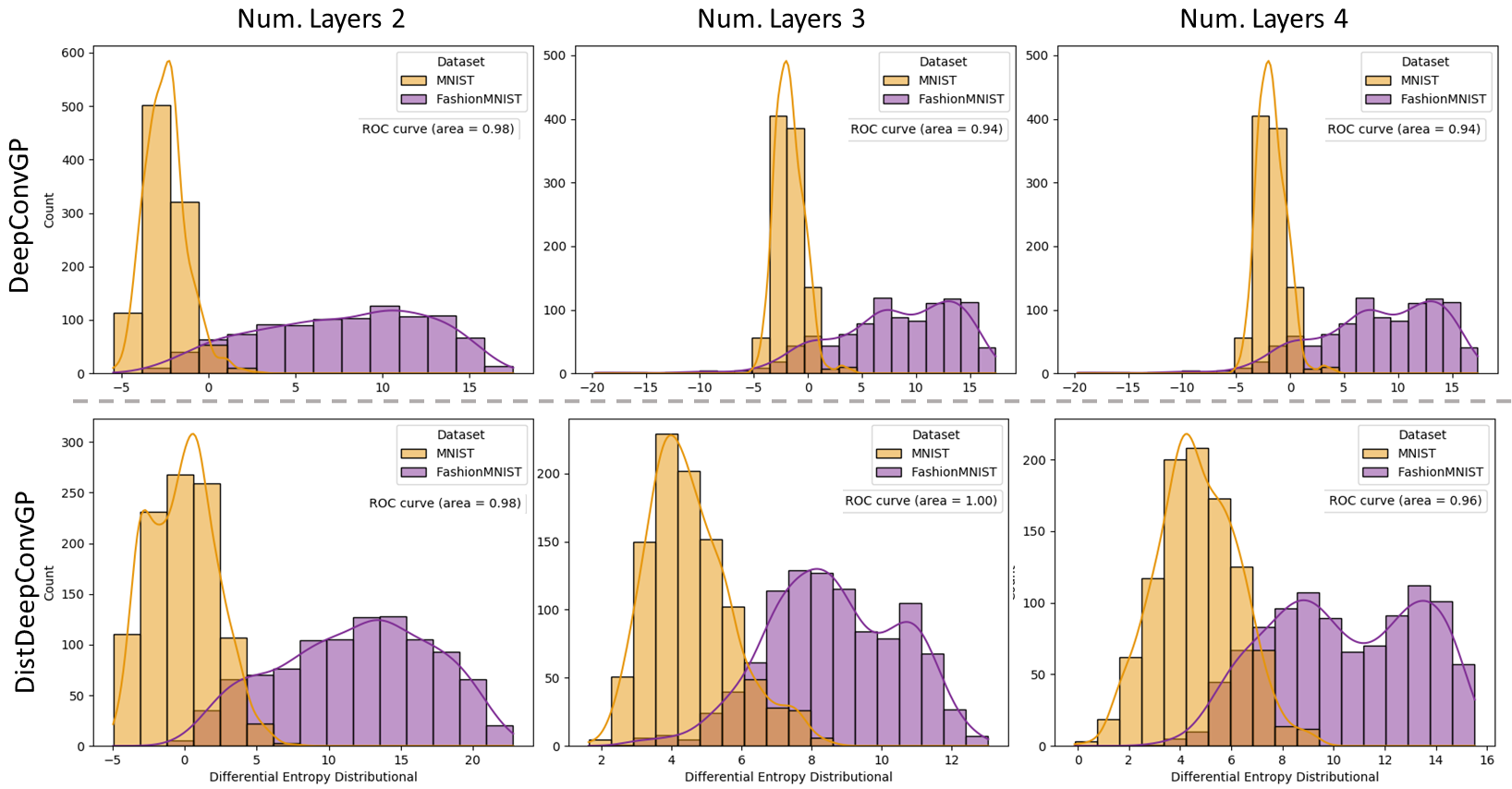}}
    \quad
    \subfigure[100 inducing points]{\includegraphics[width=\textwidth]{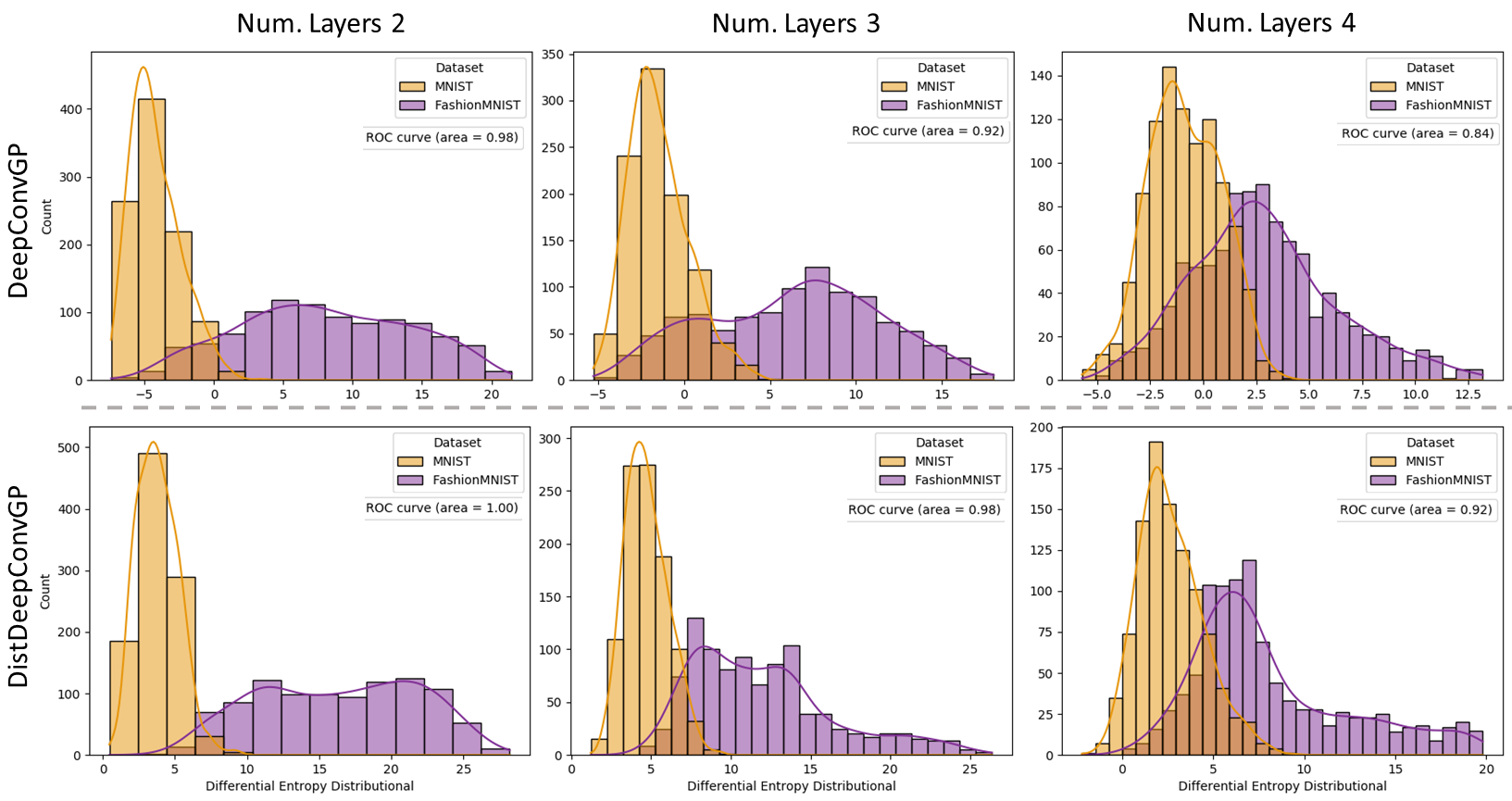}}
    \caption[OOD detection MNIST vs. FashionMNIST]{\textbf{OOD detection MNIST vs. FashionMNIST.} Histograms of distributional differential entropy computed on MNIST and FashionMNIST. Higher values indicate an increased uncertainty.}
    \label{fig:fashion_mnist_appendix} 
\end{figure}